\newif\ifonlineapp
\onlineapptrue

\ifonlineapp
\documentclass[11pt]{article}
\usepackage{fullpage}
\else
\documentclass[onefignum,onetabnum,siopt]{siamart190516}
\fi

\usepackage[normalem]{ulem}

\usepackage{graphicx}

\usepackage{amsmath,amssymb}
\usepackage{bbm}
\usepackage{float}
\usepackage{accents}
\usepackage{hyperref}
\usepackage{algorithm}
\usepackage{algorithmic}
\usepackage{wrapfig}
\usepackage{caption}
\ifonlineapp
    \usepackage[none]{hyphenat}
\fi

\newcommand{\red}{\color{red}}

 \renewcommand{\aa}{\mathbf{a}}
  
  \providecommand{\md}{\mathrm{d}}

  \providecommand{\vv}{\mathbf{v}}
  \providecommand{\ww}{\mathbf{w}}
  \providecommand{\xx}{\mathbf{x}}
  \providecommand{\yy}{\mathbf{y}}
  \providecommand{\zz}{\mathbf{z}}

  \providecommand{\e}{\mathrm{e}}

  \providecommand{\mE}{\mathbf{E}}

  \providecommand{\mS}{\mathbf{S}}

  \providecommand{\mV}{\mathbf{V}}

  \providecommand{\cE}{\mathcal{E}}
  
  \providecommand{\cG}{\mathcal{G}}

  \providecommand{\cN}{\mathcal{N}}
  \providecommand{\cO}{\mathcal{O}}

\providecommand{\bone}{\mathbbm{1}}
\def\leref#1{Lemma~\ref{#1}}
\def\figref#1{Fig.~\ref{#1}}
\ifonlineapp
    \newtheorem{lemma}{Lemma}
    \newtheorem{theorem}{Theorem}
    \newtheorem{corollary}{Corollary}
    \newtheorem{definition}{Definition}
    \newtheorem{proof}{Proof}
\fi

\newcommand{\comment}[1]{}

\def\secref#1{Section~\ref{#1}}
\def\leref#1{Lemma~\ref{#1}}

\def\thref#1{Theorem~\ref{#1}}

\def\coref#1{Corollary~\ref{#1}}
\def\figref#1{Figure~\ref{#1}}

\def\appref#1{Appendix~\ref{#1}}
\def\asref#1{A\ref{#1}}

\def\pref#1{P\ref{#1}}

\newcommand{\R}{\mathbb{R}}
\newcommand{\abs}[1]{\left\lvert #1\right\rvert}
\newcommand{\norm}[1]{\left\lVert #1\right\rVert}

\DeclareMathOperator{\E}{\mathbb{E}}
\DeclareMathOperator{\argmin}{arg\,min}
\newcommand{\lin}[1]{\ensuremath \left\langle #1 \right\rangle}

\def\remark{\addtocounter{remark}{1}\def\@currentlabel{\theremark}%
\emph{Remark~\theremark}. } \makeatother
\newcounter{remark}

\newtheorem{assumption}{A}
\newtheorem{property}{P}
 \usepackage[]{color-edits}
 \addauthor{mh}{blue}
 \addauthor{xw}{magenta}
 \addauthor{ne}{brown}

\usepackage{subcaption}

\title{Understanding A Class of Decentralized and Federated Optimization Algorithms: A Multi-Rate Feedback Control Perspective}
 \author{
   Xinwei Zhang, Mingyi Hong, Nicola Elia \thanks{ECE Department, University of Minnesota, \{zhan6234, mhong, nelia\}@umn.edu}.
}

\begin{document}
\maketitle

\begin{abstract}
    Distributed algorithms have been playing an increasingly important role in many applications such as machine learning, signal processing, and control. Significant research efforts have been devoted to developing and analyzing new algorithms for various applications. In this work, we provide a fresh perspective to understand, analyze, and design distributed optimization algorithms. Through the lens of multi-rate feedback control, we show that a wide class of distributed algorithms, including popular decentralized/federated schemes, can be viewed as discretizing a certain continuous-time feedback control system,  possibly with multiple sampling rates, such as decentralized gradient descent, gradient tracking, and federated averaging. This key observation not only allows us to develop a generic framework to analyze the convergence of the entire algorithm class. More importantly, it also leads to an interesting way of designing new distributed algorithms. We develop the theory behind our framework and provide examples to highlight how the framework can be used in practice.
\end{abstract}

\setlength{\abovedisplayskip}{2pt}
\setlength{\belowdisplayskip}{2pt}
\setlength{\abovedisplayshortskip}{2pt}
\setlength{\belowdisplayshortskip}{2pt}

\section{Introduction}\label{sec:intro}

Distributed computation has played an important role in popular applications such as machine learning, signal processing, and wireless communications, partly due to the dramatically increased size of the models and the datasets. In this paper, we consider a distributed system with $N$ agents connected by a graph $\cG =(\mV, \mE)$, each optimizing a smooth and possibility non-convex local function $f_i(x)$. The global optimization problem is formulated as \cite{wang2011control}
\begin{equation}\label{eq:problem}
    \begin{aligned}
        \min_{\xx\in \R^{Nd_x}} & \quad f(\xx) := \frac{1}{N}\sum^N_{i=1} f_i(x_i),\quad 
        \mbox{s.t.} & x_i = x_j ,\; \forall\; (i,j) \in \mE,
    \end{aligned}
\end{equation}
where $\xx\in\mathbb{R}^{N\times d_x}$ stacks $N$ local variables $\xx := [x_1;\dots;x_N]$;  $x_i\in \R^{d_x},~\forall~i\in[N]$.

 This problem has received much attention in recent years, see \cite{chang2020distributed,li2020federated} for a few recent surveys. {Heterogeneous computational and communication resources in the distributed system create a number of different scenarios in distributed learning. In specific, based on the application scenarios, we can roughly classify distributed optimization algorithms into  those that solve Decentralized Optimization (DO) problems, that solve Federated Learning (FL) problems, and those that can achieve optimal resource utilization (OPT). Some of the related works are discussed below.}

a) When solving the DO problems, the agents are typically modeled as nodes on a communication graph, and the communication and computation resources are equally important. So the algorithms alternatingly perform communication and communication steps. For instance, the Decentralized Gradient Descent (DGD) algorithm~\cite{nedic2009distributed,yuan2016convergence} extends gradient descent (GD) to the decentralized setting, where each agent performs one step of local gradient descent and local model average in each round. Other related algorithms such as the DLM~\cite{ling2015dlm}, the Decentralized Gradient Tracking (DGT)~\cite{yuan2020can} and the NEXT~\cite{di2016next} all utilize this kind of alternating updates.

b) The FL problems  typically consider the setting that the clients are directly connected to a parameter-server, and that the communication at the server is the bottleneck of the system. The FL algorithms, such as the well-known FedAvg \cite{bonawitz2019towards}, perform multiple local updates  before one communication step.
However, when the data  is {\it heterogeneous} among the agents, it is difficult for these algorithms to achieve  convergence~\cite{khaled2019first,li2019convergence}.
Recent algorithms such as the FedProx~\cite{li2018federated}, SCAFFOLD~\cite{karimireddy2020scaffold} and FedPD~\cite{zhang2020fedpd} have developed new techniques to improve upon FedAvg. 

c) There have been a number of recent algorithms which are designed to utilize the {\it minimum} computation and/or communication resources, while computing high-quality solutions. They typically perform multiple communication steps before one local update. For examples, in \cite{scaman2017optimal} a multi-step gossip protocol is used to achieve the optimal convergence rate in decentralized convex optimization; the xFilter~\cite{sun2019distributed}
is designed for decentralized non-convex problems, and it implements the Chebyshev filter on the communication graph, which requires multi-step communication, and achieves the optimal dependency on the graph spectrum.

Despite the proliferation of distributed algorithms, there are a few concerns and challenges.  First, for some  hot applications, there are simply {\it too many  algorithms} available, so much so that it becomes difficult to track all the technical details. 
Is it possible to establish some general guidelines to understand the relations between,  and the fundamental principles of, those algorithms that provide similar functionalities? 
Second, much of the recent research on this topic appears to be {\it increasingly focused} on a specific setting (e.g., those mentioned in the previous paragraph). However, an algorithm developed for FL may have already been rigorously developed, analyzed, and tested for the DO setting; and vice versa. Since developing algorithms and performing analyses take significant time and effort,  it is desirable to have some mechanisms in place to reduce the possibility of reinventing the wheel.

\subsection{Contribution of This Work} \label{sub:contribution} We argue that there is a strong demand for a framework of distributed optimization, which can help researchers  and practitioners  {\it understand} algorithm behaviors, {\it predict} algorithm performance, and {\it streamline} algorithm design.  This paper intends to provide such a framework, for a substantial sub-class of distributed algorithms, using tools from multi-rate feedback control systems.
We will first show that a customized continuous-time feedback control system is well-suited to model some key components (such as local computation, inter-agent communication) of distributed algorithms. We then show that when such a continuous-time system is discretized properly (i.e., different parts of the system adopt different sampling rates), it recovers a wide range of distributed optimization algorithms. Finally, we provide a generic convergence result that covers different feedback schemes and discretization patterns. The major benefits of our proposed framework are listed below:

\noindent {\bf 1)} One can easily establish connections between a few sub-classes of distributed algorithms that are developed for different settings. In some sense, they can be viewed as applying different discretization schemes to certain continuous-time control system.

\noindent  {\bf 2)} It helps predict the algorithm performance. On the one hand, once the continuous-time control system and the desired discretization pattern are identified, and some sufficient conditions set forth by our framework are satisfied, one can readily obtain various system parameters  as well as the convergence guarantees. On the other hand, if we found that an existing distributed algorithm performs poorly, it is likely because it does not fall into our framework (an example  is provided to show such a case).

\noindent {\bf 3)} It facilitates new algorithm design. Once the problem setting and the associated requirement are determined, one can start with selecting the desired controllers and feedback schemes for the continuous-time system, followed by finding the appropriate discretization patterns.  The performance of the new algorithm can be again readily obtained from our framework (as discussed in the previous point).

 Note that there are many existing works which analyze optimization algorithms using control theory, but they mainly focus on some very special class of algorithms. For examples, \cite{rossi2006gradient} studies continuous-time gradient flow for convex problems; \cite{wang2011control,sundararajan2021analysis} study continuous-time first-order convex optimization algorithms; \cite{lessard2016analysis,hu2017control,muehlebach2019dynamical} investigate the acceleration approaches including Nesterov and Heavy-ball momentum methods for centralized problems in discrete time and interpret them as discrete-time controllers; \cite{wang2011control,muehlebach2019dynamical} focus on the continuous-time system and ignore the impact of the discretization; \cite{swenson2021distributed,francca2018dynamical,swenson2019distributed} investigate the connection between continuous-time system and discretized gradient descent algorithm, but their approaches and analyses do not generalize to other federated/decentralized algorithms. Further, to our knowledge, none of the above referred works provide insights about relationship between  sub-classes of distributed algorithms (e.g., between DO and FL).

\noindent{\bf Notations, Assumptions.} We introduce some useful assumptions and notations.

First, let $\otimes$ denote the Kronecker product. the incidence matrix $A$ of a graph $\cG$ is defined as: if edge $e(i,j) \in \mE$ connects vertex $i$ and $j$ with $i>j$, then $A_{ei} = 1$, $A_{ej} = -1$ and $A_{ek} = 0,~ \forall k \neq i,j$. Let us use $\cN_i\subset [N]$ to denote the neighbors for agent $i$. For a symmetric matrix $X$, let us use $\lambda(X)$ to denote its eigenvalues. 
Then we can write the constraint of \eqref{eq:problem} in a more compact form:
\begin{align*}
        \min_{\xx\in \R^{Nd_x}} & \quad f(\xx) := \frac{1}{N}\sum^N_{i=1} f_i(x_i),\quad 
        \mbox{s.t.} \quad  (A\otimes I)\cdot \xx = 0.
\end{align*}
For simplicity of notation, the Kronecker products are ignored in the subsequent discussion, e.g., we use $A\xx$ in place of $(A\otimes I)\cdot \xx$. Define the averaging matrix $R := \frac{\bone\bone^T}{N}$ and the average of $x_i$'s as $\bar{\xx} := \frac{\bone^T}{N}\xx = \frac{1}{N}\sum^N_{i=1}x_i.$ Note, we have $R^2 =R$. The consensus error can be written as $[x_1-\bar{\xx}, \dots, x_N-\bar{\xx}] = (I - R)\xx,$ and we have $\nabla f(\bar{\xx}) = \frac{1}{N}\sum^N_{i=1}\nabla f_i(\bar{\xx}).$  The stationary solution of \eqref{eq:problem} is defined as follows:

\begin{definition}[First-order Stationary Point]\label{def:stationary} 
We define the first-order stationary solution and the $\epsilon$-stationary solution respectively, as: 
{\small
\begin{subequations}
\begin{align}
    &\sum^N_{i=1}\nabla f_i\bigg(\frac{1}{N}\sum^N_{i=1}x_i\bigg) = 0, \quad \xx - \frac{\bone\bone^T}{N}\xx = 0, \label{eq:stationarity}\\
    &\norm{\frac{1}{N}\sum^N_{i=1}\nabla f_i\bigg(\frac{1}{N}\sum^N_{i=1}x_i\bigg)}^2+\norm{\xx - \frac{\bone\bone^T}{N}\xx}^2 \leq \epsilon. \label{eq:gap}
\end{align}
\end{subequations}}%
We refer to the left hand side (LHS) of \eqref{eq:gap} as the stationarity gap of \eqref{eq:problem}. 
\end{definition}

We will make the following   assumptions on problem \eqref{eq:problem} throughout the paper:
\begin{assumption}[Graph Connectivity]\label{as:connect}
    The graph is fixed, and strongly connected  at all time $t \in [0,\infty)$, i.e.
    $0$ is a simple eigenvalue of $A^TA$, with corresponding eigenvector $\frac{\bone}{\sqrt{N}}$.
\end{assumption}

{This assumption can be extended to time-varying graphs (denoted as $A(t)$'s), as they can be treated as sub-sampling on a strongly-connected graph $A = \bigcup_t A(t)$. However, to stay focused on the main point of the paper (e.g., build connection of different algorithms from the control perspective) and to reduce notation, we choose to consider the simple static graph $A(t) = A,\;\forall\;t \in [0,\infty)$ in this work.}

Since the agents are connected by a fixed communication graph, we can further define the averaging matrix of the communication graph as $W := I - A^T\rm{diag}(\ww)A$, where $\ww$ is a vector each of whose entries $\ww[e(i,j)]$ is positive, and it corresponds to the weight of edge $e(i,j)$. It is easy to check that $W$ has the following properties: 
\begin{align}\label{eq:weight:property}
  W = W^T, \; \bone^T W = \bone^T, \; W_{ij} \geq 0,\quad\forall e(i,j)\in \mE.  
\end{align}

\begin{assumption}[Lipschitz gradient]\label{as:smooth}
    The $f_i$'s have Lipschitz gradient with constant $L_f$:
    \[\norm{\nabla f_i(x) - \nabla f_i(y)} \leq L_f\norm{x-y}, \quad\forall~x,y\in\R^{d_x},\forall~i\in[N].\]
\end{assumption}

\begin{assumption}[Lower bounded functions]\label{as:lower_bounded}
    Each $f_i$ is lower bounded as:
    \[f_i(x)\geq \underline{f}_i>-\infty, \quad\forall x\in\R^{d_x}, \quad \forall i \in [N].\]
\end{assumption}
\begin{assumption}[Coercive functions]\label{as:coercive}
    Each $f_i$ approaches infinity as $\norm{x}$ approaches infinity:
    \[f_i(x)\rightarrow \infty, \; \text{\rm as}\; \norm{x}\rightarrow \infty,\quad \forall i \in [N].\]
\end{assumption}
\asref{as:lower_bounded} and \asref{as:coercive} imply that there exists at least one globally optimal solution $\xx^{\star}$ for problem \eqref{eq:problem}. Let us denote the corresponding optimal objective as $f^\star:=f(\xx^\star)$.
\section{Continuous-time System}\label{SEC:CONTINUOUS}

We present a continuous-time feedback control system. {We will provide a number of key properties of the controllers and the entire system, to ensure that 
{the system converges to the set of first-order stationary points} with guaranteed speed.
These properties will be instrumental when we subsequently analyze discretized version of the system (hence, various distributed algorithms).}

\begin{minipage}{0.95\linewidth}
      \begin{minipage}{0.49\linewidth}
          \begin{figure}[H]
              \includegraphics[width=\linewidth]{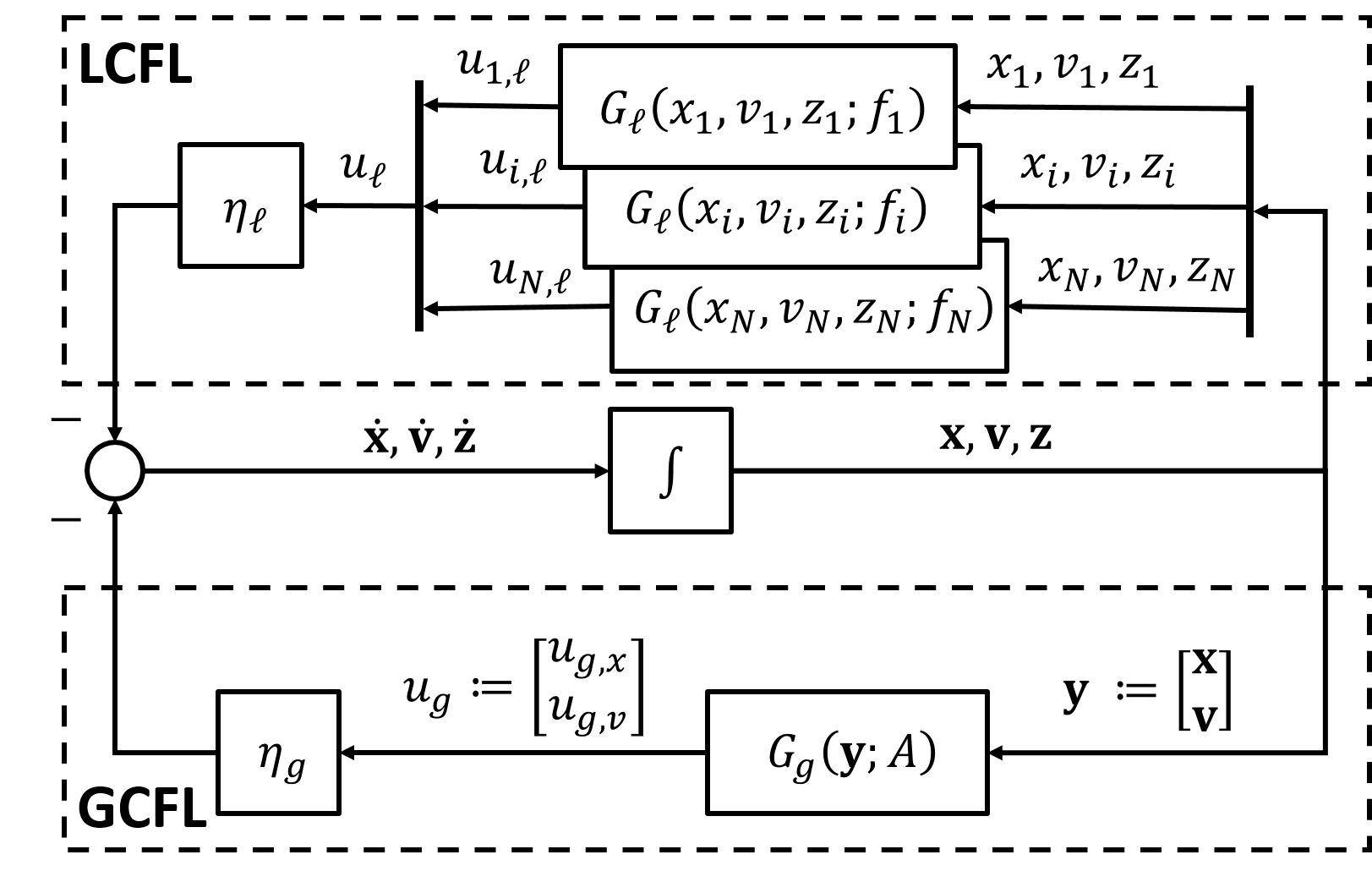}
              \caption{\small The proposed continuous-time double-feedback system for modeling the decentralized optimization problem \eqref{eq:problem}. The system dynamics are given in \eqref{eq:continuous-system}.}\label{fig:continuous-system}
          \end{figure}
      \end{minipage}
      \hfill
      \begin{minipage}{0.49\linewidth}
          \begin{figure}[H]
              \includegraphics[width=\linewidth]{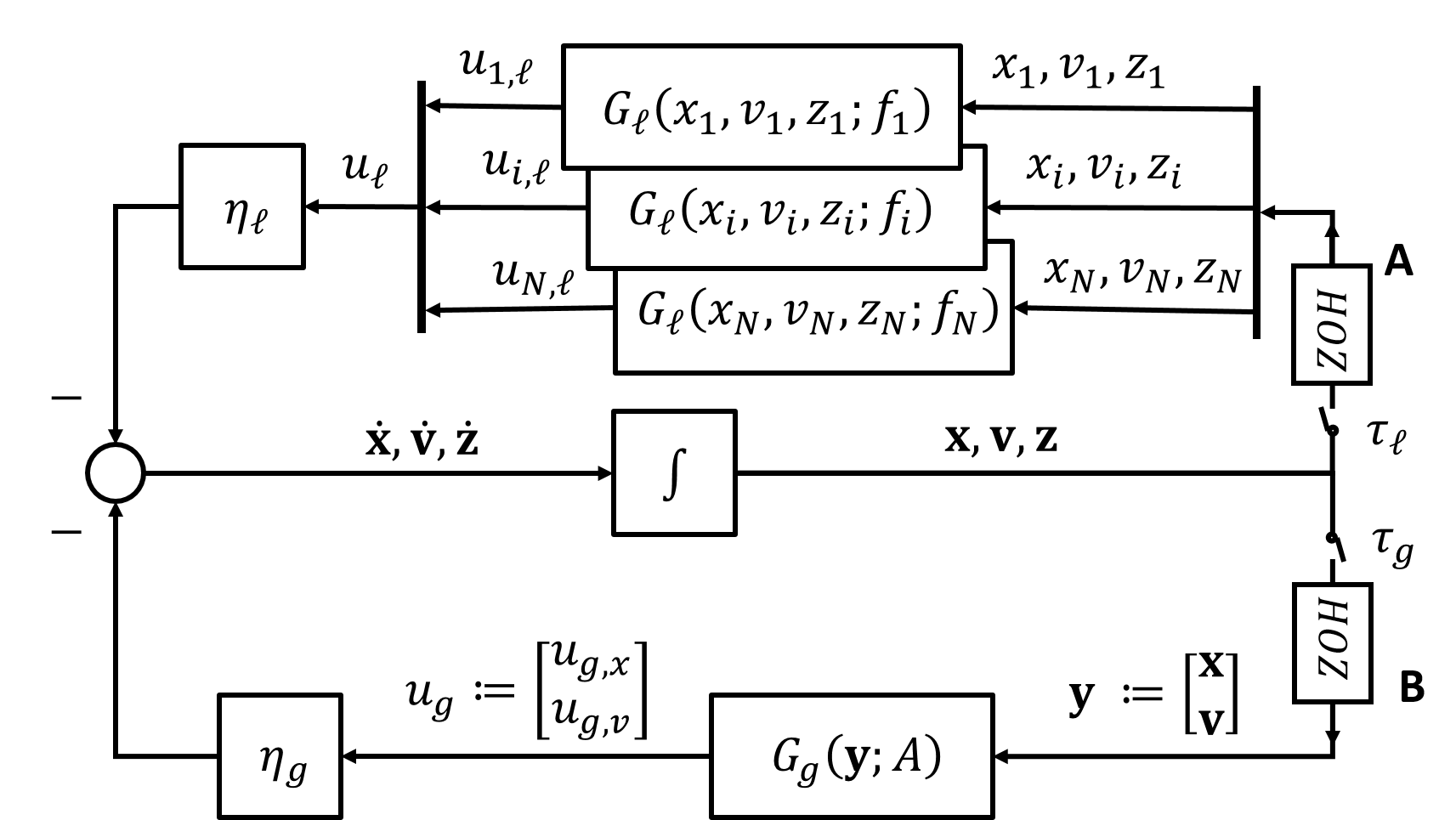}
              \caption{\small Discretized system using ZOH on both the GCFL and LCFL control loops with possibly different sampling times $\tau_g,\tau_\ell$. The system dynamics are given in \eqref{eq:discrete_case_1}-\eqref{eq:discrete_case_3}}\label{fig:CaseI}
          \end{figure}
      \end{minipage}
  \end{minipage}

\subsection{System Description}\label{sec:CTS}
To optimize problem \eqref{eq:problem}, our approach is to design a continuous-time feedback control system, such that the {state variables belong to the set of stationary points of the system if and only if they correspond} to a stationary solution of \eqref{eq:problem}.
Towards this end, define $\xx\in\mathbb{R}^{Nd_x}$ as the main state variable of the system; introduce the {\it global consensus feedback loop} (GCFL) and {\it local computation feedback loop} (LCFL), where the former incorporates the dynamics from multi-agent interactions {and pushes $\xx$ to consensus}, while the latter helps stabilize the system {and finds the stationary solution}.  Specifically, these loops are defined as below:

\noindent\textbullet~{\bf (The GCFL).} 
Define an auxiliary state variable $\vv := [v_1;\dots; v_N]\in\mathbb{R}^{N d_v}$, with $v_i\in\R^{d_v}, \; \forall~i$; define $\yy := [\xx; \vv]\in\mathbb{R}^{N(d_x+d_v)}$; define a feedback controller $G_g(\cdot; A):\mathbb{R}^{N(d_x+d_v)}\to \mathbb{R}^{N(d_x+d_v)}$. Then the GCFL uses $G_g(\cdot; A)$ to operate on $\yy$, to ensure the agents remain coordinated, and their local control variables remain close to consensus;

\noindent \textbullet~{\bf (The LCFL).}  Define an auxiliary state variable $\zz := [z_1;\dots; z_N]\in \mathbb{R}^{N d_z}$, with $z_i\in\R^{d_z}, \; \forall~i$; define a set of feedback controller $G_\ell(\cdot; f_i):\mathbb{R}^{d_x+d_v+d_z}\to \mathbb{R}^{d_x+d_v+d_z}$, one for each agent $i$. Then each agent will use LCFL to operate on its local state variables $x_i$, $z_i$ and $v_i$, to ensure that its local system can be stabilized.

The overall system is described in Fig. \ref{fig:continuous-system}. The detailed description of properties of different controllers, as well as the notations used, will be given in the next sections.

To have a rough idea of how these loops can be mapped to a distributed algorithm, let us consider the PI distributed optimization algorithm~\cite{droge2014continuous}, whose updates are:
\begin{align*}
\begin{split}
    \dot{\xx} &= -k_G \nabla f(\xx) - k_P\cdot(I-W)\cdot \xx -k_Pk_I \vv,\\
    \dot{\vv} &= k_Pk_I\cdot (I-W) \xx.
\end{split}
\end{align*}
The corresponding controllers are given by:{\small
\begin{equation*}
    G_{g}(\xx,\vv; A) := \left[\begin{array}{c}
         (I-W)\cdot\xx + k_I \vv\\
        -k_I\cdot (I-W)\cdot\xx
    \end{array}\right],
    \quad
    G_{\ell}(x_i,v_i,z_i; f_i) :=  \left[\begin{array}{c}
        \nabla f_i(x_i)\\
        0\\
        0\\
    \end{array}\right],
\end{equation*}}%
with $\eta_\ell = k_G$ and $\eta_g = k_P$. Note that auxiliary state variable $\zz$ has not been used in this algorithm.

Next, we describe in detail the properties of the two feedback loops.

\subsection{Global Consensus Feedback Loop}
The GCFL performs inter-agent communication based on the incidence matrix $A$, and it controls the consensus of the global variable $\yy := [\xx;\vv]$.
Specifically, at time $t$, {define the output of the controller as $u_g(t) = G_g(\yy(t);A)$, which can be further decomposed into two outputs $u_g(t) : = [u_{g,x}(t);u_{g,v}(t)],$ one to control the consensus of $\xx$ and the other for $\vv$.}
After multiplied by the control gain $\eta_g(t)>0$, the resulting signal will be combined with the output of the LCFL, and be fed back to local controllers. 

We require that the global controller  $G_g(\cdot; A)$ to have the following  properties:
\begin{property}[Control Signal Direction]\label{as:gg_decrease}
    The output of the controller $G_g$ aligns with the direction that reduces the consensus error, that is:
    \[\lin{(I-R)\cdot \yy, G_g(\yy;A)}\geq C_g \cdot \norm{(I-R)\cdot \yy}^2, \quad  \forall\; \yy,\]
    for some constant $C_g>0$.
    Further, the controller $G_g$ satisfies:
    \[\langle \bone, G_g(\yy;A)\rangle = 0, \quad \forall\; \yy, {~~\mbox{\rm which implies}~~ \langle \bone, u_g(t)\rangle =0,\quad \forall~t.}\]
\end{property}

\begin{property}[Linear Operator]\label{as:gg_linear}
    The controller $G_g$ is a linear operator of $\yy$, that is, we have $G_g(\yy;A) = W_A \yy$ for some matrix $W_A\in \R^{N(d_x+d_v)}$ parameterized by $A$, and its eigenvalues satisfy: $\abs{\lambda(W_A)} \in [0,1]$.
\end{property}
Combining \pref{as:gg_decrease} and \pref{as:gg_linear}, we have
$\lin{\bone, W_A} =0$, which indicates $R\cdot W_A = 0$ and the eigenvectors of $W_A$ are orthogonal to the ones of $R$. Further we have
\begin{align*}
    \norm{(I-R)\yy}^2 -\norm{G_g(\yy;A)}^2 &= \yy^T((I-R)^2 - W_A^2) \yy\\
    & = \yy^T\left(I-2R+R -W_A^2\right)\yy  = \yy^T(I - (R+W_A^2))\yy.
\end{align*}
Notice the eigenvectors of $R$ and $W_A$ are orthogonal and all eigenvalues are in $[0,1]$, so we have matrix $I-(R+W_A^2)\succeq 0$. Thus $\yy^T(I - (R+W_A^2))\yy \geq 0$ and $\norm{(I-R)\yy}^2 \geq \norm{G_g(\yy;A)}^2$.
Therefore, we have:
\begin{equation}\label{eq:gg_bound}
    C^2_g\norm{(I-R)\cdot \yy}^2\leq\norm{G_{g}(\yy;A)}^2 \leq \norm{(I-R)\cdot \yy}^2, \;\;\mbox{and}\;\;  R \cdot W_A =0.
\end{equation}%

It is easy to check that both \pref{as:gg_decrease} and \pref{as:gg_linear} hold in most of the existing consensus-based algorithms. For example, when the communication graph is strongly connected, we can choose
$G_g(\yy; A) = (I - W)\cdot \yy$.
It is easy to verify that,  $C_g= 1 - \lambda_2(W)$ where $\lambda_2(\cdot)$ denotes the eigenvalue withe the second largest magnitude~\cite{yuan2016convergence,chang2020distributed}. As another example, consider the accelerated averaging algorithms~\cite{ghadimi2011accelerated}, where we have 
\begin{align}
    &G_g(\yy, A) = \left[\begin{array}{cc}
    I - (c+1)\cdot W & c\cdot I\\
    -I & I
\end{array}\right] \left[\begin{array}{c}
    \xx\\
    \vv
\end{array}\right], \;
\mbox{with} \; c := \frac{1-\sqrt{1-\lambda_2(W)}}{1+\sqrt{1-\lambda_2(W)^2}}\nonumber.
\end{align}
In this case, one can verify that $C_g = 1-\frac{\lambda_2(W)}{1+\sqrt{1-\lambda_2(W)^2}} \geq 1-\lambda_2(W).$

By using \pref{as:gg_decrease}, we can follow the general analysis of averaging systems~\cite{olshevsky2009convergence}, and show that the GCFL will behave {\it as expected},  that is, if the system {\it only} performs GCFL and shuts off the LCFL, then the consensus can be achieved. More precisely, assuming that $\eta_\ell(t) = 0, \eta_g(t) = 1$, then under \pref{as:gg_decrease}, the local state  $\yy$ converges to the average of the initial states linearly:
    \begin{equation}\label{eq:gg_convergence}
        \norm{(I-R)\cdot \yy(t)}^2 \leq \e^{-2C_g t}\norm{(I-R)\cdot \yy(0)}^2. 
    \end{equation}
For completeness we include the derivation in the supplementary Sec. \ref{app:proof:gg}.

\subsection{The Local Computation Feedback Loop}

The LCFL optimizes the local function $f_i(\cdot)$'s for each agent. At time $t$, the $i$th local controller takes the local variables $x_i(t), v_i(t), z_i(t)$ as inputs and produces a local control signal. To describe the system, let us denote the output of the local controllers as $u_{i,\ell}(t) = G_{\ell}(x_i(t), v_i(t), z_i(t);f_i), \;\forall~i \in [N]$; further decompose it into three parts:
\[u_{i,\ell}(t) := [u_{i,\ell,x}(t); u_{i,\ell,v}(t); u_{i,\ell,z}(t)].\]
Denote the concatenated local controller outputs as:
$ u_{\ell,x}(t) := [u_{1,\ell,x}(t);\dots;u_{N,\ell,x}(t)]$,
and define $u_{\ell,v}(t),  u_{\ell,z}(t)$ similarly. Note that we have assumed that all the agents use the same local controller $G_{\ell}(\cdot; \cdot)$, but they are parameterized by different $f_i$'s. 
After multiplied by the control gain $\eta_\ell(t)>0$, the resulting signal will be combined with the output of GCFL, and be fed back to the local controllers. 

The local controllers are designed to have the following properties: 
\begin{property}[Lipschitz Smoothness]\label{as:gl_smooth} The controller is Lipschitz continuous, that is:
    \begin{align*}
        &\norm{G_{\ell}(x_i,v_i,z_i;f_i)- G_{\ell}(x'_i,v'_i,z'_i;f_i)}\leq L \norm{[x_i;v_i;z_i]-[x'_i;v'_i;z'_i]}, \\
        &\forall~i\in[N],\; x_i,x'_i\in\R^{d_x},\; v_i,v'_i\in\R^{d_v},\; z_i,z'_i\in\R^{d_z}.
    \end{align*}
\end{property}
\begin{property}[Control Signal Direction and Size]\label{as:gl_decrease}
    {The local controllers are designed such that there exist initial values $x_i(t_0)$, $v_i(t_0)$ and $z_i(t_0)$ ensuring} 
    that the following holds:
    \[\lin{\nabla f_i(x_i(t)), u_{i,\ell,x}(t)} \geq \alpha(t)\cdot \norm{\nabla f_i(x_i(t))}^2, \quad \forall\; t\geq t_0,\]
    where $\alpha(t)>0$ satisfies $\lim_{t\rightarrow \infty}\int^t_{t_0}\alpha(\tau)\md \tau \rightarrow \infty$.

    Further, for any given $x_i$, $v_i$, $z_i$, the sizes of the control signals are upper bounded by those of the local gradients. That is, for some positive constants $C_x$, $C_v$ and $C_z$:
    \[\norm{u_{i,\ell,x}}\leq C_x\norm{\nabla f_i(x_i)},\;\; \norm{u_{i,\ell,v}}\leq C_v\norm{\nabla f_i(x_i)},\;\; \norm{u_{i,\ell,z}}\leq C_z\norm{\nabla f_i(x_i)}.\]
\end{property}
Let us comment on these properties. \pref{as:gl_smooth} is easy to verify for a given realizations of the local controllers; \pref{as:gl_decrease} abstracts the convergence property of the local optimizer. This property implies that the update direction $-u_{i,\ell,x}(t)$ points to a direction that decreases the local objective. {Note that it is postulated that $x_i, v_i$ and $z_i$ are initialized properly, because in some of the cases, improper initial values lead to non-convergence of the local controllers (or equivalently, the local algorithm). For example, for accelerated gradient descent method~\cite{bubeck2015geometric,ye2020multi}, $z_i(t_0)$ should be initialized as $\nabla f_i(x_i(t_0))$.}

By using \pref{as:gl_decrease}, we can follow the general analysis of the gradient flow algorithms (e.g., {~\cite{orvieto2019continuous}}), and show that the LCFL will behave {\it as expected}, in the sense that the  agents can properly optimize their local problems. More precisely, assume that $\eta_g(t) =0, \eta_\ell(t) = 1$, that is, the system shuts off the GCFL.
   {Assume that $G_{\ell}(\cdot;\cdot)$ satisfies \pref{as:gl_decrease}}, then each local system  produces $x_i(t)$'s that satisfy:
    \begin{equation}\label{eq:gl_convergence}
        \min_\tau\norm{\nabla f_i(x_i(t+\tau))}^2\leq \gamma(\tau)\cdot(f_i(x_i(t)) - \underline{f}_i),
    \end{equation}
where $\{\gamma(\tau)\}$ is a sequence of positive constants satisfying: 
    \begin{align}
        \gamma(\tau) = \frac{1}{\int^t_{0}\alpha(\tau)\md \tau}\rightarrow 0, \quad \mbox{as } \tau \rightarrow \infty. 
    \end{align}
 We include the proof of the above result in the supplementary Sec. \ref{app:proof:gl}.

To close this subsection, we note that the continuous-time system we have presented so far (cf.~\figref{fig:continuous-system}) can be described using the following dynamics:
\begin{align}
    \dot{\vv}(t) &= -\eta_g(t)\cdot u_{g,v}(t) - \eta_\ell(t) \cdot u_{\ell,v}(t)\nonumber\\
    \dot{\xx}(t) &= -\eta_g(t)\cdot u_{g,x}(t) - \eta_\ell(t) \cdot u_{\ell,x}(t), \quad \dot{\zz}(t) = -\eta_\ell(t) \cdot u_{\ell,z}(t).\label{eq:continuous-system}
\end{align}%
Additionally, throughout the paper, we will use  $u_g$ and $G_g$, $u_\ell$ and $G_{\ell}$ interchangeably. 

\subsection{Convergence Properties}
We proceed to analyze the convergence of the continuous-time system. Towards this end, we define an energy-like function: 
\begin{align}\label{eq:energy}
\cE(t) := f(\bar{\xx}(t)) - f^\star + \frac{1}{2}\norm{(I-R)\cdot \yy(t)}^2.
\end{align}
Note that $\cE(t)\geq 0$ for all $t\ge 0$.
It follows that its derivative can be expressed as: {\small
     \begin{align}\label{eq:E:derivative}
     \dot{\cE}(t) = -\lin{\nabla f(\bar{\xx}(t), \eta_{\ell}(t)\cdot\frac{\bone^T}{N}u_{\ell,x}(t)} + \lin{(I-R)\cdot\yy(t), \eta_g(t)u_{g}(t)+ \eta_\ell(t) u_{\ell,y}(t)}. 
   \end{align}}%
In the following, we study the convergence of ${\cal E}(t)$ and characterize the set of stationary points that the states satisfy $\dot{{\cal E}}(t)=0$. We do not attempt to analyze the stronger property of {\it stability}, not only because such kind of analysis can be challenging due to the non-convexity of the local functions $f_i(\cdot)$'s, but more importantly, analyzing the convergence of ${\cal E}(t)$ is already sufficient for us to understand the convergence of the state variable $\xx$ to the set of stationary solutions of problem \eqref{eq:problem}, as we will show shortly.

To proceed, we require that the system satisfies the following property:
\begin{property}[Energy Function Reduction]\label{as:energy}
The derivative of the energy function, $\dot{\cE}(\cdot)$ as expressed in \eqref{eq:E:derivative}, satisfies the following: {\small
     \begin{align}\label{eq:E:dynamic}
     &{-\int^t_{0}\left(\lin{\nabla f(\bar{\xx}(\tau), \eta_{\ell}(\tau)\cdot\frac{\bone^T}{N}u_{\ell,x}(\tau)} + \lin{(I-R)\cdot\yy(\tau), \eta_g(\tau)u_{g}(\tau)+ \eta_\ell(\tau) u_{\ell,y}(\tau)}\right)\md \tau\nonumber}\\
      &\leq -\int_{0}^{t}\left(\gamma_1(\tau)\cdot  \bigg\|\nabla f(\bar{\xx}(\tau))\bigg\|^2 + \gamma_2(\tau) \cdot \norm{(I - R)\cdot \yy(\tau)}^2\right) d\tau,
   \end{align}
}%
where $\gamma_1(\tau), \gamma_2(\tau) >0$ are some time-dependent coefficients.

\end{property}
\pref{as:energy} is a property about the entire continuous-time system. {Although one could show that by using \pref{as:gg_decrease} - \pref{as:gl_decrease}, and by selecting $\eta_g(t)$ and $\eta_\ell(t)$ appropriately, this property can be satisfied with some {\it specific} $\gamma_1(\tau)$ and $\gamma_2(\tau)$ (cf. Corollary \ref{cor:continuous}.), here we still list it as an independent property, because at this point we want to keep the choice of $\gamma_1(\tau)$, $\gamma_2(\tau)$ general; please see Sec. \ref{sub:summary} for more detailed discussion.} 

Next, we will show that under \pref{as:energy}, the continuous-time system will converge to {the set of stationary points}, and that $\xx$ will converge to the set of stationary solutions of problem \eqref{eq:problem}.
\begin{theorem}\label{theorm:energy:continuous}
    Suppose \pref{as:energy} holds true. Then we have the following results:
    
   \noindent {\bf 1)} {Further, suppose that \pref{as:gg_decrease}, \pref{as:gg_linear} and \pref{as:gl_decrease} hold, then $\dot{\cE} = 0$ implies that the corresponding state variable $\xx_s$ is bounded, and the following holds:
    \begin{align}
        \dot{\xx}_s = 0, \quad \dot{\vv}_s = 0,\quad \dot{\zz}_s = 0, \quad u_g = 0,\quad u_\ell = 0.
    \end{align}
    Additionally, let us define the set $\mS$ as below:
     \[ \mS:=\left\{\vv, \zz \left| \begin{array}{l}
            \eta_\ell u_{\ell,v}+ \eta_gu_{g,v} = 0, \;u_{\ell,z} = 0,\; \eta_\ell u_{\ell,x}+\eta_g u_{g,x} = 0\\
        \end{array}\right.\right\}.
    \]
    If we assume that $\mS$ is compact for any state variable $\xx$ that satisfies the stationarity condition \eqref{eq:stationarity},
    then the auxiliary state variables $\{\vv(t)\}$ and $\{\zz(t)\}$ are also bounded. 
    
    \noindent {\bf 2)} The control system asymptotically {converges to the set of stationary points}, in that $\xx(t)$ is bounded $\forall t\in[0,\infty)$, and {$\dot{\cE}\to 0$}. Further, the stationary gap  \eqref{eq:gap} can be upper bounded by the following:} {\small\begin{align}\label{eq:ct_decrease}
         &\min_t\bigg\{\norm{\nabla f(\bar{\xx}(t))}^2 +\norm{(I - R)\cdot \yy(t)}^2\bigg\}= \cO\left(\max\left\{\frac{1}{\int^T_{0}\gamma_1(\tau)\md \tau}, \frac{1}{\int^T_{0}\gamma_2(\tau)\md \tau}\right\}\right).
    \end{align}}%
\end{theorem}
\begin{proof} To show part (1), consider a set of states $\xx_s, \vv_s, \zz_s$ in which $\dot{\cE}(\xx_s,\vv_s) = 0$. \pref{as:energy} implies that $\nabla f(\bar{\xx}_s) = 0$, and \pref{as:gl_decrease} implies $\norm{u_\ell} \leq (C_x+C_v+C_z)\norm{\nabla f(\bar{\xx}_s)} = 0$. Similarly, with \pref{as:gg_decrease} and \pref{as:gg_linear} we have that $\lin{u_{g}, (I-R)\yy_s} = 0$ and $\bone^Tu_g = 0$ so $u_g = 0$. Therefore $\dot{\xx}_s = 0, \dot{\vv}_s=0, \dot{\zz}_s = 0$. Combining $\nabla f(\bar{\xx}_s) = 0$ and \asref{as:coercive} implies that $\xx_s$ is bounded.
Note that the value of $\vv(t), \zz(t)$ may not be bounded, even if the system converges to a stationary solution. Using the compactness assumption on the set $\mathbf{S}$, it is easy to show that $\vv(t), \zz(t)$ are also bounded.

To show part (2), we can integrate $\dot{\cE}(t)$ from $t=0$ to $T$ to obtain:
\[\int^T_{0}\gamma_2(t)\norm{(I-R)\cdot \yy(t)}^2\md t + \int^T_{0} \gamma_1(t)\norm{\nabla f(\bar{\xx}(t))}^2\md t\leq \cE(0) - \cE(T),\]
divide both sides by $\int^T_{0}\gamma_1(t)\md t$ or $\int^T_{0}\gamma_2(t)\md t$, we obtain\eqref{eq:ct_decrease}. By \pref{as:energy} we know $\int^t_{0}\dot{\cE}(\tau)\md\tau \leq 0,\; \forall t$, but since $\cE(t)\geq 0$, it follows that $\lim_{t\rightarrow \infty}\dot{\cE}(t) = 0$.
\end{proof}

Note that without the  compactness assumption,  $\vv$ and $\zz$ can be unbounded. As an example, FedYogi uses AdaGrad for LCFL~\cite{reddi2020adaptive} where $\vv(t)$ accumulates the norm of the gradients and does not satisfy the compactness assumption, so $\lim_{t\to\infty}\vv(t)\to\infty$. Although such unboundedness does not affect the convergence of the main state variable in part (2), from the control perspective it is still desirable to have a sufficient condition to guarantee the boundedness of all state variables.

{Part (2) of the above result indicates that if \pref{as:energy} is satisfied, not only will the system asymptotically {converge to the set of stationary points}, but more importantly, we can use $\{\gamma_1(t),\gamma_2(t)\}$ to characterize the rate in which the stationary gap of problem \eqref{eq:problem} shrinks. This result, although rather simple, will serve as the basis for our subsequent system discretization analysis.}

\subsection{Summary}\label{sub:summary} So far, we have completed the setup of the continuous-time feedback control system, by specifying the state variables, the feedback loops, and by introducing a few desired properties of the local controllers and the entire system. In particular, we show that property \pref{as:energy} is instrumental in ensuring {that the system converges to the set of stationary points}. However, there are {two} key questions remain to be answered:

\noindent {\bf (i)} How to ensure property \pref{as:energy} for a given continuous-time feedback control system?

\noindent {\bf (ii)} How to map the continuous-time system to a distributed optimization algorithm, and to transfer the convergence guarantees of the former to the latter?

There are two different ways to answer question {\bf (i)}. First, for a {\it generic} system that satisfies properties \pref{as:gg_decrease} -- \pref{as:gl_decrease}, we can show that when the control gains $\eta_g(t), \eta_\ell(t)$ are selected appropriately, then \pref{as:energy} will be satisfied; see Corollary \ref{cor:continuous} below. 
\begin{corollary}
\label{cor:continuous}
    Suppose that \pref{as:gg_decrease}, \pref{as:gl_smooth}, \pref{as:gl_decrease} are satisfied.
     {By choosing $\eta_g(t) = 1, \eta_\ell(t) = \cO(1/\sqrt{T}))$, \pref{as:energy} holds true with $\gamma_1(t) = \cO(\eta_\ell(t))$, $\gamma_2(t) = \cO(1)$} Further,
    \begin{align*}
        &\min_{t}\left\{\norm{\nabla f(\bar{\xx}(t))}^2 + \norm{(I-R)\cdot \yy(t)}^2\right\} = \cO\left(\frac{1}{\int^T_{0}\eta_\ell(\tau)\md \tau}\right) = \cO\left(\frac{1}{\sqrt{T}}\right).
    \end{align*}
\end{corollary}
The proof of the above result follows the steps used in analyzing  distributed gradient flow algorithm~\cite{swenson2019distributed}; see the supplementary Sec.~\ref{app:proof:continuous}.

The second answer to question {\bf (i)} is that, one can also verify \pref{as:energy} in a case-by-case manner for individual systems. In this way, it is possible that one can obtain  larger gains $\eta_{\ell}(t), \eta_{g}(t)$, hence larger coefficients $\gamma_1(t)$ and $\gamma_2(t)$ to further improve the convergence rate estimate. In fact, verifying property \pref{as:energy}, and computing the corresponding coefficients is a key step in our proposed analysis framework for distributed algorithms. Shortly in Sec. \ref{sec:application}, we will provide an example to showcase how to verify that the continuous-time system which corresponds to the DGT algorithm satisfies \pref{as:energy} with $\gamma_1(t)=\cO(1)$ and $\gamma_2(t)=\cO(1)$, leading to a convergence rate of  $\cO(1/T)$. 

On the other hand, the answer to question {\bf (ii)} is more involved, so this question will be addressed in the main technical part of this work to be presented shortly. Generally speaking, one needs to discretize the continuous-time system properly to map the system to a particular distributed algorithm. Further, one needs to utilize all the properties \pref{as:gg_decrease} -- \pref{as:energy}, and carefully select the discretization intervals, to ensure that the resulting discretized systems perform appropriately.
\section{System Discretization}\label{SEC:DISCRETIZATION}
\begin{wrapfigure}{r}{0.3\textwidth}
\vspace{-0.8cm}
    \begin{center}        \includegraphics[width=0.28\textwidth]{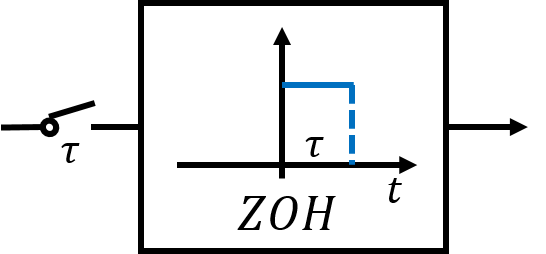}
        \caption{The discretization block that has a switch and a Zero-Order Hold.}\label{fig:ZOH}
    \end{center}
    \vspace{-0.8cm}
\end{wrapfigure}
In this section, we discuss how to use system discretization to map the continuous-time system introduced in the previous section to distributed algorithms.

\subsection{Modeling the Discretization}
Typically, a continuous-time system is discretized by using a switch that samples the input with sample time $\tau$, followed by a zeroth-order hold (ZOH) that keeps the signal constant between the consecutive sampling instances~\cite{kuo1980digital}; see \figref{fig:ZOH}. 

Now let us use ZOH to discretize the continuous-time system depicted in Fig. \ref{fig:continuous-system}. We will place the ZOH before the variables enter the controllers, i.e., at points A and B in Fig. \ref{fig:CaseI}.
Note that, the original continuous-time system can be discretized in many different ways, by customizing the sampling rates for the discretization blocks. Each of these discretization scheme will correspond to a {\it multi-rate} control system, in which different parts of the system run on different sampling rates. To describe such kinds of multi-rate system, let us define the {\it sampling intervals} for the GCFL and LCFL as $\tau_g$ and $\tau_\ell$, respectively. Then we can consider the following five cases:

\noindent $\bullet$ {\bf Case I.} $\tau_g > 0, \tau_\ell = 0$, the GCFL is discretized while the LCFL is not;

\noindent $\bullet$ {\bf Case II.} $\tau_g = 0, \tau_\ell > 0$, the GCFL remains continuous while the LCFL is not;
    
\noindent $\bullet$ {\bf Case III.} $\tau_g = \tau_\ell > 0$, the GCFL and LCFL are discretized with the same rate;

\noindent $\bullet$ {\bf Case IV.} $\tau_g > \tau_\ell > 0$, both the GCFL and LCFL are discretized, while the local computation loop is updated more frequently;
    
\noindent $\bullet$ {\bf Case V.} $\tau_\ell > \tau_g > 0$, both GCFL and LCFL are discretized, while the global communication loop is updated more frequently.

We note that the systems in cases I and II are {\it sampled data} systems which has both continuous-time part and discretized part, while systems in cases IV, V are {\it multi-rate discrete-time} systems. Further, the entire system in case III operates on the same sampling rate. For simplicity, we refer both sampled data system and fully discretized system as {\it discretized system} in the rest of the paper.

\subsection{Distributed Algorithms as Multi-Rate Discretized Systems}
In this section, we make the connection between {\it sub-classes} of distributed algorithms and different discretization patterns. For convenience, let $t_k$ denote the times at which the inputs of the ZOHs get sampled by {\it both} the global and local controllers.

\noindent{\bf Case I} ($\tau_g > 0, \tau_\ell = 0$):
The system can be described as:
\begin{align}\label{eq:discrete_case_1}
\begin{split}
    \dot{\vv}(t) &= -\eta_g(t)\cdot u_{g,v}(t_k) - \eta_\ell(t) \cdot u_{\ell,v}(t)\\
    \dot{\xx}(t) &= -\eta_g(t)\cdot u_{g,x}(t_k) - \eta_\ell(t) \cdot u_{\ell,x}(t), \quad \dot{\zz}(t) = -\eta_\ell(t) \cdot u_{\ell,z}(t).
\end{split}
\end{align}
Due to the use of ZOH, during an interval $[t_k,t_k+\tau_g)$, the control signals $u_{g,v}$ and $u_{g,x}$ are fixed. By \pref{as:gl_decrease}, it follows that the dynamic system finds a stationary point of the local problem satisfying $\dot{x}_i = 0,\; \forall~i$, that is $\eta_\ell(t)\cdot u_{\ell,x}(t) + \eta_g(t)\cdot u_{g,x}(t_k) = 0$. This is the stationary solution of the following perturbed problem for each agent:
 \begin{align}\label{eq:local:problem}
    \min_{x_i}{\; \widetilde{f}_i(x_i)} := f_i(x_i) + \frac{\eta_g(t)}{\eta_l(t)}\lin{u_{i,g,x}(t_k),x_i}. 
 \end{align}
Using \eqref{eq:gl_convergence}, it follows that the above problem is optimized to satisfy: 
$$\min_{t\in[t_k, t_k+\tau_g]}\norm{\nabla \widetilde{f}_i (x_i(t))}^2 \leq \gamma(\tau_g)\cdot \left( \widetilde{f}_i(x_i(t_k)) - \widetilde{f}_i(x_i(t_k+\tau_g)\right),$$
with $\gamma(\tau_g) = \frac{1}{\int^{\tau_g}_0 \alpha(t)\md t}$. That is, we obtain a $\gamma(\tau_g)$-stationary solution for the local problem \eqref{eq:local:problem}.
This system has the same form as the distributed algorithms that require to solve some local problems to a given accuracy, before any local communication steps take place; see for examples FedProx~\cite{li2018federated}, FedPD~\cite{zhang2020fedpd} and NEXT~\cite{di2016next}.

\noindent{\bf Case II} ($\tau_g = 0, \tau_\ell > 0$): 
The system can be described as:
\begin{align}\label{eq:discrete_case_2}
\begin{split}
    \dot{\vv}(t) &= -\eta_g(t)\cdot u_{g,v}(t) - \eta_\ell(t) \cdot u_{\ell,v}(t_k)\\
    \dot{\xx}(t) &= -\eta_g(t)\cdot u_{g,x}(t) - \eta_\ell(t) \cdot u_{\ell,x}(t_k), \quad \dot{\zz}(t) = -\eta_\ell(t) \cdot u_{\ell,z}(t_k).
\end{split}
\end{align}
During $[t_k,t_k+\tau_\ell)$ the control signals $u_{\ell,x}(t), u_{\ell,v}(t), u_{\ell,z}(t)$ are fixed. By \pref{as:gg_decrease}, the system finds a solution $\dot{\yy} = 0$, which implies  that $-\eta_g(t)\cdot u_{g,x}(t) - \eta_\ell(t) \cdot u_{\ell, x}(t_k) = 0$. By \eqref{eq:gg_convergence}, in $[t_k,t_k+\tau_\ell)$, the system optimizes the following network problem:
\[\min_{\yy}g(\yy) :=  \norm{(I-R)\cdot \yy + ({\eta_\ell(t)}/{\eta_g(t)})\cdot u_{\ell,y}(t_k)}^2,\]
and obtain a solution that satisfies:
$\norm{\nabla g(\yy(t_k+\tau_\ell))}^2  \le \e^{-2C_g\tau_\ell}g(\yy(t_k)).$
This system is related to those algorithms that achieve the optimal communication complexity~\cite{scaman2017optimal,sun2019distributed}. In these algorithms, it is often the case that some networked problems are solved (to sufficient accuracies) between two local optimization steps. 

\noindent{\bf Case III} ($\tau_g = \tau_\ell > 0$): The system is discretized with a single sampling interval. {Once sampled at $t_k$, the controllers' inputs remain to be $\xx(t_k),\vv(t_k),\zz(t_k)$ during the sampling interval, the output of the controllers are also kept constant $u_g(t) = u_g(t_k), u_\ell(t) = u_\ell(t_k), \forall t\in[t_k,t_k+\tau_g)$.} So the system update can be written as: 
\begingroup
\allowdisplaybreaks
\begin{align}\label{eq:discrete_case_3}
    \xx(t_{k+1}) &= \xx(t_k) - \eta'_\ell(t_k)\cdot  u_{\ell,x}(t_k) - \eta'_g(t_k)\cdot  u_{g,x}(t_k),\nonumber\\
    \vv(t_{k+1}) &= \vv(t_k) - \eta'_\ell(t_k)\cdot u_{\ell,v}(t_k) - \eta'_g(t_k) \cdot u_{g,v}(t_k),\\
    \zz(t_{k+1}) &= \zz(t_k) - \eta'_g(t_k) \cdot u_{\ell,z}(t_k),\nonumber
\end{align}
\endgroup%
where $\eta'_\ell(t_k) = \int^{t_k+\tau_g}_{t_k}\eta_\ell(t)\md t$, $\eta'_g(t_k) = \int^{t_k+\tau_g}_{t_k}\eta_g(t)\md t$.
The above updates are equivalent to many existing decentralized optimization algorithms, such as DGD, 
DLM, which perform one step local update, followed by one step of communication. 

\noindent{\bf Case IV} ($\tau_g > \tau_\ell > 0$): We assume that $\tau_g = Q\cdot \tau_\ell$, which means that each agent performs $Q$ steps of local computation between every two communication steps. This update strategy is related to the class of (horizontal) federated learning algorithms~\cite{bonawitz2019towards}.

\noindent{\bf Case V} ($\tau_\ell > \tau_g > 0$): We assume that $\tau_\ell = K\cdot \tau_g$, that the agents perform $K$ steps of communication between two local computation steps. {Although $K$ can be arbitrary, in practice it is typically chosen large enough so that certain network problem is solved approximately; therefore in practice this case is closely related to Case II.} 

We summarize the above discussion in Table~\ref{tab:discretization}, and provide some example algorithms for each case.
In Sec. \ref{sec:application}, we will specify the controllers for these algorithms so that we can precisely map them to a discretization setting.  
It is important to note that the connection identified here is useful in helping predict algorithm performance, as well as facilitates new algorithm design; see the related discussions in Sec.\ref{sub:contribution}, points 2) and 3).
However, these benefits can be realized only if there is a systematic way of transferring the theoretical results from the continuous-time system to different discretization settings. This will be discussed in detail in the next subsection.

\begin{table*}
\centering\small
    \begin{tabular}{c|c|cc|c}
        \hline
        {\bf Case}& ${\bf \tau_\ell, \tau_g}$  & {\bf Comm.}     & {\bf Comp.}  & {\bf Related Algorithm}     \\
        \hline
        I   &$\tau_g > 0, \tau_\ell = 0$& Slow & Continuous &  NEXT~\cite{di2016next}, FedProx~\cite{li2018federated}, NIDS~\cite{li2019decentralized}\\\
        II  &$\tau_g = 0, \tau_\ell > 0$& Continuous & Slow &  MSDA~\cite{scaman2017optimal}, xFilter~\cite{sun2019distributed},AGD~\cite{ye2020multi}\\
        III &$\tau_g = \tau_\ell > 0$ &\multicolumn{2}{c|}{Same rate} & DGD~\cite{yuan2016convergence}, DGT~\cite{yuan2020can}\\
        IV  &$\tau_g > \tau_\ell > 0$ & Slow & Fast &  Local GD~\cite{khaled2019first}, Scaffold~\cite{karimireddy2020scaffold}\\
        V   & $\tau_\ell > \tau_g > 0 $& Fast & Slow  & Same as Case II\\
        \hline
    \end{tabular}
    \caption{\small Summary of discretization settings, and the corresponding distributed  algorithms.}\label{tab:discretization}
\end{table*}

\subsection{Convergence of Discretized Systems}
Next, we leverage the convergence results of the continuous-time system to  analyze distributed algorithms. The key challenge is to properly deal with the potential instability introduced by discretization. The proof of this subsection is relegated to \appref{app:proof:case1} -- \ref{app:proof:case345}.

\noindent{\bf Discretized Communication ($\tau_g>0, \tau_{\ell}=0$, Case I).} Recall that the system dynamics are given in \eqref{eq:discrete_case_1}. Let us first show how the sampling error affects $\dot{\cE}$.

\begin{lemma}[$\dot{\cE}$ in Case I]\label{le:case1}
    Suppose the GCFL and LCFL satisfy  \pref{as:gg_decrease}-\pref{as:energy}, and consider the discretized system with  $\tau_\ell = 0, \tau_g>0$. Then we have the following:
    {\small
    \begin{equation}
        \begin{aligned}
        \int^t_{0}\dot{\cE}(\tau)\md \tau & \leq  \int^t_{0}-\underbrace{\left(\gamma_1(\tau)- C_{11}\right)}_{\rm := \hat{\gamma}_1(\tau) }\cdot\norm{\nabla f(\bar{\xx}(\tau))}^2 \md \tau\\
        & \qquad + \int^t_0- \underbrace{\left(\frac{\gamma_2(\tau)}{2} - C_{11}\right)}_{\rm :=\hat{\gamma}_2(\tau)} \cdot\norm{(I - R)\cdot \yy(\tau)}^2\md \tau,
        \end{aligned}
    \end{equation}}%
    where $C_{11}:= \frac{q^2_{\max}}{2\gamma_2(\tau)}$ and $q_{\max} := \exp\left\{\sqrt{2}\tau_g\cdot\left(\sqrt{C_x^2+C_v^2}\eta_\ell(t) \cdot\left(1+\frac{L_f}{N}\right)^2\right)\right\}-1$.
\end{lemma}
The lemma shows that discretizing the communication with sufficiently small $\tau_g$ leads to a small $q_{\max}$, which preserves the desired descent property. 

\noindent{\bf Discretized Computation ($\tau_\ell>0, \tau_g=0$, Case II).}
Recall that the system dynamics can be expressed in \eqref{eq:discrete_case_2}. We have the following result:
\begin{lemma}[$\dot{\cE}$ in Case II]\label{le:case2}
    Suppose the GCFL and LCFL satisfy  \pref{as:gg_decrease}-\pref{as:energy}, and consider the discretized system with $\tau_g = 0, \tau_\ell>0$. Then we have the following:
    {\small
    \begin{equation}
        \begin{aligned}
            \int^t_{0}\dot{\cE}(\tau)\md \tau &\leq \int^t_{0}-\underbrace{\left(\frac{\gamma_1(\tau)}{2} - C_{21}\right)}_{\rm := \hat{\gamma}_1(\tau)}\cdot\norm{\nabla f(\bar{\xx}(\tau))}^2 \md \tau\\
            &\qquad + \int^t_0- \underbrace{\left(\frac{\gamma_2(\tau)}{2} - C_{22}\right)}_{\rm := \hat{\gamma}_2(\tau)} \cdot\norm{(I - R)\cdot\yy(\tau)}^2\md \tau,
        \end{aligned}
    \end{equation}
    }%
    where we have defined: {\small \begin{align}
        C_{21} &:= \frac{4L^2C_fC^2_\ell\eta_\ell^2(\tau)}{2(1-2L^2C_\ell^2)\cdot\min\{N\gamma_1(\tau), \gamma_2(\tau)\}}, \;C_{22} := \frac{L^2\eta_\ell^2(\tau)\cdot\left(\left(\frac{1-C_y}{C_y^2}\right)+ 4L_f^2C_fC^2_\ell\right)}{2(1-2L^2C_\ell^2)\cdot\min\{N\gamma_1(\tau), \gamma_2(\tau)\}},\nonumber\\
        C_f & := C_x^2+C_v^2+C_z^2, \; C_y = \e^{-C_g\tau_\ell \eta_g(\tau)}, \quad \; C_\ell := \frac{\tau_\ell \eta_\ell(\tau)}{\min\{2C_g\eta_g(\tau),1\}}\nonumber.
    \end{align}
    }%
\end{lemma}
Note that the requirements on $\hat{\gamma}_1(\tau)>0, \hat{\gamma}_2(\tau)>0$ result in the constraint on $\tau_\ell$, which will be discussed at the end of this section.

\noindent{\bf Two-sided Discretization ($\tau_{\ell}>0, \tau_{g}>0$, Case III-V).} We then analyze the more challenging cases where {\it both} the communication and the computation are discretized. Note that Case III with $\tau_\ell = \tau_g >0$ can be merged into Case IV, with $Q=1$.
\begin{lemma}[$\dot{\cE}$ in Case III-IV]\label{le:case4}
     Suppose the GCFL and LCFL satisfy properties \pref{as:gg_decrease}-\pref{as:energy}, and consider the discretized system with $\tau_g = Q\cdot \tau_\ell$. Then we have:
     {\small
    \begin{equation}
        \begin{aligned}
            \int^t_{0}\dot{\cE}(\tau)\md \tau &\leq \int^t_{0}-\underbrace{\left(\frac{\gamma_1(\tau)}{2}-C_{41}(\tau)\right)}_{\rm := \hat{\gamma}_1(\tau)}\cdot\norm{\nabla f(\bar{\xx}(\tau))}^2 \md \tau\\
            & \qquad + \int^t_0- \underbrace{\left(\frac{\gamma_2(\tau)}{2} - C_{42}(\tau)\right)}_{\rm := \hat{\gamma}_2(\tau)}\cdot \norm{(I - R)\cdot \yy(\tau)}^2\md\tau,
        \end{aligned}
    \end{equation}}%
     where the constants $C_{41}(\tau)$ and $C_{42}(\tau)$ are defined as: {\small
    \begin{align*}
        C_{41} &:= \frac{L^2\eta_\ell^2(\tau)\cdot \left(C_{45}\cdot(1+L_f^2C_{47}+C_{45})+C_{46}L_f^2\right)}{2\min\{N\gamma_1(\tau),\gamma_2(\tau)\}} + \frac{C_g\eta_g^2(\tau)\cdot(C_{43} + L_f^2C_{47})}{2\gamma_2(\tau)},\\
        C_{42} &:= \frac{L^2\eta_\ell^2(\tau)\cdot \left(C_{46} + C_{45}C_{47}\right)}{2\min\{N\gamma_1(\tau),\gamma_2(\tau)\}} + \frac{C_g\eta_g^2(\tau)C_{47}}{2\gamma_2(\tau)}, \quad C_{47} := Q^2C_{44}^2\cdot(C_x^2+C_v^2),\\
        C_{43} & := \frac{4\tau_g^2\eta_g^2(t)}{1-4\tau_g^2\eta_g^2(\tau)}, \; C_{44} := \frac{2\tau_\ell^2\tau_\ell^2(\tau)}{1-4\tau_g^2\eta_g^2(\tau)},\\
        C_{45} &:= \frac{4\tau_\ell^2\eta_g^2(\tau)}{1- 4L^2\tau_\ell^2\eta_\ell^2(\tau)},\;C_{46} := \frac{8L^2C_f\tau_\ell^2\eta_\ell^2(\tau)}{1-4L^2\tau_\ell^2\eta_\ell^2(\tau)}.
    \end{align*}}
\end{lemma}
Furthermore, we can check that when $\tau_g=0$ and $\tau_\ell=0$, then $C_{41}(\tau)$, $C_{42}(\tau)$ are both zero.  Additionally, $\hat{\gamma}_1(\tau)>0, \hat{\gamma}_2(\tau)>0$ determine the upper bounds for $\tau_g, \tau_\ell$, as well as the choice of the stepsizes of the discretized algorithms. 

Finally, we note that for Case V, a similar result with different $\hat{\gamma}_1(\tau), \hat{\gamma}_2(\tau)$ can be proved using the same technique as \leref{le:case2} and \leref{le:case4}. Since the utility of Case V can be covered mostly by that of Case II (cf. Table \ref{tab:discretization}), and due to the space limitation, we will not discuss this case in detail here.

By using the above results, {it is easy to obtain} the following convergence characterization. The proof is straightforward and follows that of { \thref{theorm:energy:continuous}}.
\begin{theorem}[Convergence of the discretized systems]\label{thm:discretized_system}
   Suppose the GCFL and LCFL satisfy properties \pref{as:gg_decrease}-\pref{as:energy}, and consider the discretized system with $\tau_\ell\geq 0, \tau_g\geq 0$. Then the convergence of the discretized system can be characterized as:
    {\small
    \begin{align*}
        \min_t\Bigg\{\norm{\nabla f(\bar{\xx}(t))}^2 &+ \norm{(I - R)\cdot \yy(t)}^2\Bigg\}= \cO\left(\max\left\{\frac{1}{\int^T_{0}\hat{\gamma}_1(\tau)\md \tau}, \frac{1}{\int^T_{0}\hat{\gamma}_2(\tau)\md \tau}\right\}\right),
    \end{align*}
    }
    where $\hat{\gamma}_1(\tau)>0$ and $\hat{\gamma}_2(\tau)>0$ depend on $\gamma_1(\tau), \gamma_2(\tau),N,C_g,L$ and $\eta_\ell,\eta_g,\tau_\ell,\tau_g,K,Q$, and their choices are specified in Lemmas \ref{le:case1} -- \ref{le:case4}.
\end{theorem}
This result indicates that as long as $\hat{\gamma}_1(\tau)> 0$ and $\hat{\gamma}_2(\tau) >0$, the discretized system preserves the convergence rate of the continuous-time system, but it slows down by a factor $\max\left\{\gamma_1(\tau)\left/\hat{\gamma}_1(\tau)\right.,\gamma_2(\tau)\left/\hat{\gamma}_2(\tau)\right.\right\}.$ Further, the condition that $\hat{\gamma}_1(\tau)>0,\hat{\gamma}_2(\tau)>0$ give a way to decide the maximum sampling intervals and the choice of the hyper-parameters (e.g., stepsize, the number of communication steps and local update steps $K$,$Q$) for different  algorithms, as we explain below. 

Let us consider Case I first.  By \leref{le:case1}, 
\[\min\{\gamma_2, 2\gamma_1\}\geq \frac{q^2_{\max}}{\gamma_2},\; \,\mbox{\rm with}\;\; q_{\max} = \e^{\sqrt{2}\tau_g\cdot\left(\sqrt{C_x^2+C_v^2}\eta_\ell(t) \cdot\left(1+\frac{L_f}{N}\right)^2\right)}-1.\]
It follows that $\tau_g \leq \frac{\ln(\min\{\gamma_2(t), \sqrt{2\gamma_1(t)\cdot\gamma_2(t)}\}+1)}{\sqrt{2}\sqrt{C_x + C_v}\eta_\ell(t)\cdot\left(\frac{L_f}{N} + 1\right)^2}.$
Note that all the variables on the right hand side (RHS) can be determined from the continuous-time system. This indicates that by having a convergent continuous-time system, the maximum sampling interval of the GCFL can be determined.  Similarly, for Case II, by  \leref{le:case2}, $\gamma_1(t)\geq 2C_{21}, \; \gamma_2(t) \geq 2C_{22},$ which implies:
{\small\[\tau_\ell \leq \min\left\{\frac{\tilde{\gamma}_1(t)}{\sqrt{2(\tilde{\gamma}_1^2(t)+4C_f)}L\eta_\ell^2(t)}, \frac{\log\left(\frac{\tilde{\gamma}_2(t)+2L\eta_\ell(t)}{2L\eta_\ell(t)}\right)}{C_g\eta_g(t)}\right\},\]}%
where ${\tilde{\gamma}^2_1(t)} := \min\{N\gamma_1^2(t), \gamma_1(t)\cdot\gamma_2(t)\}, \; \tilde{\gamma}_2^2(t) := \min\{\gamma_2^2(t), N\gamma_1(t)\cdot\gamma_2(t)\}$. All the variables on the RHS can be determined from the continuous-time system, so the maximum sampling interval of the LCFL can be determined. 

For Case III-IV, it requires $2C_{41} \leq \gamma_1(t), 2C_{42}\leq \gamma_2(t)$ and $\{C_{4i}\}_{i=3}^{6}$ to be positive. It may be difficult to obtain the exact bound for $\tau_g$, $\tau_\ell$ and $Q$, but we can derive an approximate bound on these parameters. For  $\{C_{4i}\}_{i=3}^{6}$ to be positive, it requires $\tau_\ell \leq \frac{1}{2L\eta_\ell(t)}$, $\tau_g \leq \frac{1}{2\eta_g(t)}$. Set $\tau_\ell = \frac{c}{2L\eta_\ell(t)}, \tau_g = \frac{c}{2\eta_g(t)}$ for some $c<1$. By choosing \begin{equation}\label{eq:stepsize_3}
    c^2 < \min\bigg\{\frac{1}{4}, \min\{\tilde{\gamma}^2_1(t), \tilde{\gamma}^2_2(t)\bigg\}\cdot \min\bigg\{\frac{1}{L^2\eta_\ell(t)^2\cdot(1+L_f^2)}, \frac{1}{C_g\eta_g^2(t)}\bigg\},
\end{equation} we have $C_{41} = \cO(\gamma_1(t)), C_{42} = \cO(\gamma_2(t))$. In addition, $Q = \tau_g/\tau_\ell \approx \frac{2L\eta_\ell(t)}{\eta_g(t)}.$
\section{Application of the Framework}\label{sec:discussion:existing_alg}
In this section, we discuss some applications of the proposed framework. We first show that by properly choosing the  controllers and the discretization scheme, the {multi-rate feedback control system} can be specialized to a number of popular distributed algorithms. Due to space limitations, we relegate the discussion some additional algorithms to appendix~\appref{sup:algorithms}.
Second, we show how the proposed framework can help identify the relationship between different algorithms. {Finally, we use DGT as an example to show how the framework can be used to streamline the convergence analysis of a series of algorithms, as well as to facilitate the development of new ones.}

\subsection{A New Interpretation of Distributed Algorithms}\label{sec:application}
In this part, we map some popular distributed algorithms to the discretized multi-rate systems, with specific GCFL and LCFL, and specific discretization setting. These mappings together provides a new perspective for understanding distributed algorithms.

Let us begin with mapping the decentralized optimization algorithms.

{\noindent\bf DGT~\cite{yuan2020can}:} The updates are given by:
\begin{align}\label{eq:GT}
    &\xx(k+1) = W\xx(k) - c\vv(k),\; \vv(k+1) = W\vv(k) + \nabla f(\xx(k+1)) - \nabla f(\xx(k)),
\end{align}
where $c>0$ is the stepsize.
It corresponds to the discretization Case III with the following continuous-time controllers:
\begin{align}\label{eq:GT:controller}
    u_{g,x} &= (I-W)\cdot\xx,\quad u_{g,v} = (I-W)\cdot\vv,\\
    u_{\ell,x} &= c\vv, \quad  u_{\ell,v} = -\nabla f(\xx)+ \nabla f(\zz), \quad  u_{\ell,z} = \zz - \xx.\nonumber
\end{align}

{\noindent\bf NEXT~\cite{di2016next}:} The updates of NEXT in discrete time are:
\begin{align*}
    \xx(k+{1}/{2}) & = \argmin_{\xx}\tilde{f}(\xx; \xx(k)) + \lin{N\vv(k) - \nabla f(\xx(k)), \xx-\xx(k)},\\
    \xx(k+1) & = W\left(\xx(k) + \alpha\cdot  (\xx(k+{1}/{2})-\xx(k))\right),\\
    \vv(k+1) & = W\vv(k) + \nabla f(\xx(k+1)) - \zz(k),\quad \zz(k+1) = \nabla f(\xx(k+1)),  
\end{align*}
where $\tilde{f}$ is some surrogate function; $k$ indicates the iteration index; $\alpha>0$ and $c>0$ are some stepsize parameters. By using the common choice that $\tilde{f}(\xx; \xx(k)) = \lin{\nabla f(\xx(k)), \xx - \xx(k)} + \frac{\eta}{2}\norm{\xx - \xx(k)}^2,$ (where $\eta>0$ are some constant) the algorithm can be simplify as:
\begin{align}\label{eq:next}
\begin{split}
      \xx(k+1) &= W\xx(k) - {N\alpha}/{\eta}\cdot \vv(k), \quad \zz(k+1) = \xx(k+1),\\
    \vv(k+1) &= W\vv(k) + \nabla f(\xx(k+1)) - \nabla f(\zz(k)).
\end{split}
  \end{align}
 Here, $\xx$ is the optimization variable, $\vv$ tracks the average of the gradients, $\zz$ records the {one-step-behind state of $\xx$}. It corresponds to Case III, with the continuous-time controllers given by:{\small
\begin{equation}\label{eq:next_ct}
    G_{g}(\xx,\vv;A) := \left[\begin{array}{c}
         (I-W)\cdot\xx\\
         (I-W)\cdot\vv
    \end{array}\right],
    \quad
    G_{\ell}(x_i,v_i,z_i; f_i) :=  \left[\begin{array}{c}
         v_i\\
         \nabla f_i(z_i) - \nabla f_i(x_i)\\
         z_i - x_i
    \end{array}\right].
\end{equation}}%

Next, we discuss two popular federated learning algorithms. {In this class of algorithms, the agents are connected with a central server which performs averaging. So the communication graph is a fully connected graph, with the weight matrix being the averaging matrix, i.e., $W = R, \; W_A = I-R$.}

{\noindent\bf FedAvg~\cite{bonawitz2019towards}:} The updates are given by (where GD is used for the local steps):
\begin{align*}
    \xx(k+1) = \begin{cases}
          R\xx(k) - \eta \nabla f(\xx(k)), \quad k \;\mbox{\rm mod}\;Q = 0,\\
        \xx(k) - \eta \nabla f(\xx(k)), \quad k \;\mbox{\rm mod}\;Q \neq 0.
    \end{cases}
\end{align*}
This algorithm has the following continuous-time controller: 
\begin{align}\label{eq:fedavg}
    u_{g,x} = 
        \sum_{k=0}^{\infty}\delta(t-k \tau_g)\cdot (I-R)\cdot\xx(t)
\end{align}
where $\delta(t)$ denotes the Dirac delta function.
It is interesting to note that FedAvg {\it cannot} be mapped to a continuous-time double-feedback system, as it does not have a {\it persistent} GCFL (it is only activated when $t=k\tau_g$; see \eqref{eq:fedavg}). This partially explains why FedAvg algorithm requires additional assumptions for convergence. 

{\noindent\bf Scaffold~\cite{karimireddy2020scaffold}:} The updates are given by (where $k_0:= k - (k\;\mbox{\rm mod}\; K)$): 
{\small
\begin{align*}
    \xx(k+1) &= \begin{cases}
           \xx(k) - \eta_1 \cdot(\nabla f(\xx(k)) - \zz(k) + \vv(k_0)) - \eta_2 \cdot (\xx(k) - \ww(k)), (k\;\mbox{\rm mod}\;Q) = 0,\\
        \xx(k) - \eta_1 \cdot(\nabla f(\xx(k)) - \zz(k) + \vv(k_0)), \; (k\;\mbox{\rm mod}\;Q) \neq 0.
    \end{cases}\\
    \vv(k+1) &= \begin{cases}
        \vv(k) - R\cdot(\vv(k) + \frac{1}{Q\eta_1}\cdot(\ww(k) - \xx(k))), &k\;\mbox{\rm mod}\; Q = 0\\
        \vv(k), &k\;\mbox{\rm mod}\; Q \neq 0,
    \end{cases}\\
        \ww(k+1) &= \begin{cases}
        R\xx(k) &k\;\mbox{\rm mod}\; Q = 0\\
        \ww(k), &k\;\mbox{\rm mod}\; Q \neq 0,
    \end{cases}\\
    \zz(k+1) &= \zz(k) - \frac{1}{Q}\vv(k) - \frac{1}{Q\eta_1}\cdot(\xx(k+1) - \xx(k)). 
\end{align*}
}%
 So it uses the discretization Case IV.  
 Observe that $\ww$ tracks $R\xx$, so in continuous-time we have: $\xx - \ww = (I-R)\cdot\xx + (R\xx - \ww) = (I-R)\cdot \xx + R\dot{\xx}$. Then we can replace $\ww$ by $R\cdot (\xx-\dot{\xx})$, and obtain the continuous-time controller as:
\begin{align}
\begin{split}\label{eq:dynamics:scaffold}
    u_{g,x} &= \eta_2\cdot (I-R)\cdot\xx + \eta_1\vv + \eta_2R\dot{\xx},\quad  u_{g,v} = -(I-R)\cdot(\vv +\dot{\xx}/\eta_1), \quad \\
    u_{\ell,x} &= \nabla f(\xx) - \zz, \quad u_{\ell,v} = \vv +\dot{\xx}/\eta_1, \quad u_{\ell,z} = \vv+\dot{\xx}/\eta_1.
\end{split}
\end{align}

Finally, we discuss one rate optimal algorithm:

{\noindent\bf xFilter~\cite{sun2019distributed}:} The updates are given by (where $k_0:= k - (k\;\mbox{\rm mod}\; K)$):
{\small
\begingroup
\allowdisplaybreaks
\begin{align*}
    \xx(k+1) &= \eta_1\cdot((1-\eta_2)I-\eta_2\cdot(I-W))\cdot\xx(k) + (1-\eta_1)\cdot\xx(k-1) + \eta_1\eta_2\vv(k_0)\\
        &= \xx(k) - \eta_1\eta_2 \cdot(2I-W)\xx(k) - (1-\eta_1)\cdot(\xx(k)-\xx(k-1)) + \eta_1\eta_2 \vv(k_0),\\
    \vv(k+1) &= \begin{cases}
        \vv(k) + (\ww_1(k) - \ww_2(k)) - (I-W)\cdot\xx(k), &k\;\mbox{\rm mod}\; K = 0\\
        \vv(k), &k\;\mbox{\rm mod}\; K \neq 0,
    \end{cases}\\
    \ww_1(k+1) &= \begin{cases}
        \xx(k) - \eta_3\nabla f(\xx(k)), &k\;\mbox{\rm mod}\; K = 0\\
        \ww_1(k), &k\mod K\neq 0,
    \end{cases}\\
    \ww_2(k+1) &= \begin{cases}
        \ww_1(k) , &k\;\mbox{\rm mod}\; K = 0\\
        \ww_2(k), &k\;\mbox{\rm mod}\; K \neq 0,
    \end{cases}
\end{align*}
\endgroup
}%
This algorithm uses the discretization Case V. We can see $\ww_2$ tracks $\ww_1$, and $\ww_1$ tracks $\xx - \eta_3\nabla f(\xx)$, therefore in continuous-time we have $\ww_1 - \ww_2 = \dot{\xx} - \eta_3\cdot \dot{\nabla} f(\xx)$, with the following continuous-time system:
\begin{align}
\begin{split}\label{eq:xfilter:1}
    \dot{\xx} &= -\eta_1\eta_2\cdot(2I-W)\cdot\xx +\eta_1\eta_2 \vv - (1-\eta_1)\cdot\dot{\xx},\\
    \dot{\vv} &= \dot{\xx} - \eta_3\dot{\nabla} f(\xx) - (I-W)\cdot\xx.
\end{split}
\end{align}
Integrating over time, and use the initialization that $\vv(0) = \xx(0) - \eta_3\nabla f(\xx(0))$, we have the following expression for $\vv(t)$: 
{\small\[\vv(t) = \int^t_{0}(\dot{\xx}(\tau) - \eta_3 \dot{\nabla} f(\xx(\tau))-(I-W)\cdot \xx(\tau)) \md \tau = \xx(t) - \eta_3\nabla f(\xx(t))-\int^t_0(I-W)\cdot \xx(\tau) \md \tau.\]}%
Define $\vv_1 = \frac{1}{2-\eta_1}\cdot(\xx -\vv)$,
$\zz = \frac{\eta_3}{2-\eta_1}\nabla f(\xx)$, then \eqref{eq:xfilter:1} can be equivalently written as:
\begin{align*}
    \dot{\xx} &=  -\eta_1\eta_2\cdot(I-W)\cdot\xx -\eta_1\eta_2\cdot(2-\eta_1)\cdot\vv_1- (1-\eta_1)\cdot\dot{\xx},\\
    &= -\eta_1\eta_2\cdot(I-W)\cdot\xx -\eta_1\eta_2\cdot(2-\eta_1)\cdot(\vv_1 - \zz) - (1-\eta_1)\cdot\dot{\xx} - \eta_1\eta_2\eta_3\nabla f(\xx)\\
    \dot{\vv}_1 &= \frac{1}{2-\eta_1}\cdot(I-W)\cdot\xx + \frac{\eta_3}{2-\eta_1}\dot{\nabla} f(\xx),\quad \dot{\zz} = \frac{\eta_3}{2-\eta_1}\dot{\nabla} f(\xx).
\end{align*}
The dynamic of $\dot{\xx}$ implies $\frac{1}{2-\eta_1}(I-R)\cdot(I-W)\cdot\xx = -(I-R)\cdot\left(\vv_1 +\frac{1}{\eta_1\eta_2}\dot{\xx}\right),$ where $(I-R)\cdot(I-W) = (I-W)$ by \pref{as:gg_decrease}. Substituting this into $\dot{\vv}_1$, defining $\eta_4:= \eta_1\eta_2, \eta_5 := (2-\eta_1), \eta_6 :=  \eta_1\eta_2\eta_3$, and rearranging the terms, we obtain the following equivalent controller:
\begin{align*}
    u_{g,x} &= \eta_4\cdot(I-W)\cdot\xx + \eta_4\eta_5\vv_1 + (\eta_5-1)\cdot\dot{\xx},\quad  u_{g,v} = -(I-R)\cdot(\vv_1 +\dot{\xx}/\eta_4), \quad \\
    u_{\ell,x} &= \eta_6\nabla f(\xx) - \eta_4\eta_5\zz, \quad u_{\ell,v} = \frac{\eta_3}{\eta_5}\dot{\nabla} f(\xx), \quad u_{\ell,z} = \frac{\eta_3}{\eta_5}\dot{\nabla} f(\xx).
\end{align*}
Interestingly, the above dynamics is  close to those of Scaffold in \eqref{eq:dynamics:scaffold}, except that Scaffold uses  $R$ instead of $W$, a different stepsize, and use $R \dot{\xx}$ in $u_{g,x}$ instead of $\dot{\xx}$.

\subsection{Algorithms Connections}
We summarize the discussion in the previous subsection in Table \ref{tab:alg_summary}. It is interesting to observe that, some seemingly unrelated algorithms, in fact are very closely related in continuous-time. For example, somewhat surprisingly, Scaffold and xFilter share very similar  continuous-time dynamics, although they are designed for very different purposes: the former is designed to improve FedAvg algorithm to better deal with data heterogeneity, while the latter is a primal-dual algorithm designed to achieve the optimal graph dependency. Similarly, each pair of algorithms FedPD and DLM, FedProx and DGD shares the same continuous-time dynamics (these algorithms are discussed in detail in appendix \ref{sup:algorithms}).
The latter two relations are relatively easier to identify. For example, FedPD and DLM are in fact designed from the same primal-dual perspective.  

\begin{table}[tb]
    \centering
    \begin{tabular}{cc|rrr}
        \hline
         GCFL           & LCFL                          & FL     & RO     & DO      \\
         \hline
         $(I-W)\cdot\xx$    & $\nabla f(\xx)$               & FedProx           & --            & DGD       \\
         $(I-W)\cdot\yy$    & $-\nabla f(\xx) + \nabla f(\zz)$  & --           & --            & DGT, NEXT  \\
         $c\cdot(I-W)\cdot\xx + \vv$ & $\nabla f(\xx)$           & FedPD           & --            & DLM        \\
         $(I-W)\cdot\xx + \eta\vv+ R\dot{\xx}$& $\nabla f(\xx) - \zz$& Scaffold   & --       & --\\
         $(I-W)\cdot\xx + \eta\vv+ \dot{\xx}$& $\nabla f(\xx) - \zz$& --   & xFilter       & --\\
         \hline
    \end{tabular}
    \caption{A summary of the controllers used in different algorithms. In GCFL and LCFL we abstract the most important steps of the controller.}
    \label{tab:alg_summary}
\end{table}

Additionally, from the table we can see that there are a few missing entries. Each of these entries represents a new algorithm. Also, we can combine different GCFLs and LCFLs, or design new controllers, to create new control systems (hence algorithms) that are not included in this table.
\subsection{Convergence Analysis and Algorithm Design: A Case Study}\label{sec:example}
In this subsection, we use the DGT algorithm as an example to illustrate how our proposed framework can be used in practice to analyze algorithm behavior, and to facilitate the development of new algorithms.

The iteration of the DGT is given in \eqref{eq:GT}. Under \asref{as:connect} -- \asref{as:lower_bounded}, this algorithm converges to the stationary point of the problem at a rate of $\cO(1/T)$~\cite{lu2019gnsd,sun2020improving}. To use our framework to analyze it, we will first construct a continuous-time double-feedback system, apply the discretization scheme III, and finally leverage Lemma \ref{le:case4} and Theorem \ref{thm:discretized_system} to obtain the convergence rate.

\subsubsection{Continuous-time Analysis}
 We begin by analyzing the continuous-time counterpart of the DGT, whose dynamics, according to \eqref{eq:GT:controller}, is given by:
\begingroup
\allowdisplaybreaks
\begin{equation}\label{eq:gt:dynamics}
    \begin{aligned}
        &\dot{\xx}(t) = - \eta_g(t)\cdot(I-W)\cdot\xx(t) - \eta_\ell(t)\cdot(c  \vv(t)), \quad         \dot{\zz}(t) = -\eta_\ell(t)\cdot(\zz(t) - \xx(t))\\
        &\dot{\vv}(t) = -  \eta_g(t)\cdot(I-W)\cdot\vv(t) + \eta_\ell(t)\cdot(\nabla f(\xx(t)) - \nabla f(\zz(t)))
    \end{aligned}
\end{equation}
\endgroup
where $\eta_g(t)=1, \eta_\ell(t)=1, \forall~t$. 

Let us verify properties \pref{as:gg_decrease}-\pref{as:energy}. First, it is easy to prove \pref{as:gg_linear} with the definition of $u_{g}$ given in \eqref{eq:GT:controller}. To show \pref{as:gg_decrease}, {recall that we have defined $W := I - A^T\rm{diag}(\ww)A$, so it is easy to verify that $\bone^T \cdot(I-W) = \bone^T\cdot A^T\rm{diag}(\ww)A = 0$ and $C_g = 1 - \lambda_2(W)$.}

To show \pref{as:gl_smooth}, we have the following bounds for different parts of the local controller:
\begingroup
\allowdisplaybreaks
{\small
\begin{align*}
    \norm{G_{\ell,x}(x_i,v_i,z_i;f_i) - G_{\ell,x}(x'_i,v'_i,z'_i;f_i)} &= \norm{c(v_i - v'_i)} = c \norm{v_i-v'_i}\\
    \norm{G_{\ell,v}(x_i,v_i,z_i;f_i) - G_{\ell,v}(x'_i,v'_i,z'_i;f_i)} &= \norm{\nabla f_i(x_i) - \nabla f_i(z_i) - \nabla f_i(x'_i) + \nabla f_i(z'_i)}\\
    &\leq \norm{\nabla f_i(x_i) - \nabla f_i(x'_i)} + \norm{\nabla f_i(z_i)  - \nabla f_i(z'_i)} \\
    &\leq L_f(\norm{x_i-x'_i}+\norm{z_i-z'_i})\\
    \norm{G_{\ell,z}(x_i,v_i,z_i;f_i) - G_{\ell,z}(x'_i,v'_i,z'_i;f_i)} &= \norm{x_i - z_i -x'_i + z'_i} \leq \norm{x_i-x'_i} + \norm{z_i-z'_i},
\end{align*}
}%
\endgroup
where $L_f$ is the constant of the Lipschitz gradient in \asref{as:smooth}. So the smoothness constant of the local controller $g_\ell$ can be expressed as $L = \max\{L_f, c, 1\}$.

To verify \pref{as:gl_decrease}, let us initialize $\vv(t) = \nabla f(\xx(t)), \zz(t) = \xx(t)$, and assume that $\eta_g(t)=0$ in \eqref{eq:gt:dynamics}, that is, the GCFL is inactive. Then we have: 
{\small
\begin{align}
\begin{split}
\label{eq:gt:state_v}
   \zz(t+\tau) & = \xx(t+\tau), \;\vv(t+\tau)  = \nabla f(\xx(t+\tau)),\\ \dot{\xx}(t+\tau)  & = -c \vv(t+\tau) = -c \nabla f(\xx(t+\tau)).
   \end{split}
\end{align}
Further, we can verify that the output of the LCFL can be bounded by {
\begin{align*}
    \norm{u_{i,\ell,x}(t)} &= \norm{c\cdot v_i(t)} = c\norm{\nabla f_i(x_i(t))}\\
    \norm{u_{i,\ell,v}(t)} &= \norm{\nabla f_i(x_i(t)) - \nabla f_i(z_i(t))} \leq 2\norm{\nabla f_i(x_i(t))}\\
    \norm{u_{i,\ell,z}(t)} &= \norm{z_i(t) - x_i(t)} = \norm{c\cdot v_i(t)} = c \norm{ \nabla f_i(x_i(t))}.
\end{align*}
}
The algorithm becomes the gradient flow algorithm that satisfies \pref{as:gl_decrease} with $\alpha(t) = c$, $C_x = c, C_v \leq 2, C_z = c$.} 
Finally, we verify \pref{as:energy}. 
\comment{Recall that in this case, the potential function is given by: {\red[can we remove this and cite? Isn't the potential function always of this form?]}
\begin{align}
    \cE(\xx(t), \vv(t)) := f(\bar{\xx}(t)) - f^\star + \frac{1}{2}\norm{(I-R)\cdot \yy(t)}^2.
\end{align}}
We can compute $\dot{\cE}(t)$ as follows: 
\small{
\begin{align}
    \dot{\cE}(t) & = -\lin{\nabla f(\bar{\xx}(t)), \frac{1}{N}\sum^N_{i=1}u_{\ell,x}(t)} -\lin{(I-R)\cdot\yy(t), u_{g,y}(t) + u_{\ell,y}(t)}\nonumber\\
    & \stackrel{\eqref{eq:GT:controller}}= -\lin{\nabla f(\bar{\xx}(t)), c\bar{\vv}(t)} -\lin{(I-R)\cdot\yy(t), (I-W)\cdot\yy(t)}\label{eq:E:dynamics}\\
    & \qquad- \lin{(I-R)\cdot\xx(t), c \vv(t)} + \lin{(I-R)\cdot\vv(t), \nabla f(\xx(t)) - \nabla f(\zz(t))}.\nonumber
\end{align}
}%
Then we bound each term on the RHS above separately, and finally integrate. The detailed derivation is relegated to supplementary Sec. \ref{sec:P5:DGT}. 
The final bound we can obtain is:
\begingroup
\allowdisplaybreaks
\small{
\begin{align*}
    \int^t_{0}\dot{\cE}
    &\leq -\frac{c}{2}\int^t_{0}\norm{\nabla f(\bar{\xx}(\tau))}^2\md \tau -\frac{c-8L_fc^2/\beta}{2}\int^t_{0}\norm{\bar{\vv}(\tau)}^2\md \tau\\
    &\quad- (C_g - \frac{c+2cL_f+\beta + 16cL_f/\beta}{2})\cdot\int^t_{0}\norm{(I-R)\cdot\yy(\tau)}^2\md \tau.\nonumber
\end{align*}
}%
\endgroup
By choosing $\beta < C_g/2$, $\frac{C_g^2}{64L_f}\leq c \leq \frac{C_g^2}{32L_f}$, we can verify that the dynamics of the continuous-time system \eqref{eq:gt:dynamics} satisfy \eqref{eq:E:dynamic}, with  $\gamma_1(t) \geq\frac{C_g^2}{128L_f}$ and $\gamma_2(t) \geq \frac{C_g}{4}$. Applying Theorem \ref{theorm:energy:continuous}, we know that continuous-time gradient tracking algorithm converges in $\cO(1/T)$.

\subsubsection{New Algorithm Design}
Now that we have verified properties \pref{as:gg_decrease}-\pref{as:energy}  for the  continuous-time system \eqref{eq:gt:dynamics}, we can derive a number of related algorithms by adjusting the discretization schemes, or by changing the GCFL.

Let us first consider changing the discretization scheme from Case III to Case IV, where $\tau_g = Q\tau_\ell > 0.$ In this case, there will be $Q$ local computation steps between every two communication steps. This kind of update scheme is closely related to algorithms in FL, and we refer to the resulting algorithm the Decentralized Federated Gradient Tracking (D-FedGT) algorithm. Its steps are listed below ({where}~$k_0 = k - (k~\mbox{mod}~ Q)$):
\begin{equation}
    \begin{aligned}
        &\xx(k+1) = \xx(k) - \tau_\ell \vv(k) - \tau_g (I-W)\xx(k_0),\\
        &\vv(k+1) = \vv(k) + \nabla f(\xx(k+1)) - \nabla f(\xx_k) - \tau_g (I-W)\vv(k_0).
    \end{aligned}
\end{equation}
By applying \leref{le:case4} and \thref{thm:discretized_system}, we can directly obtain that this new algorithm also converges with rate $\cO(\frac{1}{T})$ with properly chosen constant $\tau_\ell,\tau_g$ and $Q$ following \leref{le:case4} and \eqref{eq:stepsize_3}.

Second, we can replace the GCFL of the DGT with an {\it accelerated} consensus controller~\cite{ghadimi2011accelerated}. This leads to the a new Accelerated Gradient Tracking (AGT) algorithm:
\begin{equation}
    \begin{aligned}
        &\xx(k+1) = \xx(k) - \eta'_\ell \vv(k) - \eta'_g (1+c)\xx(k) + c \vv_{x}(k),\\
        &\vv(k+1) = \vv(k) + \nabla f(\xx(k+1)) - \nabla f(\xx(k)) - \eta_g (1+c)\vv(k) + c \vv_{v}(k),\\
        &\vv_{x}(k+1) = \xx(k),\quad \vv_{v}(k+1) = \vv(k), \; \text{ where } c := \frac{1-\sqrt{1-\lambda_2(W)}}{1+\sqrt{1-\lambda_2(W)^2}}.
    \end{aligned}
\end{equation}
Then by examining \pref{as:gg_decrease}, we know that the network dependency of the new algorithm improved from $C_g$ to $\hat{C}_g = C_g \cdot \frac{\sqrt{C_g}+ \sqrt{2-C_g}}{\sqrt{C_g}+ C_g\sqrt{2-C_g}} > C_g.$ And when $C_g$ is small, $\hat{C}_g$ scales with $\sqrt{C_g}$.
Then according to the derivation in the last subsection, we have $\gamma_2(t) \ge \frac{\hat{C}_g}{4}$. Finally, we can apply \thref{thm:discretized_system}, and asserts that the new algorithm improves the convergence speed from $\cO(\frac{1}{C_gT})$ to $\cO(\frac{1}{\hat{C}_gT})$.

\subsubsection{Numerical Results}
We provide numerical results for implementations of Continuous-time (CT) DGT, the D-FedGT and D-AGT algorithms discussed in the previous subsection. We first verify an observation from \thref{thm:discretized_system}, that discretization slows down the convergence speed of the system. Towards this end, we conduct numerical experiments with different discretization patterns and compare the convergence speed in terms of the stationarity gap. Then we compare the convergence speed of  CT-DGT and CT-AGT,  to demonstrate the benefit of changing the controller in the GCFL from the standard consensus controller to the accelerated one.

In the experiments, we consider the non-convex regularized logistic regression problem:
 \[f_i(\xx;(\aa_i, b_i)) = \log(1+\exp(-b_i\xx^T\aa_i)) + \sum^{d_x}_{d=1}\frac{\beta\alpha(\xx[d])^2}{1+\alpha(\xx[d])^2},\]
 where $\aa_i$ denotes the features and $b_i$ denotes the labels of the dataset on the $i^{\mathrm{th}}$ agent. We set the number of agent $N =20$ and each agent has local dataset of size $500$. We use an Erd{\H o}s--R{\'e}nyi random graph with density 0.5 for the network and optimize the weight matrix $W$ to achieve the optimal $C_g$. We set $c = 1$ for gradient tracking algorithm.

We first compare CT-DGT ($\tau_g = \tau_\ell = 0$) and D-FedGT ($\tau_g = 0.1, \tau_\ell = 0.005, Q = 20$), the result of CT-DGT and D-FedGT is showed in \figref{fig:GT_exp_1}. We can see that by discretizing each loop, the system converges slower as compared with the continuous time system. \figref{fig:GT_exp_2} shows the convergence behavior of the D-FedGT algorithm with different $\tau_g$. We observe that by increasing the sampling interval for GCFL, the convergence of the system slows down and it eventually diverges. \figref{fig:GT_exp_3} and \figref{fig:GT_exp_4} show the convergence results of D-AGT compared with DGT in both continuous time and in Case III. We observe that by changing the GCFL, D-AGT converges faster than DGT.

\begin{figure}[tb!]
    \begin{subfigure}[t]{0.45\textwidth}
    \centering
    \includegraphics[width = 0.9\linewidth]{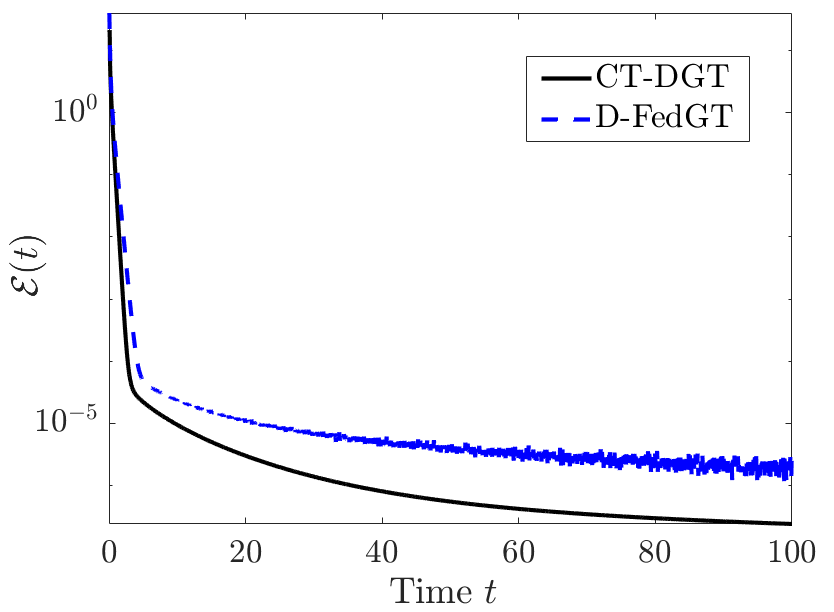}
    \caption{The evoluation of the Energy function $\cE(t)$ of CT-CGT, D-FedGT. }
    \label{fig:GT_exp_1}
    \end{subfigure}
    \hfill
    \begin{subfigure}[t]{0.45\textwidth}
    \centering
    \includegraphics[width = 0.9\linewidth]{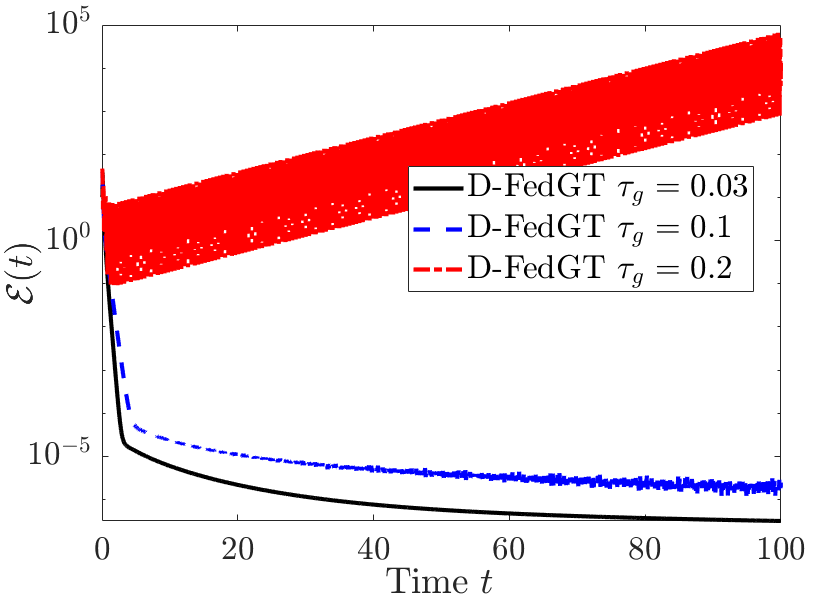}
    \caption{Energy function $\cE(t)$ of D-FedGT with different intervals $\tau_g$.}
    \label{fig:GT_exp_2}
    \end{subfigure}
    \begin{subfigure}[t]{0.45\textwidth}
    \centering
    \includegraphics[width = 0.9\linewidth]{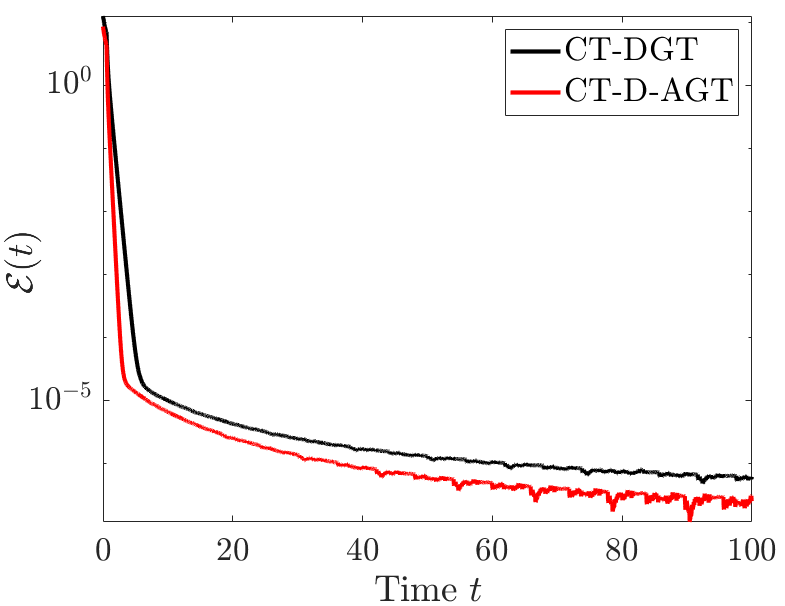}
    \caption{The evolution of the Energy function $\cE(t)$ of CT-DGT and CT-D-AGT.}
    \label{fig:GT_exp_3}
    \end{subfigure}
    \hfill
    \begin{subfigure}[t]{0.45\textwidth}
    \centering
    \includegraphics[width = 0.9\linewidth]{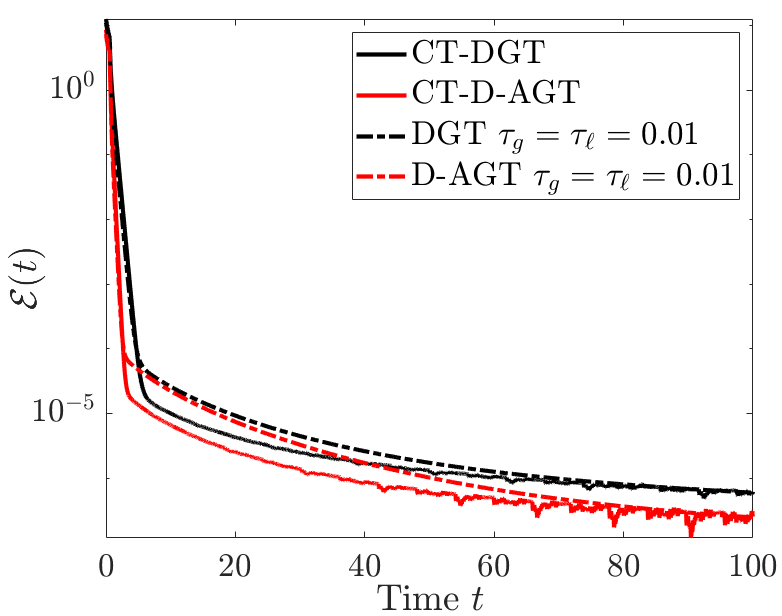}
    \caption{The evolution of the Energy function $\cE(t)$ of DGT and D-AGT.}
    \label{fig:GT_exp_4}
    \end{subfigure}
    \caption{The performance of Continuous-GT, D-FedGT, D-MGT and AGT.}
\end{figure}

\section{Conclusion}
In this work, we have designed a framework to understand distributed optimization algorithms from a control perspective. We have shown that a multi-rate double-feedback control system can represent a wide range of deterministic distributed optimization algorithms. We use a few examples to demonstrate how the proposed framework can help understand the connection between algorithms, as well as facilitate new algorithm design. In the future, we plan to extend the framework to model distributed stochastic algorithms.

\appendix

\section{Proofs of \secref{SEC:DISCRETIZATION}}\label{app:A}
Let $t_{\ell}$ (resp. $t_g$) denote the time at which the local (resp. global) controller samples, that is: $t_\ell := t - t \mbox{ mod }\tau_\ell$ and $t_g := t - t \mbox{ mod }\tau_g.$ To simplify the analysis, we treat the stepsizes $\eta_\ell(t), \eta_g(t)$ as constants in each sampling intervals. Also recall that $\yy(t) =[\xx(t); \vv(t)]$. The following relations will be useful:  \begin{align}
    \lin{a,b} &= \frac{1}{2\alpha}\norm{a}^2+ \frac{\alpha}{2} \norm{b}^2 - \frac{1}{2}\norm{\frac{1}{\sqrt{\alpha}}a + \sqrt{\alpha}b}^2 \leq \frac{1}{2\alpha}\norm{a}^2+ \frac{\alpha}{2} \norm{b}^2\label{eq:ab}\\
    (I-R)^2 &= I - 2R + R^2 = I - R, \quad  {\|R\|\le 1, \quad \|I-R\|\le 1}\label{eq:IR}.
\end{align}

The proofs of \leref{le:case1} - \leref{le:case4} adopt the similar concept in robust control theory. The time derivative of the energy function of the discretized system is given by:
{\small
\begin{align}
    \dot{\cE}(t) & = \underbrace{-\lin{\nabla f(\bar{\xx}(t)), \frac{1}{N}\bone^T\eta_\ell(t) u_{\ell,x}(t)} -\lin{(I - R)\cdot\yy(t), \eta_{\ell}(t)\cdot u_{\ell,y}(t)+\eta_g(t)\cdot u_{g}(t)}}_{\mbox{\rm term I}}\nonumber\\
        &\qquad + \hat{\cE}(t),\label{eq:hat_e}
\end{align}
}%
where ``term I" is the derivative of the continuous-time energy function given in \eqref{eq:E:derivative}; $\hat{\cE}(t)$ is the error caused by discretization. Integrate \eqref{eq:hat_e} and apply \pref{as:energy}, we have:
{\small
\begin{align}
    \int^{t}_0\dot{\cE}(t) & \leq -\int^t_{0}\gamma_1(\tau)\norm{\nabla f(\bar{\xx}(\tau))}^2 + \gamma_2(\tau)\norm{(I-R)\cdot \yy(\tau)}^2\md \tau + \int^t_0\hat{\cE}(\tau)\md \tau.\label{eq:hat:int}
\end{align}
}%
The key idea of proofs is to bound $\int_{0}^{t}\hat{\cE}(\tau) d\tau$ by the first two terms.

\subsection{Proof of \leref{le:case1}}\label{app:proof:case1}
In this case $\hat{u}_{g}(t) = G_g(\xx(t_g), \vv(t_g);A)$. By taking derivative of $\cE(t)$, and by comparing with \eqref{eq:hat_e}, we can obtain
{\small
    \begin{align}
        \hat{\cE}(t) =
        {\eta_g(t)\lin{(I - R)\cdot\yy(t), u_g(t)-\hat{u}_{g}(t)}}.\label{eq:hat:E:3.1} 
    \end{align}}%
Next, we bound $\int^t_0\hat{\cE}(\tau) d \tau$. Towards this end, we  first observe that: {\small
    \begin{align*}
        \lin{(I - R)\cdot\yy(t), u_g(t)-\hat{u}_{g}(t)} &\stackrel{(i)}{=} \lin{(I - R)\cdot\yy(t), G_g(\yy(t)-\yy(t_g);A)}\\
        &= \lin{(I - R)\cdot\yy(t), G_{g}\bigg(\int^t_{t_g}\dot{\yy}(s)ds;A\bigg)}\\
        &\stackrel{\eqref{eq:ab}}\leq \frac{\gamma_2(t)}{2} \norm{(I - R)\cdot\yy(t)}^2 + \frac{1}{2\gamma_2(t)} \norm{G_{g}\bigg(\int^t_{t_g}\dot{\yy}(s)ds;A\bigg)}^2,
    \end{align*}}%
where $(i)$ is due to the linearity property \pref{as:gg_linear}. {Next, we bound the last term above by  $\norm{\nabla f(\bar{\xx}(t))}^2$ and $\norm{(I-R)\cdot \yy(t)}^2$}. To proceed, let us define
\begin{align}\label{def:q}
\begin{split}
    \tilde{\yy}(t) &:=G_{g}\bigg(\int^t_{t_g}\dot{\yy}(s)ds;A\bigg) = u_g(t) -\hat{u}_g(t),\quad \ww(t):= [(I - R)\cdot\yy(t); \nabla f(\bar{\xx}(t))],\\
    q(t) & := \norm{G_{g}\bigg(\int^t_{t_g}\dot{\yy}(s)ds;A\bigg)}/\norm{[(I - R)\cdot\yy(t); \nabla f(\bar{\xx}(t)]]} = \|\tilde{\yy}(t)\|/\|\ww(t)\|.
    \end{split}
\end{align}  
{Using the above definition, we have:
{{\small
\begin{equation}\label{eq:case1:2}
    \bigg\|G_{g}\big(\int^t_{t_g}\dot{\yy}(s)ds;A\big)\bigg\|^2 =   \norm{\tilde{\yy}(t)}^2 = q^2(t)\norm{\ww(t)}^2.
\end{equation}}}%
It then suffices to bound $q(t)$.} Towards this end, let us first bound $\|\dot{\ww}(t)\|$ by: 
\begingroup
\allowdisplaybreaks
\begin{align}
      \norm{\dot{\ww}(t)} & \stackrel{(i)}= \norm{\left[(I-R)\cdot(\eta_g(t)\hat{u}_g(t) + \eta_\ell(t)u_{\ell,y}(t)); \lin{\partial^2 f(\bar{\xx}(t)), \eta_{\ell}(t)\frac{\bone^T}{N}u_{\ell,x}(t)}\right]}\nonumber\\
        & \leq \eta_g(t)\norm{(I-R)\cdot \hat{u}_g(t)} + \min\bigg\{\eta_\ell(t), \frac{\eta_\ell(t)\norm{\partial^2 f(\bar{\xx}(t))}}{N}\bigg\} \norm{u_{\ell,y}(t)}\nonumber\\
        & \stackrel{(ii)} \leq \eta_g(t)\left(\norm{(I-R)\cdot(u_g(t) - \hat{u}_g(t))} + \norm{(I-R)\cdot u_g(t)}\right) \nonumber\\
        &\qquad + \sqrt{C_x^2+C_v^2}\cdot\eta_\ell(t)\cdot(1+\frac{L_f}{N}) \cdot \norm{\nabla f(\xx(t))}\nonumber\\
        & \stackrel{(iii)} \leq \eta_g(t)\left(\norm{\tilde{\yy}(t)} + \norm{(I-R)\cdot \yy(t)}\right) + \sqrt{C_x^2+C_v^2}\cdot\eta_\ell(t)\cdot(1+\frac{L_f}{N}) \cdot \norm{\nabla f(\xx(t))}\nonumber\\
        & \stackrel{(iv)} \leq \eta_g(t)\cdot q(t)\cdot \norm{\ww(t)} + \eta_g(t)\cdot\norm{(I-R)\cdot \yy(t)} \nonumber\\
        &\qquad + \sqrt{C_x^2+C_v^2}\cdot\eta_\ell(t) \cdot(1+\frac{L_f}{N})\cdot\left(\norm{\nabla f(\bar{\xx}(t))} + \frac{L_f}{N}\norm{(I-R)\cdot \xx(t)}\right)\nonumber\\
        & \stackrel{(v)} \leq \sqrt{2}\left(\eta_g(t) q(t) +\eta_g(t)+ \sqrt{C_x^2+C_v^2}\cdot\eta_\ell(t) \cdot\left(1+\frac{L_f}{N}\right)^2\right)\cdot\norm{\ww(t)}, \label{eq:case_1:1_1}
\end{align}
\endgroup
where $(i)$ can be derived similarly as in \eqref{eq:E:derivative}; in $(ii)$ we add and subtract $u_g(t)$ to the first term, 
apply \pref{as:gl_decrease} to the last term, 
used the following definition
of sub-Hessian:  \[\lim_{\delta \rightarrow 0} \frac{\norm{f(x+\delta) - f(x) - \lin{\nabla f(x), \delta} - \frac{1}{2}\delta^T\partial^2 f(x)\delta}}{\norm{\delta}^2} = 0,\]
and the fact that that under the smoothness  \asref{as:smooth}, it holds that $\norm{\partial^2 f(\xx)} \leq L$ \cite[Theorem 3.1]{poliquin1996generalized}; in $(iii)$ we combine $\|I-R\|\le 1$ and  \eqref{eq:gg_bound} to the second term, use the definition of $\tilde{\yy}(t)$ in \eqref{def:q}; in $(iv)$ we use the definition of $q(t)$ in \eqref{def:q}, add and subtract $\nabla f(\bar{\xx}(t))$ to the last term and apply \asref{as:smooth}; in $(v)$ we use the fact that $\norm{a} +\norm{b} \leq \sqrt{2(\|a\|^2+\|b\|^2)}$, and $\xx$ is a subvector of $\yy$. Then we can bound $\dot{q}(t)$ by:
{\small
\begin{equation*}
    \begin{aligned}
        &\dot{q}(t) = \frac{\dot{\tilde{\yy}}(t)^T\tilde{\yy}(t)}{\norm{\ww(t)}\norm{\tilde{\yy}(t)}} - \frac{\norm{\tilde{\yy}(t)}\ww(t)^T\dot{\ww}(t)}{\norm{\ww(t)}^3}\\
        & \stackrel{(i)}\leq \frac{\norm{\dot{\tilde{\yy}}(t)}\norm{\tilde{\yy}(t)}}{\norm{\ww(t)}\norm{\tilde{\yy}(t)}} + \frac{\norm{\tilde{\yy}(t)}\norm{\ww(t)}\norm{\dot{\ww}(t)}}{\norm{\ww(t)}^3} \stackrel{(ii)}\leq (1+q(t))\frac{\norm{\dot{\ww}(t)}}{\norm{\ww(t)}}\\
        & \stackrel{\eqref{eq:case_1:1_1}}\leq (1+q(t))\cdot\sqrt{2}\left(q(t)\eta_g(t) +\eta_g(t)+ \sqrt{C_x^2+C_v^2}\eta_\ell(t) \cdot\left(1+\frac{L_f}{N}\right)^2\right),
    \end{aligned}
\end{equation*}}%
where in $(i)$ we apply the Cauchy–Schwarz inequality; $(ii)$ is due to the definition of $q(t)$ in \eqref{def:q}, and the relations below (where equality comes from the linearity property \pref{as:gg_linear}): \[\norm{\dot{\tilde{\yy}}(t)} = \norm {G_{g}(\dot{\yy}(t);A)}\stackrel{\eqref{eq:gg_bound}}\leq \norm{(I-R)\cdot \dot{\yy}(t)} \leq {\norm{\dot{\ww}(t)}}.\]
Note that $q(t_g) = 0$, solve the above inequality of $\dot{q}(t)$ by using Grownwall's inequality, we obtain $q(t) \leq q_{\max} := \exp\bigg\{\sqrt{2}\tau_g\cdot\left(\sqrt{C_x^2+C_v^2}\cdot \eta_\ell(t) \cdot\big(1+{L_f}/{N}\right)^2\big)\bigg\}-1$.
Plug in this estimate to \eqref{eq:case1:2}, and further to \eqref{eq:hat:E:3.1} and \eqref{eq:hat:int}, we obtain:
{\small
\begin{align*}
      &\int^t_0\dot{\cE}(\tau)\md \tau\leq \int^t_0\left(-\gamma_1(\tau)\norm{\nabla f(\bar{\xx}(\tau))}^2-\gamma_2(\tau)\norm{(I-R)\cdot \yy(\tau)}^2\right)\md \tau\\
        &\qquad+ \int^t_{0}\left(\frac{\gamma_2(\tau)}{2} \norm{(I - R)\cdot\yy(\tau)}^2 + \frac{1}{2\gamma_2(\tau)}q^2_{\max}\norm{\ww(\tau)}^2 \right)\md \tau\\
        & = \int^t_{0}-\left(\gamma_1(\tau) - \frac{q_{\max}^2}{2\gamma_2(\tau)}\right)\cdot\norm{\nabla f(\bar{\xx}(\tau))}^2-\left(\frac{\gamma_2(\tau)}{2} - \frac{q^2_{\max}}{2\gamma_2(\tau)}\right) \cdot\norm{(I - R)\cdot \yy(\tau)}^2\md \tau.
\end{align*}}%

\subsection{Proof of \leref{le:case2}}\label{app:proof:case2}
For notation simplicity, let us define the discrete time controller output as  $\hat{u}_{i,\ell}(t) = G_{i,\ell}(x_i(t_\ell), v_i(t_\ell), z_i(t_\ell);f_i)$. Then we can write $\dot{\cE}(t)$ similarly as in \eqref{eq:hat_e}, and the error term $\hat{\cE}(t)$ in this case can be expressed, and bounded as below: {\small
\begin{align}
      \hspace{-0.5cm}\hat{\cE}(t)= &\lin{\nabla f(\bar{\xx}(t)), \hspace{-0.1cm} \frac{\eta_\ell(t)}{N}\bone^T\hspace{-0.1cm}(u_{\ell,x}(t)-\hat{u}_{\ell,x}(t))} + \lin{(I-R)\yy(t), \eta_\ell(t) (I-R)\cdot(u_{\ell,y}(t)-\hat{u}_{\ell,y}(t))}\nonumber\\
        &\stackrel{\eqref{eq:ab}}\leq \frac{\gamma_1(t)}{2}\norm{\nabla f(\bar{\xx}(t))}^2 + \frac{\gamma_2(t)}{2}\norm{(I-R)\cdot\yy(t)}^2\nonumber\\
        &\quad + \frac{\eta_\ell^2(t) }{2N\gamma_1(t)}\norm{R\cdot(u_{\ell,y}(t)-\hat{u}_{\ell,y}(t))}^2 + {\frac{\eta_\ell^2(t) }{2\gamma_2(t)}}\norm{(I-R)\cdot(u_{\ell,y}(t)-\hat{u}_{\ell,y}(t))}^2\nonumber\\
        & \leq \frac{\gamma_1(t)}{2}\norm{\nabla f(\bar{\xx}(t))}^2 + \frac{\gamma_2(t)}{2}\norm{(I-R)\cdot\yy(t)}^2\nonumber\\
        & \quad \quad + \frac{\eta_\ell^2(t)L^2 }{2\min\{N\gamma_1(t), \gamma_2(t) \}}\left(\norm{\yy(t) -\yy(t_\ell)}^2 +\norm{\zz(t) -\zz(t_\ell)}^2\right). \label{eq:inner:case2}
    \end{align}}%
 where the last inequality combines \eqref{eq:IR} and the Lipschitz gradient property \pref{as:gl_smooth}, which gives:
\begin{align*}
    \norm{u_{\ell,y}(t)-\hat{u}_{\ell,y}(t)}^2  & = \sum_{i=1}^{N}\|G_{\ell}(\xx_i(t),\vv_i(t),\zz_i(t)) -G_{\ell}(\xx_i(t_{\ell}),\vv_i(t_{\ell}),\zz_i(t_{\ell}))\|^2 \\
    & \leq L^2(\norm{\yy(t) -\yy(t_\ell)}^2 +\norm{\zz(t) -\zz(t_\ell)}^2).
\end{align*}%
The key step is to bound the last term in \eqref{eq:inner:case2}. Towards this end, first note that we have the following relations from \eqref{eq:discrete_case_2} and \pref{as:gg_linear}:
\begin{align*}
    (I-R)\cdot \dot{\yy}(t) &= -\eta_g(t) \cdot (I-R) \cdot u_{g,y}(t) - \eta_\ell(t) \cdot (I-R) \cdot \hat{u}_{\ell,y}(t).\\
    &= -\eta_g(t) \cdot (I-R) \cdot W_A\yy(t) - \eta_\ell(t) \cdot (I-R) \cdot \hat{u}_{\ell,y}(t).
\end{align*}
Solving this differential equation with initial condition $\yy(t_\ell)$, we obtain: 
\begin{align}
    (I-R)\cdot \yy(t) & = \e^{-(I-R)\cdot W_A \int^t_{t_\ell}\eta_g(s)\md s}\left(\yy(t_\ell) - \int^t_{t_\ell}\eta_\ell(s)\e^{(I-R)\cdot W_A \int^s_{t_\ell}\eta_g(s_1)\md s_1}\md s \cdot \hat{u}_{\ell,y}(t)\right). \label{eq:y:inte}
\end{align}
This expression for $\yy(t_{\ell})$ can be used to further bound the following term: 
{\small
\begin{align}
    &\norm{(I-R)\cdot (\yy(t) - \yy(t_\ell))}^2\nonumber\\
    & \stackrel{\eqref{eq:y:inte}}= \Bigg\|(I-R) \cdot \bigg(\yy(t) - \left(\e^{-(I-R)\cdot W_A \int^t_{t_\ell}\eta_g(s)\md s}\right)^{-1}(I-R)\cdot\yy(t) \\
    & \quad - \int^t_{t_\ell}\eta_\ell(s)\e^{(I-R)\cdot W_A \int^s_{t_\ell}\eta_g(s_1)\md s_1}\md s \cdot \hat{u}_{\ell,y}(t)\bigg)\Bigg\|^2 \nonumber\\
    & \stackrel{(i)}\leq (1+ \beta)\norm{I-(I-R)\cdot\left(\e^{-(I-R)\cdot W_A \int^t_{t_\ell}\eta_g(s)\md s}\right)^{-1}}^2\norm{(I-R)\cdot \yy(t)}^2\nonumber\\
    &\qquad + (1+\frac{1}{\beta})\norm{\int^t_{t_\ell}\eta_\ell(s)\e^{(I-R)\cdot W_A \int^s_{t_\ell}\eta_g(s_1)\md s_1}\md s \cdot (I-R) \cdot \hat{u}_{\ell,y}(t))}^2 \nonumber\\
    & \stackrel{(ii)}\leq (1+ \beta)\cdot \left(\frac{1-C_y}{C_y}\right)^2\cdot \norm{(I-R)\cdot \yy(t)}^2 + (1+\frac{1}{\beta})\cdot \left(\frac{\tau_\ell\eta_\ell(t)}{C_y}\right)^2\cdot \norm{(I-R) \cdot\hat{u}_{\ell,y}(t))}^2\nonumber\\
    & \stackrel{(iii)} = \left(\frac{1-C_y}{C_y^2}\right)\cdot \norm{(I-R)\cdot \yy(t)}^2 + \left(\frac{\tau_\ell^2\eta_\ell^2(t)}{C_y}\right)\cdot \norm{(I-R) \cdot \hat{u}_{\ell,y}(t))}^2,\label{eq:case_2_1_1}
\end{align}}%
where in $(i)$ we use Cauchy–Schwarz inequality (with $\beta>0$ being an arbitrary constant); in $(ii)$ we bound the first norm with \pref{as:gg_decrease} so that $\|(I-R) W_A\| = \norm{W_A}\ge C_g$, which implies the following: \[\norm{I-(I-R)\cdot\left(\e^{-(I-R)\cdot W_A \int^t_{t_\ell}\eta_g(s)\md s}\right)^{-1}}^2 \leq \left(1 - (\e^{-C_g \int^t_{t_\ell}\eta_g(s)\md s})^{-1} \right)^2;\]then by using the fact that $t-t_\ell \leq \tau_\ell$, $\eta_g(s)$ can be treat as constant in the integration, and define $C_y := \e^{-C_g\tau_\ell\eta_g(t)}$, the bound can be further simplified as $\left(1 - (\e^{-C_g \int^t_{t_\ell}\eta_g(s)\md s})^{-1} \right)^2 \leq \left(1 - \frac{1}{C_y}\right)^2;$ in $(iii)$ we choose $\beta = \frac{C_y}{1-C_y}$.

Using the system dynamics \eqref{eq:discrete_case_2}, we have
\begin{align}
    R\cdot\yy(t) = R\cdot \yy(t_\ell) - \left(\int^t_{t_\ell} \eta_\ell(s)\md s\right)  R \hat{u}_{\ell,y}(t).  
    \label{eq:case_2_1_2}
\end{align}
Then we can bound the last term of \eqref{eq:inner:case2} by:
\begingroup
\small
\allowdisplaybreaks
\begin{align}
        &\norm{\yy(t) -\yy(t_\ell)}^2 +\norm{\zz(t) -\zz(t_\ell)}^2\nonumber\\
        & \stackrel{(i)}= \norm{(I-R)\cdot (\yy(t)-\yy(t_\ell))}^2+ \norm{R\cdot(\yy(t)-\yy(t_\ell))}^2 + \norm{\int^t_{t_\ell}\eta_\ell(s)\md s}^2\cdot \norm{\hat{u}_{\ell,z}(t)}^2\nonumber\\
        & \stackrel{\eqref{eq:case_2_1_1},\eqref{eq:case_2_1_2}}\leq \left(\frac{1-C_y}{C_y^2}\right)\cdot\norm{(I-R)\cdot\yy(t)}^2+ \left(\frac{\tau_\ell^2\eta_\ell^2(t)}{C_y}\right)\cdot \norm{(I-R)\cdot\hat{u}_{\ell,y}(t)}^2\nonumber\\
        &\qquad + \norm{\int^t_{t_\ell} \eta_\ell(s)\md s}^2\norm{R\hat{u}_{\ell,y}(t)}^2+\norm{\int^t_{t_\ell}\eta_\ell(s)\md s}^2\cdot \norm{\hat{u}_{\ell,z}(t)}^2\nonumber\\
        &\stackrel{(ii)}\leq \left(\frac{1-C_y}{C_y^2}\right)\cdot\norm{(I-R)\cdot\yy(t)}^2 +\frac{\left(\tau_\ell \eta_\ell(t)\right)^2}{\min\{C_y,1\}}\left(\norm{\hat{u}_{\ell,y}(t)}^2+\norm{\hat{u}_{\ell,z}(t)}^2\right)\nonumber\\
        &\stackrel{(iii)}\leq \left(\frac{1-C_y}{C_y^2}\right)\cdot\norm{(I-R)\cdot\yy(t)}^2+2C^2_\ell\left(\norm{u_{\ell}(t)-\hat{u}_{\ell}(t)}^2+ \norm{u_{\ell}(t)}^2\right)\nonumber\\
        &\stackrel{(iv)} \leq \left(\frac{1-C_y}{C_y^2}\right)\cdot\norm{(I-R)\cdot\yy(t)}^2+2L^2C^2_\ell\left(\norm{\yy(t) -\yy(t_\ell)}^2 +\norm{\zz(t) -\zz(t_\ell)}^2\right)\nonumber\\
        &\qquad + 4C^2_\ell\cdot(C_x^2+C_v^2+C_z^2)\cdot(\norm{\nabla f(\bar{\xx}(t))}^2+\norm{\nabla f(\xx(t)) - \nabla f(\bar{\xx}(t))}^2)\nonumber\\
        & \stackrel{(v)}\leq \frac{\left(\frac{1-C_y}{C_y^2}\right)+ 4L_f^2C^2_\ell C_f}{1-2L^2C_\ell^2}\norm{(I-R)\cdot\yy(t)}^2 + \frac{4C^2_\ell C_f}{1-2L^2C_\ell^2}\norm{\nabla f(\bar{\xx}(t))}^2,\label{eq:case_2_2}
\end{align}
\endgroup%
where in $(i)$ we separate $\yy(t)-\yy(t_\ell)$ into $R\cdot(\yy(t)-\yy(t_\ell)) + (I-R)\cdot (\yy(t)-\yy(t_\ell))$, expand the square, and use the fact that   $R\cdot(I-R) = 0$;
{in $(ii)$ we bound the integration interval in the last two terms with $t-t_\ell\leq \tau_\ell$, using the fact that $\eta_\ell(s)$ is treated as constant in the integration, and combine the last three terms; in $(iii)$ we add and subtract $u_{\ell}(t)$ to the last term and apply the Cauchy–Schwarz inequality and further define $C_\ell:= \frac{\tau_\ell \eta_\ell(t)}{\min\{C_y,1\}}$; in $(iv)$ we apply \pref{as:gl_smooth} and \pref{as:gl_decrease} to the last two terms and define \begin{equation}\label{eq:def:C}
   C_f := C_x^2+C_v^2+C_z^2;
\end{equation}
in $(v)$ }we apply \asref{as:smooth} to the last term {and move $\norm{\yy(t) -\yy(t_\ell)}^2 +\norm{\zz(t) -\zz(t_\ell)}^2$ to the left and divide both sides by $1-2L^2C_\ell^2$ {(note that this operation is legitimate since we have chosen {$\tau_\ell \leq \frac{1+2C_g\eta_g(t)}{2L\eta_\ell(t)}$} such that $2L^2C_\ell^2<1$})}.
Substitute to $\hat{\cE}$ in \eqref{eq:hat:int}, we have:
{\small
\begin{equation}
    \begin{aligned}
       \int^t_{0}\dot{\cE}(\tau)\md \tau &\leq \int^t_{0}\hspace{-0.1cm}\left(-\left(\frac{\gamma_1(\tau)}{2}-C_{21}\right)\hspace{-0.1cm}\norm{\nabla f(\bar{\xx}(\tau))}^2- \left(\frac{\gamma_2(\tau)}{2} -C_{22} \right) \hspace{-0.1cm} \norm{(I - R)\cdot\yy(\tau)}^2\right)\md \tau, \nonumber
    \end{aligned}
\end{equation}}%
where $C_{21} := \frac{4L^2C^2_\ell\eta_\ell^2(\tau)\cdot C_f}{2(1-2L^2C_\ell^2)\cdot\min\{N\gamma_1(\tau), \gamma_2(\tau)\}}$ and $C_{22} := \frac{L^2\eta_\ell^2(\tau)\cdot\left(\left(\frac{1-C_y}{C_y^2}\right)+ 4L_f^2C^2_\ell C_f\right)}{2(1-2L^2C_\ell^2)\cdot\min\{N\gamma_1(\tau), \gamma_2(\tau)\}}$.

\subsection{Proof for \leref{le:case4}}\label{app:proof:case345}
In Case III-IV, we have $\tau_g= Q \tau_{\ell}$. Also note that  $t_g$, $t_\ell$ were defined at the beginning of Appendix \ref{app:A}. The update of the states can be written as:
{\small
\begin{align}
\begin{split}\label{eq:case_4_1}
\yy(t_g+ (q+1)\tau_\ell) &= \yy(t_g+ q\tau_\ell) - \int^{t_g+(q+1)\tau_\ell}_{t_g+q\tau_\ell}\eta_g(s)\hat{u}_{g}(s) + \eta_\ell(s)\hat{u}_{\ell,y}(s)\md s,\\
    \zz(t_g+ (q+1)\tau_\ell) &= \zz(t_g+ q\tau_\ell) - \int^{t_g+(q+1)\tau_\ell}_{t_g+q\tau_\ell}\eta_\ell(s) \hat{u}_{\ell,z}(s)\md s.
    \end{split}
\end{align}}%

Using the decomposition $\cE(t)=\mbox{\rm term I} + \hat{\cE}(t)$, one can express, and subsequently bound the sampling error as:
\begingroup
\small
\allowdisplaybreaks
\begin{align}
    {\hat{\cE}(t)}&=\lin{\nabla f(\bar{\xx}(t)), \frac{\eta_\ell(t)}{N}\bone^T\cdot(u_{\ell,x}(t)-\hat{u}_{\ell,x}(t))} + \lin{(I - R)\cdot\yy(t), \eta_g(t)\cdot(u_g(t)-\hat{u}_{g}(t))}\nonumber\\
    & \qquad + \lin{(I - R)\cdot\yy(t), \eta_\ell(t)\cdot(u_{\ell,y}(t)-\hat{u}_{\ell,y}(t))}\nonumber\\
    &\stackrel{\eqref{eq:ab}}\leq \frac{\gamma_1(t)}{2}\norm{\nabla f(\bar{\xx}(t))}^2+\frac{\gamma_2(t)}{2}\norm{(I - R)\cdot\yy(t)}^2 \nonumber\\
    &\quad + \frac{\eta_g^2(t)}{2\gamma_2(t)} \norm{(I-R)\cdot(u_g(t)-\hat{u}_{g}(t))}^2+\frac{\eta_\ell^2(t) }{2\min\{N\gamma_1(t), \gamma_2(t)\}}\norm{u_{\ell,y}(t)-\hat{u}_{\ell,y}(t)}^2\nonumber\\
    &\stackrel{(i)}\leq \frac{\gamma_1(t)}{2}\norm{\nabla f(\bar{\xx}(t))}^2+ \frac{\gamma_2(t)}{2}\norm{(I - R)\cdot\yy(t)}^2 + \frac{\eta_g^2(t)}{2\gamma_2(t)} \norm{(I-R)\cdot(\yy(t)-\yy(t_g))}^2\nonumber\\
    &\quad + \frac{L^2\eta_\ell^2(t)}{2\min\{N\gamma_1(t), \gamma_2(t)\}}\left(\norm{\yy(t)-\yy(t_\ell)}^2+\norm{\zz(t)-\zz(t_\ell )}^2\right),\label{eq:hat_e_345}
\end{align}
\endgroup%
where in $(i)$ we {apply \pref{as:gg_linear} and \eqref{eq:gg_bound} to the third term, such that $\norm{(I-R)\cdot(u_g(t)-\hat{u}_{g}(t))}^2 = \norm{(I-R)\cdot W_A(\yy(t)-\yy(t_g))}^2 \leq \norm{(I-R)\cdot (\yy(t)-\yy(t_g))}^2$, and we have used \pref{as:gl_smooth} to the last term}. The key is to bound the last three terms of \eqref{eq:hat_e_345}. We divide it into three steps.

{\bf Step 1)} We  bound the {third term involving $\norm{(I-R)\cdot(\yy(t)-\yy(t_g))}^2$. With \eqref{eq:IR}, we have $\norm{(I-R)\cdot(\yy(t)-\yy(t_g))}^2 \leq \norm{\yy(t)-\yy(t_g)}^2$, then we bound} the RHS by: 
\begingroup
\allowdisplaybreaks
{\small
\begin{align}
    &\norm{\yy(t)-\yy(t_g)}^2  \stackrel{(i)}= \norm{(I-R)\cdot \int^t_{\tau_g} \eta_g(s) \hat{u}_g(s) \md s + \int^t_{t_g} \eta_\ell(s) \hat{u}_{\ell,y}(s)\md s}^2 \nonumber\\
    & \stackrel{(ii)}\leq 2\tau_g^2\eta_g^2(t)\norm{\hat{u}_g(t)}^2 + 2\norm{\int^t_{t_g}\eta_\ell(s)\hat{u}_{\ell,y}(s)\md s}^2\nonumber\\
    & \stackrel{(iii)} \leq 4\tau_g^2\eta_g^2(t)\left(\norm{\hat{u}_g(t) - u_g(t)}^2 + \norm{u_g(t)}^2\right) + 2\tau_\ell^2\sum^{t_\ell}_{\tau = t_g}\eta_\ell^2(\tau)\norm{\hat{u}_{\ell,y}(\tau)}^2\nonumber\\
    & \stackrel{(iv)} \leq 4\tau_g^2\eta_g^2(t)\left(\norm{\yy(t) - \yy(t_g)}^2 + \norm{(I-R)\cdot \yy(t)}^2\right) + 2\tau_\ell^2\sum^{t_\ell}_{\tau = t_g}\eta_\ell^2(\tau)\norm{\hat{u}_{\ell,y}(\tau)}^2\nonumber\\
    & \stackrel{(v)}\leq \frac{4\tau_g^2\eta_g^2(t)}{1-4\tau_g^2\eta_g^2(t)}\norm{(I-R)\cdot \yy(t)}^2 + \frac{2\tau_\ell^2}{1-4\tau_g^2\eta_g^2(t)}\sum^{t_\ell}_{\tau = t_g}\eta_\ell^2(\tau)\norm{\hat{u}_{\ell,y}(\tau)}^2,\label{eq:case_4_2}
\end{align}}%
\endgroup%
{where $(i)$ uses the first relation in \eqref{eq:case_4_1}, and $R\cdot \hat{u}_g(t) = 0$ (see \pref{as:gg_decrease});} 
in $(ii)$ we apply Cauchy-Schwarz inequality {and use the fact that $t-t_\ell \leq \tau_g$ and $\hat{u}_g(s), \eta_g(s)$ remain constants in the integration}; in $(iii)$ we add and subtract $u_g(t)$ in the first term and applied Cauchy-Schwarz inequality, and \eqref{eq:IR}; in $(iv)$ we apply \pref{as:gg_linear} to the first term and get $\hat{u}_g(t) - u_g(t) = G_g(\yy(t) - \yy(t_g); A)$, and apply the second inequality in \eqref{eq:gg_bound}, and the last inequality in \eqref{eq:IR}; $(v)$ holds because we moved $\norm{\yy(t) - \yy(t_g)}^2$ to the left and divide both sides by $1-4\tau_g^2\eta_g^2(t)$, and choose {$\tau_g < \frac{1}{2\eta_g(t)}$} such that $4\tau^2_g \eta^2_g(t)<1$.
To bound the last term of \eqref{eq:case_4_2}, we note that following series of relations:
\begingroup
\allowdisplaybreaks
{\small
\begin{align}
    &\norm{\hat{u}_{\ell,y}(\tau)}^2 \leq \norm{\hat{u}_\ell(\tau)}^2 \leq 2\norm{\hat{u}_{\ell}(\tau) - u_{\ell}(\tau)}^2 + 2\norm{u_{\ell}(\tau)}^2 \label{eq:case_4_2_1}\\
    & \stackrel{(\pref{as:gl_smooth})}\leq 2L^2\cdot\left(\norm{\yy(\tau) - \yy(t_\ell)}^2 + \norm{\zz(\tau) - \zz(t_\ell)}^2\right) + 2\norm{u_{\ell}(\tau)}^2\nonumber\\
    & \stackrel{(\pref{as:gl_decrease})} \leq 2L^2\cdot\left(\norm{\yy(\tau) - \yy(t_\ell)}^2 + \norm{\zz(\tau) - \zz(t_\ell)}^2\right) + 2C_f\norm{\nabla f(\xx(\tau))}^2\nonumber \\
    & 
    \leq 2L^2\cdot\left(\norm{\yy(\tau) - \yy(t_\ell)}^2 + \norm{\zz(\tau) - \zz(t_\ell)}^2\right)\nonumber\\
    & \quad+ 4C_f\left(\norm{\nabla f(\xx(\tau)) - \nabla f(\bar{\xx}(\tau))}^2 + \norm{\nabla f(\bar{\xx}(\tau))}^2\right)\nonumber\\
    & \stackrel{(\asref{as:smooth})} \leq 2L^2\cdot\left(\norm{\yy(\tau) - \yy(t_\ell)}^2 + \norm{\zz(\tau) - \zz(t_\ell)}^2\right) \nonumber\\
    & \quad+ 4C_f\left(L_f^2\norm{(I-R)\cdot\xx(\tau))}^2 + \norm{\nabla f(\bar{\xx}(\tau))}^2\right),\nonumber
\end{align}}%
\endgroup%
where $C_f$ is defined in \eqref{eq:def:C}. Note that we need to further bound $\norm{\yy(\tau)-\yy(t_\ell)}^2+\norm{\zz(\tau)-\zz(t_\ell )}^2$, which is the same to the last two terms in \eqref{eq:case_4_1}.

{\bf Step 2.} We then bound $\norm{\yy(t)-\yy(t_\ell)}^2+\norm{\zz(t)-\zz(t_\ell )}^2$. By \eqref{eq:case_4_1}, we have:
\begingroup\small
\allowdisplaybreaks
    \begin{align}\label{eq:case_4_3}
        &\norm{\yy(t)-\yy(t_\ell )}^2 +\norm{\zz(t)-\zz(t_\ell )}^2\stackrel{\eqref{eq:case_4_1}} = \norm{\int^t_{t_\ell} \eta_g(s)\hat{u}_g(s) + \eta_\ell(s)\cdot \hat{u}_{\ell}(s)\md s}^2\\
        &\stackrel{(i)}\leq 2\tau_\ell^2\eta_g^2(t)\norm{\hat{u}_g(t)}^2 + 2\tau_\ell^2 \eta_\ell^2(t)\cdot\norm{\hat{u}_\ell(t)}^2\nonumber\\
        &\stackrel{\eqref{eq:case_4_2_1}}\leq 2\tau_\ell^2\eta_g^2(t)\norm{\hat{u}_g(t)}^2 + 4L^2\tau_\ell^2\eta_\ell^2(t)\cdot\left(\norm{\yy(t)-\yy(t_\ell)}^2 + \norm{\zz(t) - \zz(t_\ell)}^2\right)\nonumber\\
        & \qquad + 8L^2C_f\tau_\ell^2\eta_\ell^2(t)\cdot\left(\norm{\nabla f(\bar{\xx}(t))}^2 + L_f^2\norm{(I-R)\cdot \xx(t)}^2\right)\nonumber\\
        & \stackrel{(ii)}\leq  \frac{4\tau_\ell^2\eta_g^2(t)}{1- 4L^2\tau_\ell^2\eta_\ell^2(t)}\left(\norm{u_g(t) - \hat{u}_g(t)}^2 + \norm{u_g(t)}^2\right)\nonumber\\
        &\qquad + \frac{8L^2C_f\tau_\ell^2\eta_\ell^2(t)}{1-4L^2\tau_\ell^2\eta_\ell^2(t)}\cdot\left(\norm{\nabla f(\bar{\xx}(t))}^2 + L_f^2\norm{(I-R)\cdot \xx(t)}^2\right)\nonumber\\
        &\stackrel{(iii)}\leq  \frac{4\tau_\ell^2\eta_g^2(t)}{1- 4L^2\tau_\ell^2\eta_\ell^2(t)}\left(\norm{\yy(t) - \yy(t_g)}^2 + \norm{(I-R)\cdot \yy(t)}^2\right)\nonumber\\
        &\qquad + \frac{8L^2C_f\tau_\ell^2\eta_\ell^2(t)}{1-4L^2\tau_\ell^2\eta_\ell^2(t)}\cdot\left(\norm{\nabla f(\bar{\xx}(t))}^2 + L_f^2\norm{(I-R)\cdot \xx(t)}^2\right),\nonumber
    \end{align}
\endgroup%
where in $(i)$ we apply Cauchy-Schwarz inequality; in $(ii)$ add and subtract $u_g(t)$ to the first term and move $\norm{\yy(t)-\yy(t_\ell)}^2 + \norm{\zz(t) - \zz(t_\ell)}^2$ to the left and divide both sides by $1-4L^2\tau_\ell^2\eta_\ell^2(t)$, and choose {$\tau_\ell< \frac{1}{2L\eta_\ell(t)}$} such that $4L^2\tau_\ell^2\eta_\ell^2(t)<1$; in $(iii)$ we apply the second inequality in \eqref{eq:gg_bound}, as well as the fact that $\|I-R\|\le 1$.

To proceed, let us define $C_{43} := \frac{4\tau_g^2\eta_g^2(t)}{1-4\tau_g^2\eta_g^2(t)}$, $C_{44} := \frac{2\tau_\ell^2\eta_\ell^2(t)}{1-4\tau_g^2\eta_g^2(t)}$, $C_{45} := \frac{4\tau_\ell^2\eta_g^2(t)}{1- 4L^2\tau_\ell^2\eta_\ell^2(t)}$, $C_{46} := \frac{8L^2C_f\tau_\ell^2\eta_\ell^2(t)}{1-4L^2\tau_\ell^2\eta_\ell^2(t)}$. Then by plug  \eqref{eq:case_4_2} into \eqref{eq:case_4_3}, 
we have:
\begingroup
\allowdisplaybreaks
{\small
\begin{align}
    &\norm{\yy(t)-\yy(t_\ell )}^2 +\norm{\zz(t)-\zz(t_\ell )}^2 \stackrel{(i)}\leq \left(C_{45} + C_{43}C_{45} + C_{46}L_f^2\right)\cdot\norm{(I-R)\cdot \yy(t)}^2  \label{eq:case_4_3_1}\\
    & \qquad + C_{46}\norm{\nabla f(\bar{\xx}(t))}^2+ QC_{44}C_{45}\cdot\sum^{t_\ell}_{\tau = t_g}\norm{\hat{u}_{\ell,y}(\tau)}^2 \nonumber\\
    & \stackrel{(ii)}\leq \left(C_{45} + C_{43}C_{45} + C_{46}L_f^2\right)\cdot\norm{(I-R)\cdot \yy(t)}^2  + C_{46}\norm{\nabla f(\bar{\xx}(t))}^2\nonumber\\
    & \qquad + QC_{44}C_{45}\cdot\sum^{t_\ell}_{\tau = t_g}(C_x^2+C_v^2)\cdot\left(\norm{\nabla f(\xx(\tau)) - \nabla f(\bar{\xx}(\tau))}^2 + \norm{\nabla f(\bar{\xx}(\tau))}^2\right)\nonumber\\
    &\stackrel{(\asref{as:smooth})}\leq \left(C_{45} + C_{43}C_{45} + C_{46}L_f^2\right)\cdot\norm{(I-R)\cdot \yy(t)}^2  + C_{46}\norm{\nabla f(\bar{\xx}(t))}^2\nonumber \\
    & \qquad+ QC_{44}C_{45}\cdot\sum^{t_\ell}_{\tau = t_g}(C_x^2+C_v^2)\cdot\left(L_f^2\norm{(I-R)\cdot \xx(\tau)}^2 + \norm{\nabla f(\bar{\xx}(\tau))}^2\right),\nonumber
\end{align}}%
\endgroup%
where in $(i)$ we use the fact that $t - t_g \leq Q\tau_\ell;$ in $(ii)$ we first apply \pref{as:gl_decrease} to the last term, then subtract $\nabla f(\bar{\xx}(\tau))$, and finally used Cauchy-Schwartz inequality. This completes Part II of the proof.

{\bf Step 3.} Finally, we substitute \eqref{eq:case_4_3_1} into Part I \eqref{eq:case_4_2_1} then to \eqref{eq:case_4_2}, we obtain:
\begingroup
\allowdisplaybreaks
{\small
\begin{align*}
     &\norm{\yy(t)-\yy(t_g)}^2 \stackrel{\eqref{eq:case_4_2_1}}\leq 4C_fC_{44}\sum^t_{\tau = t_g}\left(L_f^2\norm{(I-R)\cdot\xx(\tau))}^2 + \norm{\nabla f(\bar{\xx}(\tau))}^2\right)\\
     &\qquad + C_{43}\norm{(I-R)\cdot \yy(t)}^2 + 2L^2C_{44}\sum^t_{\tau = t_g}\left(\norm{\yy(\tau) - \yy(t_\ell)}^2 + \norm{\zz(\tau) - \zz(t_\ell)}^2\right)\\
     &\stackrel{\eqref{eq:case_4_3_1}}\leq C_{43}\norm{(I-R)\cdot \yy(t)}^2 + 4C_fC_{44}\sum^t_{\tau = t_g}\left(L_f^2\norm{(I-R)\cdot\xx(\tau))}^2 + \norm{\nabla f(\bar{\xx}(\tau))}^2\right)\\
     &\qquad+ 2L^2C_{44}\sum^t_{\tau = t_g} C_{46}\norm{\nabla f(\bar{\xx}(\tau))}^2\\
     &\qquad + 2L^2C_{44}^2C_{45}\cdot\sum^t_{\tau = t_g} \sum^{\tau}_{\tau_1 = t_g}(C_x^2+C_v^2)\cdot\left(L_f^2\norm{(I-R)\cdot \xx(\tau_1)}^2 + \norm{\nabla f(\bar{\xx}(\tau_1))}^2\right).
\end{align*}
}%
\endgroup%
 Then we substitute \eqref{eq:case_4_2} and \eqref{eq:case_4_3_1} to \eqref{eq:hat_e_345} then to \eqref{eq:hat:int}, we obtain:
\begingroup\allowdisplaybreaks
{\small
\begin{align*}
        \int^t_0\dot{\cE}(\tau)\md \tau &\leq -\int^t_0(\frac{\gamma_1(\tau)}{2}-C_{41}(\tau))\cdot\norm{\nabla f(\bar{\xx}(\tau))}^2\md \tau\\
        &\quad  - \int^t_0 (\frac{\gamma_2(\tau)}{2} - C_{42}(\tau))\cdot\norm{(I - R)\cdot\yy(\tau)}^2\md \tau,
    \end{align*}}%
where we have defined {\small
\begin{align*}
    C_{41} &:= \frac{L^2\eta_\ell^2(\tau)\cdot \left(C_{45}\cdot(1+L_f^2C_{47}+C_{45})+C_{46}L_f^2\right)}{2\min\{N\gamma_1(\tau),\gamma_2(\tau)\}} + \frac{C_g\eta_g^2(\tau)\cdot(C_{43} + L_f^2C_{47})}{2\gamma_2(\tau)},\\
    C_{42} &:= \frac{L^2\eta_\ell^2(\tau)\cdot \left(C_{46} + C_{45}C_{47}\right)}{2\min\{N\gamma_1(\tau),\gamma_2(\tau)\}} + \frac{C_g\eta_g^2(\tau)C_{47}}{2\gamma_2(\tau)}, \; \mbox{ and }\; C_{47}: = Q^2C_{44}^2\cdot(C_x^2+C_v^2).
\end{align*}}%
\endgroup%

\section{Distributed Algorithms as Discretized Multi-Rate Systems}\label{sup:algorithms}

In this section, we provide additional discussions on how to map the distributed algorithms to the discretized multi-rate systems. First, let us discuss decentralized algorithms. 

{\noindent\bf DGD~\cite{nedic2009distributed}:} The updates are given by (where $c>0$ is the stepsize):
\[\xx(k+1) = W \xx(k) - c \nabla f(\xx(k)) = \xx(k) - ((I-W) \xx(k)+ c )\cdot \nabla f(\xx(k)). \]
It uses the discretization Case III, with the following continuous-time controllers:
\begin{align*}
    &u_{g,x} = (I-W)\cdot\xx,\quad u_{\ell, x} = \nabla f(\xx).
\end{align*}

{\noindent\bf DLM~\cite{ling2015dlm}:} The updates are given by:
\begin{align*}
    &\xx(k+1) = \xx(k) - \eta\cdot \left(\nabla f(\xx(k)) + c\cdot  (I-W)\cdot\xx(k) + \vv(k)\right), \\
    &\vv(k+1) = \vv(k) + c \cdot (I-W)\cdot\xx(k+1).
\end{align*}
It corresponds to Case III, with the following continuous-time controllers:
\begin{align*}
    u_{g,x} & = c\cdot (I-W)\cdot\xx + \vv,\quad 
    u_{g,v} = (I-W)\cdot\xx,\quad u_{\ell,x}  = \nabla f(\xx),\quad 
    u_{\ell,v} = 0.
\end{align*}

Next, we discuss some popular federated learning algorithms. {For this class of algorithms, the agents are connected with a central server which performs averaging. The corresponding communication graph is a fully connected graph, with the weight matrix being the averaging matrix, i.e., $W = R, W_A = I-R$.}

{\noindent\bf FedProx~\cite{li2018federated}:} The updates are given by (where GD is used to solve local problems):
{\small
\begin{align*}
    \xx(k+1) = \begin{cases}
        \xx(k) - \eta_1 \nabla f(\xx(k)) - \eta_2 (\xx(k) - \xx(k_0)), \; k\;\mbox{\rm mod}\;Q \neq 0, \; k_0 = k - (k\;\mbox{\rm mod}\;Q),\\
        R\xx(k) - \eta_1 \nabla f(\xx(k)) - \eta_2\cdot (\xx(k) - \xx(k_0)), k\;\mbox{\rm mod}\;Q = 0.
    \end{cases}
\end{align*}
}%
It uses the discretization Case I, with the following continuous-time controllers:
\begin{align*}
    u_{g,x} &= (I-R)\cdot\xx ,\quad 
    u_{\ell,x} = \nabla f(\xx).
\end{align*}

{\noindent\bf FedPD~\cite{zhang2020fedpd}:} The updates are given by (where GD is used to solve local problems):
\begin{align*}
    \xx(k+1) &= \xx(k) - \eta_1\cdot (\nabla f(\xx(k)) + \vv(k) + \eta_2\cdot (\xx(k_0) - R\xx(k_0))), \; k_0 = k - (k\;\mbox{\rm mod}\; Q),\\
    \ww(k+1) &= \begin{cases}
        R \xx(k), k\;\mbox{\rm mod}\; Q = 0\\
        \ww(k), k\;\mbox{\rm mod}\; Q \neq 0,
    \end{cases}\\
    \vv(k+1) &= \begin{cases}
        \vv(k) + \frac{1}{\eta_2}\cdot(\xx(k) - \ww(k)), k\;\mbox{\rm mod}\; Q = 0\\
        \vv(k), k\;\mbox{\rm mod}\;Q \neq 0.
    \end{cases}
\end{align*}
It uses the discretization Case I or IV. Observe that $\ww$ tracks $R\xx$. 
Replace $\ww$ with $R\xx$, we can obtain the following controller:
\begin{align*}
    u_{g,x} &= (I-R)\cdot\xx+ \vv, \quad u_{g,v} = -(I-R)\cdot\xx, \quad u_{\ell,x} = \nabla f(\xx), \quad u_{\ell, v} = 0.
\end{align*}

Finally, we discuss one more rate optimal algorithm:

{\noindent\bf D-GPDA~\cite{sun2019distributed}:} The update step of Distributed Gradient Primal-Dual Algorithm (D-GPDA) is given by:
\begin{align*}
    \xx(k+1) &= \argmin_{\xx} \lin{\nabla f(\xx(k)) + A^T\vv(k), \xx - \xx(k)}\\
    &\quad+ \frac{1}{2}\norm{\eta_1A\xx}^2 + \norm{\eta_1\abs{A}\cdot(\xx -\xx(k))}^2 +\norm{\eta_2\cdot(\xx-\xx(k))}^2\\
    \vv(k+1) &= \vv(k) + \eta_1^2A\xx(k+1),
\end{align*}
where $\vv$ is the dual variable for the linear consensus constraint. By assuming the minimization is solved with gradient flow or $K$-step gradient descent, this algorithm is using the discretization Case II, with the following continuous-time controllers: 
\begin{align*}
    u_{g,x} &= \eta_1W\xx +\eta_2\cdot(\xx-\vv_2) - \eta_1\abs{A^TA}\vv_2+ A^T\vv_1,\quad  u_{g,v} = [-\eta_1^2A\xx; \;0], \quad \\
    u_{\ell,x} &= \nabla f(\xx), \quad u_{\ell,v} = [0; \;-(\xx-\vv_2)].
\end{align*}
\bibliographystyle{IEEEtran}
\bibliography{references}

\begin{thebibliography}{10}
\providecommand{\url}[1]{#1}
\csname url@samestyle\endcsname
\providecommand{\newblock}{\relax}
\providecommand{\bibinfo}[2]{#2}
\providecommand{\BIBentrySTDinterwordspacing}{\spaceskip=0pt\relax}
\providecommand{\BIBentryALTinterwordstretchfactor}{4}
\providecommand{\BIBentryALTinterwordspacing}{\spaceskip=\fontdimen2\font plus
\BIBentryALTinterwordstretchfactor\fontdimen3\font minus
  \fontdimen4\font\relax}
\providecommand{\BIBforeignlanguage}[2]{{%
\expandafter\ifx\csname l@#1\endcsname\relax
\typeout{** WARNING: IEEEtran.bst: No hyphenation pattern has been}%
\typeout{** loaded for the language `#1'. Using the pattern for}%
\typeout{** the default language instead.}%
\else
\language=\csname l@#1\endcsname
\fi
#2}}
\providecommand{\BIBdecl}{\relax}
\BIBdecl

\bibitem{wang2011control}
J.~Wang and N.~Elia, ``A control perspective for centralized and distributed
  convex optimization,'' in \emph{2011 50th IEEE conference on decision and
  control and European control conference}.\hskip 1em plus 0.5em minus
  0.4em\relax IEEE, 2011, pp. 3800--3805.

\bibitem{chang2020distributed}
T.-H. Chang, M.~Hong, H.-T. Wai, X.~Zhang, and S.~Lu, ``Distributed learning in
  the nonconvex world: From batch data to streaming and beyond,'' \emph{IEEE
  Signal Processing Magazine}, vol.~37, no.~3, pp. 26--38, 2020.

\bibitem{li2020federated}
T.~Li, A.~K. Sahu, A.~Talwalkar, and V.~Smith, ``Federated learning:
  Challenges, methods, and future directions,'' \emph{IEEE Signal Processing
  Magazine}, vol.~37, no.~3, pp. 50--60, 2020.

\bibitem{nedic2009distributed}
A.~Nedic and A.~Ozdaglar, ``Distributed subgradient methods for multi-agent
  optimization,'' \emph{IEEE Transactions on Automatic Control}, vol.~54,
  no.~1, pp. 48--61, 2009.

\bibitem{yuan2016convergence}
K.~Yuan, Q.~Ling, and W.~Yin, ``On the convergence of decentralized gradient
  descent,'' \emph{SIAM Journal on Optimization}, vol.~26, no.~3, pp.
  1835--1854, 2016.

\bibitem{ling2015dlm}
Q.~Ling, W.~Shi, G.~Wu, and A.~Ribeiro, ``Dlm: Decentralized linearized
  alternating direction method of multipliers,'' \emph{IEEE Transactions on
  Signal Processing}, vol.~63, no.~15, pp. 4051--4064, 2015.

\bibitem{yuan2020can}
K.~Yuan, W.~Xu, and Q.~Ling, ``Can primal methods outperform primal-dual
  methods in decentralized dynamic optimization?'' \emph{arXiv preprint
  arXiv:2003.00816}, 2020.

\bibitem{di2016next}
P.~Di~Lorenzo and G.~Scutari, ``Next: In-network nonconvex optimization,''
  \emph{IEEE Transactions on Signal and Information Processing over Networks},
  vol.~2, no.~2, pp. 120--136, 2016.

\bibitem{bonawitz2019towards}
K.~Bonawitz, H.~Eichner, W.~Grieskamp, D.~Huba, A.~Ingerman, V.~Ivanov,
  C.~Kiddon, J.~Kone{\v{c}}n{\`y}, S.~Mazzocchi, H.~B. McMahan \emph{et~al.},
  ``Towards federated learning at scale: System design,'' \emph{arXiv preprint
  arXiv:1902.01046}, 2019.

\bibitem{khaled2019first}
A.~Khaled, K.~Mishchenko, and P.~Richt{\'a}rik, ``First analysis of local {GD}
  on heterogeneous data,'' \emph{arXiv preprint arXiv:1909.04715}, 2019.

\bibitem{li2019convergence}
X.~Li, K.~Huang, W.~Yang, S.~Wang, and Z.~Zhang, ``On the convergence of fedavg
  on non-iid data,'' in \emph{International Conference on Learning
  Representations}, 2019.

\bibitem{li2018federated}
T.~Li, A.~K. Sahu, M.~Zaheer, M.~Sanjabi, A.~Talwalkar, and V.~Smith,
  ``Federated optimization in heterogeneous networks,'' \emph{arXiv preprint
  arXiv:1812.06127}, 2018.

\bibitem{karimireddy2020scaffold}
S.~P. Karimireddy, S.~Kale, M.~Mohri, S.~Reddi, S.~Stich, and A.~T. Suresh,
  ``Scaffold: Stochastic controlled averaging for federated learning,'' in
  \emph{International Conference on Machine Learning}.\hskip 1em plus 0.5em
  minus 0.4em\relax PMLR, 2020, pp. 5132--5143.

\bibitem{zhang2020fedpd}
X.~Zhang, M.~Hong, S.~Dhople, W.~Yin, and Y.~Liu, ``Fedpd: A federated learning
  framework with optimal rates and adaptivity to non-iid data,'' \emph{arXiv
  preprint arXiv:2005.11418}, 2020.

\bibitem{scaman2017optimal}
K.~Scaman, F.~Bach, S.~Bubeck, Y.~T. Lee, and L.~Massouli{\'e}, ``Optimal
  algorithms for smooth and strongly convex distributed optimization in
  networks,'' in \emph{international conference on machine learning}.\hskip 1em
  plus 0.5em minus 0.4em\relax PMLR, 2017, pp. 3027--3036.

\bibitem{sun2019distributed}
H.~Sun and M.~Hong, ``Distributed non-convex first-order optimization and
  information processing: Lower complexity bounds and rate optimal
  algorithms,'' \emph{IEEE Transactions on Signal processing}, vol.~67, no.~22,
  pp. 5912--5928, 2019.

\bibitem{rossi2006gradient}
R.~Rossi and G.~Savar{\'e}, ``Gradient flows of non convex functionals in
  hilbert spaces and applications,'' \emph{ESAIM: Control, Optimisation and
  Calculus of Variations}, vol.~12, no.~3, pp. 564--614, 2006.

\bibitem{sundararajan2021analysis}
A.~Sundararajan, \emph{Analysis and Design of Distributed Optimization
  Algorithms}.\hskip 1em plus 0.5em minus 0.4em\relax The University of
  Wisconsin-Madison, 2021.

\bibitem{lessard2016analysis}
L.~Lessard, B.~Recht, and A.~Packard, ``Analysis and design of optimization
  algorithms via integral quadratic constraints,'' \emph{SIAM Journal on
  Optimization}, vol.~26, no.~1, pp. 57--95, 2016.

\bibitem{hu2017control}
B.~Hu and L.~Lessard, ``Control interpretations for first-order optimization
  methods,'' in \emph{2017 American Control Conference (ACC)}.\hskip 1em plus
  0.5em minus 0.4em\relax IEEE, 2017, pp. 3114--3119.

\bibitem{muehlebach2019dynamical}
M.~Muehlebach and M.~Jordan, ``A dynamical systems perspective on nesterov
  acceleration,'' in \emph{International Conference on Machine Learning}, 2019,
  pp. 4656--4662.

\bibitem{swenson2021distributed}
B.~Swenson, R.~Murray, H.~V. Poor, and S.~Kar, ``Distributed gradient flow:
  Nonsmoothness, nonconvexity, and saddle point evasion,'' \emph{IEEE
  Transactions on Automatic Control}, 2021.

\bibitem{francca2018dynamical}
G.~Fran{\c{c}}a, D.~P. Robinson, and R.~Vidal, ``A dynamical systems
  perspective on nonsmooth constrained optimization,'' \emph{arXiv preprint
  arXiv:1808.04048}, 2018.

\bibitem{swenson2019distributed}
B.~Swenson, R.~Murray, H.~V. Poor, and S.~Kar, ``Distributed gradient descent:
  Nonconvergence to saddle points and the stable-manifold theorem,'' in
  \emph{2019 57th Annual Allerton Conference on Communication, Control, and
  Computing (Allerton)}.\hskip 1em plus 0.5em minus 0.4em\relax IEEE, 2019, pp.
  595--601.

\bibitem{droge2014continuous}
G.~Droge, H.~Kawashima, and M.~B. Egerstedt, ``Continuous-time
  proportional-integral distributed optimisation for networked systems,''
  \emph{Journal of Control and Decision}, vol.~1, no.~3, pp. 191--213, 2014.

\bibitem{ghadimi2011accelerated}
E.~Ghadimi, M.~Johansson, and I.~Shames, ``Accelerated gradient methods for
  networked optimization,'' in \emph{Proceedings of the 2011 American Control
  Conference}.\hskip 1em plus 0.5em minus 0.4em\relax IEEE, 2011, pp.
  1668--1673.

\bibitem{olshevsky2009convergence}
A.~Olshevsky and J.~N. Tsitsiklis, ``Convergence speed in distributed consensus
  and averaging,'' \emph{SIAM journal on control and optimization}, vol.~48,
  no.~1, pp. 33--55, 2009.

\bibitem{bubeck2015geometric}
S.~Bubeck, Y.~T. Lee, and M.~Singh, ``A geometric alternative to nesterov's
  accelerated gradient descent,'' \emph{arXiv preprint arXiv:1506.08187}, 2015.

\bibitem{ye2020multi}
H.~Ye, L.~Luo, Z.~Zhou, and T.~Zhang, ``Multi-consensus decentralized
  accelerated gradient descent,'' \emph{arXiv preprint arXiv:2005.00797}, 2020.

\bibitem{orvieto2019continuous}
A.~Orvieto and A.~Lucchi, ``Continuous-time models for stochastic optimization
  algorithms,'' \emph{Advances in Neural Information Processing Systems},
  vol.~32, 2019.

\bibitem{reddi2020adaptive}
S.~Reddi, Z.~Charles, M.~Zaheer, Z.~Garrett, K.~Rush, J.~Kone{\v{c}}n{\`y},
  S.~Kumar, and H.~B. McMahan, ``Adaptive federated optimization,'' \emph{arXiv
  preprint arXiv:2003.00295}, 2020.

\bibitem{kuo1980digital}
B.~C. Kuo, ``Digital control systems,'' 1980.

\bibitem{li2019decentralized}
Z.~Li, W.~Shi, and M.~Yan, ``A decentralized proximal-gradient method with
  network independent step-sizes and separated convergence rates,'' \emph{IEEE
  Transactions on Signal Processing}, vol.~67, no.~17, pp. 4494--4506, 2019.

\bibitem{lu2019gnsd}
S.~Lu, X.~Zhang, H.~Sun, and M.~Hong, ``Gnsd: A gradient-tracking based
  nonconvex stochastic algorithm for decentralized optimization,'' in
  \emph{2019 IEEE Data Science Workshop (DSW)}.\hskip 1em plus 0.5em minus
  0.4em\relax IEEE, 2019, pp. 315--321.

\bibitem{sun2020improving}
H.~Sun, S.~Lu, and M.~Hong, ``Improving the sample and communication complexity
  for decentralized non-convex optimization: Joint gradient estimation and
  tracking,'' in \emph{International Conference on Machine Learning}.\hskip 1em
  plus 0.5em minus 0.4em\relax PMLR, 2020, pp. 9217--9228.

\bibitem{poliquin1996generalized}
R.~A. Poliquin and R.~T. Rockafellar, ``Generalized hessian properties of
  regularized nonsmooth functions,'' \emph{SIAM Journal on Optimization},
  vol.~6, no.~4, pp. 1121--1137, 1996.

\end{thebibliography}
\newpage

\begin{center}
    \Large{\bf Supplemental Materials}
\end{center}

\section{Proofs for \secref{SEC:CONTINUOUS}}

In this section, we provide the proofs for \eqref{eq:gg_convergence}, \eqref{eq:gl_convergence} and \coref{cor:continuous} in Section \ref{SEC:CONTINUOUS}.

\subsection{Proof of \eqref{eq:gg_convergence}}\label{app:proof:gg}
From \pref{as:gg_decrease}, we show that the time derivative of the consensus error is strictly negative:
\begin{align*}
    \frac{\partial}{\partial t}\norm{(I-R)\cdot\yy(t)}^2 &=2\lin{(I-R)\cdot\yy(t), \dot{\yy}(t)} \stackrel{(i)}=- 2\lin{(I-R)\cdot\yy(t), u_g(t)}\\
    &\stackrel{(ii)}\leq -2C_g \norm{(I-R)\cdot\yy(t)}^2,
\end{align*}
where in $(i)$ we apply \eqref{eq:continuous-system} and substitute $\eta_g(t) = 1, \eta_\ell(t) = 0$ and in $(ii)$ we apply \pref{as:gg_decrease}.

By applying Gronwall's inequality, we have 
\begin{align*}
    \norm{(I-R)\cdot\yy(t+\tau)}^2 &\leq \exp\bigg\{\int^{t+\tau}_{t}-2C_g\md \tau_1\bigg\}\norm{(I-R)\cdot\yy(t)}^2 \\
    & = \exp\big\{-2C_g\tau\big\}\norm{(I-R)\cdot\yy(t)}^2,
\end{align*}
which completes the proof of \eqref{eq:gg_convergence}.



\subsection{Proof of \eqref{eq:gl_convergence}}\label{app:proof:gl}

From \pref{as:gl_decrease}, we show that the time derivative of the local functions are strictly negative:
\begin{align*}
    \frac{\partial}{\partial t}f_i(x_i(t)) &= \lin{\nabla f_i(x_i(t)), \dot{x}_i(t)} \stackrel{(i)}= -\lin{\nabla f_i(x_i(t)), u_{i,\ell,x}(t)}\\
    &\stackrel{(ii)}\leq -\alpha(t) \cdot\norm{\nabla f_i(x_i(t))}^2.
\end{align*}
where in $(i)$ we apply \eqref{eq:continuous-system} and substitute $\eta_g(t) = 0, \eta_\ell(t) = 1$; in $(ii)$ we apply \pref{as:gl_decrease}.
Integrate it over time we have:
\begin{align}
    \int^t_{0} \beta(\tau,t)\cdot\norm{\nabla f_i(x_i(\tau))}^2 \md \tau &\leq \frac{1}{\int^t_{0}\alpha(\tau)\md \tau}\left(f_i(x_i(0)) - f_i(x_i(t))\right),\label{eq:gl:1}\\
    \min_{\tau\in[0,t]}\norm{\nabla f_i(x_i(\tau))}^2 \md \tau &\leq \frac{1}{\int^t_{0}\alpha(\tau)\md \tau}\left(f_i(x_i(0)) - f_i(x_i(t))\right),\label{eq:gl:2}
\end{align}
where in \eqref{eq:gl:1}$, \beta(\tau,t) = \frac{\alpha(\tau)}{\int^t_{0}\alpha(\tau)\md \tau}$ defines a distribution over time $[0,t]$ and the LHS is the expected value of $\norm{\nabla f_i(x_i(\tau))}^2$; in \eqref{eq:gl:2} we use the fact that $\E_t[X(t)] \geq \min_t\{X(t)\}$ for an arbitrary random variable $X(t)$. This completes the proof of \eqref{eq:gl_convergence}.

\subsection{Proof of Corollary \ref{cor:continuous}}\label{app:proof:continuous}

In this part, we prove the convergence of the system under \pref{as:gg_decrease}, \pref{as:gl_smooth}, \pref{as:gl_decrease}. First, we compute the derivative of $\cE$, then we break it down into three terms. By bounding each term, we obtain \pref{as:energy}. From \thref{theorm:energy:continuous}, we perform integration over time, then we have the final convergence result.

The time derivative of $\cE$ can be bounded by
\begingroup
\allowdisplaybreaks
{\begin{align}\label{eq:ct:1}
    &\dot{\cE}(t) = \lin{\nabla f(\bar{\xx}(t)), \frac{1}{N}\sum^N_{i=1}\dot{x}_i(t)} + \lin{(I-R)\cdot\yy(t), \dot{\yy}(t)}\nonumber\\
    & \stackrel{(i)}= -\lin{\nabla f(\bar{\xx}(t)), \frac{1}{N}\sum^N_{i=1}\eta_\ell (t)\cdot u_{i,\ell,x}(t) + \eta_g(t)\cdot\frac{\bone^T}{N}u_{g,x}(t)} \nonumber\\
    &\quad- \lin{(I-R)\cdot\yy(t), \eta_g(t)\cdot u_{g}(t) + \eta_\ell (t)\cdot u_{\ell,y}(t)}\nonumber\\
    & \stackrel{(\pref{as:gg_decrease})}\leq -C_g\eta_g(t)\norm{(I-R)\cdot\yy(t)}^2 - \eta_\ell (t)\lin{\nabla f(\bar{\xx}(t)), \frac{1}{N}\sum^N_{i=1}u_{i,\ell,x}(t)}\nonumber\\
    & \quad - \eta_\ell (t)\lin{(I-R)\cdot\yy(t), u_{\ell,y}(t)}\nonumber\\
    & \stackrel{\eqref{eq:IR}}= -C_g\eta_g(t)\norm{(I-R)\cdot\yy(t)}^2 - \eta_\ell (t)\lin{(I-R)\cdot\yy(t), (I-R)\cdot u_{\ell,y}(t)}\nonumber\\
    & \quad - \eta_\ell (t)\lin{\nabla f(\bar{\xx}(t)), \frac{1}{N}\sum^N_{i=1}u_{i,\ell,x}(t) + c\nabla f(\bar{\xx}(t)) - c\nabla f(\bar{\xx}(t))}\nonumber\\
    & \stackrel{(ii)}{\leq} -C_g\eta_g(t)\norm{(I-R)\cdot\yy(t)}^2 - \eta_\ell (t)\cdot c\norm{\nabla f(\bar{\xx}(t))}^2+ \frac{\beta_1(t)}{2}\norm{(I-R)\cdot\yy(t)}^2  \nonumber\\
    & \quad + \frac{\eta_\ell ^2(t)}{2\beta_1(t)}\norm{ (I-R)\cdot u_{\ell,y}(t)}^2+ \frac{\beta_2 (t)}{2}\norm{\nabla f(\bar{\xx}(t))}^2 + \frac{\eta_\ell ^2(t)}{2\beta_2(t)}\norm{\frac{1}{N}\sum^N_{i=1}u_{i,\ell,x}(t) - c\nabla f(\bar{\xx}(t))}^2\nonumber\\
    & =  -\left(C_g\eta_g(t)-\frac{\beta_1(t)}{2}\right)\cdot\norm{(I-R)\cdot\yy(t)}^2 - (c\eta_\ell (t)-\beta_2(t)/2)\cdot\norm{\nabla f(\bar{\xx}(t))}^2\nonumber\\
    &  \quad +\frac{\eta_\ell ^2(t)}{2\beta_1(t)}\norm{(I-R)\cdot u_{\ell,y}(t)}^2+ \frac{\eta_\ell ^2(t)}{2\beta_2(t)}\norm{\frac{1}{N}\sum^N_{i=1}(u_{i,\ell,x}(t) - c\nabla f_i(\bar{\xx}(t)))}^2,
\end{align}}%
\endgroup%
where in $(i)$ we substitute the system dynamics \eqref{eq:continuous-system}, and $u_g(t):=[u_{g,x}(t); u_{g,v}(t)]$; in $(ii)$ we apply \eqref{eq:ab}.
Then, we bound the last two terms of \eqref{eq:ct:1} separately. We have:
\begin{equation*}
    \begin{aligned}
        \norm{(I-R)\cdot u_{\ell,y}(t)}^2 & = \sum^N_{i=1}\norm{u_{i,\ell,y}(t) - \frac{1}{N}\sum^N_{j=1}u_{j,\ell,y}(t)}^2\leq \frac{N-1}{N}\sum^N_{i\neq j}\norm{u_{i,\ell,y}(t)-u_{j,\ell,y}(t)}^2\\
        & \leq \frac{4(N-1)}{N}\sum^{N}_{i=1}\norm{u_{i,\ell,y}(t)}^2\stackrel{(\pref{as:gl_decrease})}\leq \frac{4(N-1)\cdot(C_x^2+C_v^2)}{N}\norm{\nabla f(\xx(t))}^2.
    \end{aligned}
\end{equation*}
Also we have:
\begingroup
\allowdisplaybreaks
\begin{align*}
       &\norm{\frac{1}{N}\sum^N_{i=1}(u_{i,\ell,x}(t) - c\nabla f_i(\bar{\xx}(t)))}^2\\
        & = \norm{\frac{1}{N}\sum^N_{i=1}(u_{i,\ell,x}(t)- c\nabla f_i(x_i(t)) + c\nabla f_i(x_i(t))- c\nabla f_i(\bar{\xx}(t)))}^2\\
        & \leq \frac{2}{N}\sum^N_{i=1}\left(\norm{u_{i,\ell,x}(t)- c\nabla f_i(x_i(t))}^2+ c^2\norm{\nabla f_i(x_i(t))-\nabla f_i(\bar{\xx}(t))}^2\right)\\
        & \stackrel{(i)}{\leq} \frac{2}{N}\sum^N_{i=1}(\norm{u_{i,\ell,x}(t)}^2 +c^2\norm{\nabla f_i(x_i(t))}^2 - 2c\lin{u_{i,\ell,x}(t), \nabla f_i(x_i(t))}+ c^2L_f^2\norm{x_i(t)-\bar{\xx}(t)}^2)\\
        & \stackrel{(ii)}{\leq} \frac{2(C_x^2+c^2 - 2c\alpha(t))}{N}\norm{\nabla f(\xx(t))}^2 + \frac{2c^2L_f^2}{N}\norm{(I-R)\cdot\xx(t)}^2,
\end{align*}
\endgroup
where $(i)$ we expand the first term and apply \asref{as:smooth} to the second term; in $(ii)$ we use \pref{as:gl_decrease} for the first three terms and plug the definition of $I-R$ into the last term.
Further, we have:{
\begin{align*}
    \norm{\nabla f(\xx(t))}^2 &= \sum^N_{i=1}\norm{\nabla f_i(x_i(t)) - \nabla f_i(\bar{\xx}(t)) + \nabla f_i(\bar{\xx}(t))}^2 \\
    &\stackrel{\eqref{eq:ab}}\leq 2\sum^N_{i=1}\left(\norm{\nabla f_i(x_i(t)) - \nabla f_i(\bar{\xx}(t))}^2 + \norm{\nabla f_i(\bar{\xx}(t))}^2\right)\\
    &\stackrel{(\pref{as:smooth})} \leq 2L_f\norm{(I-R)\cdot \xx(t)}^2 + 2\sum^N_{i=1}\norm{\nabla f_i(\bar{\xx}(t))}^2.
\end{align*}}
Substitute back to \eqref{eq:ct:1}, we have
\begin{equation}\label{eq:ct:2}
    \begin{aligned}
    \dot{{\cal E}}(t)&\leq -\left(C_g\eta_g(t)-\frac{\beta_1(t)}{2}-\eta_\ell^2(t)\cdot\left(\frac{c^2L_f^2}{N\beta_2(t)} + C_{df}L_f\right)\right)\cdot\norm{(I-R)\cdot\yy(t)}^2 \\
    &\quad- \frac{2c\eta_\ell(t)-\beta_2(t))}{2}\norm{\nabla f(\bar{\xx}(t))}^2+\eta_\ell^2(t)\cdot C_{df}\sum^N_{i=1}\norm{\nabla f_i(\bar{\xx}(t))}^2,
    \end{aligned}
\end{equation}
where $C_{df} := \left(\frac{4(N-1)\cdot(C_x^2+C_v^2)}{N\beta_1(t)}+\frac{2(C_x^2+c^2 - 2c\alpha(t))}{N\beta_2(t)}\right)$. We analyze the convergence rate in two cases: i) $C_{df}\leq 0$, and ii) $C_{df}>0$.

\noindent{\bf Case i:} If $C_{df}\leq 0$, which implies $\alpha(t)>C_x$. Then, by choosing $\beta_1(t)\leq \frac{C_g\eta_g(t)}{4}$, $\beta_2(t)\leq c\eta_\ell(t)$, $\eta_\ell (t)\leq\frac{NC_g\eta_g(t)}{4cL_f^2},$ we have:
$$\dot{{\cal E}}(t)\leq -\frac{C_g\eta_g(t)}{2}\cdot\norm{(I-R)\cdot\yy(t)}^2 - \frac{c\eta_\ell(t)}{2}\cdot\norm{\nabla f(\bar{\xx}(t))}^2.$$ 
In this case, by choosing $\eta_g(t) = 1, \eta_\ell(t) = \frac{NC_g\eta_g(t)}{4cL_f^2}, c = \alpha(t)>C_x$, then \pref{as:energy} satisfies with $\gamma_1(t) = \frac{NC_g}{4L_f^2}, \gamma_2(t) = \frac{C_g}{2},$ and 
\[\min_\tau\{\norm{(I-R)\cdot \yy(\tau)}^2 + \norm{\nabla f(\bar{\xx}(\tau))}^2\} = \cO(1/t).\]

\noindent{\bf Case ii:} If  $C_{df}>0$, 
we show that by choosing $\eta_\ell(t) = \Theta(\norm{(I-R)\cdot \yy(\tau)}^2 + \norm{\nabla f(\bar{\xx}(\tau))}^2)$, $\eta_g(t) = \cO(1)$, $\min_\tau \{\norm{(I-R)\cdot \yy(\tau)}^2 + \norm{\nabla f(\bar{\xx}(\tau))}^2]\} = \cO(1/\sqrt{t})$ is satisfied. We proceed by bounding $\sum^N_{i=1}\norm{\nabla f_i(\bar{\xx}(t))}^2$ in \eqref{eq:ct:2}. First, we define the level set $\mS(t):= \{x \mid f(x)\leq \cE(t) + f^\star\}$. By \asref{as:coercive}, we can define the upper bound of $\sum^N_{i=1}\norm{\nabla f_i(\bar{\xx}(t))}^2$ as \[D(t):= \sup_{x\in \mS(t)}\left\{\sum^N_{i=1}\norm{\nabla f_i(x)}^2\right\}.\]
{Then, to guarantee that
\begin{align}
    D(\tau) \leq \frac{C_g\eta_g(\tau)}{4C_{df}\eta^2_\ell(\tau)}\cdot\norm{(I-R)\cdot\yy(\tau)}^2+ \frac{c}{4C_{df}\eta_\ell(\tau)}\norm{\nabla f(\bar{\xx}(\tau))}^2,\; \forall \tau \in[0,t],\nonumber
\end{align}
we can solve for $\beta_1(\tau)$, $\beta_2(\tau)$ and $\eta_\ell(\tau)$, which result in the following three relations:
\begin{align*}
  \beta_1(\tau) & \leq \frac{C_g\eta_g(\tau)}{2},\quad  \beta_2(\tau)\leq c\cdot\eta_\ell(\tau),\\
  \eta_\ell(\tau) &\leq \max\left\{\frac{\sqrt{C_g\eta_g(\tau)C_{df}L_f}\norm{(I-R)\cdot \yy(\tau)}^2}{4C_{df}D(\tau) + 2C_{df}L_f\norm{(I-R)\cdot \yy(\tau)}^2}, \frac{c\norm{\nabla f(\bar{\xx}(\tau))}^2}{4C_{df} D(\tau)} \right\}.
\end{align*}
These choices of parameters guarantee that 
\begin{align}
    &\eta_\ell^2(\tau)\cdot C_{df}\sum^N_{i=1}\norm{\nabla f_i(\bar{\xx}(\tau))}^2 \leq \eta_\ell^2(\tau)\cdot C_{df} D(\tau)\nonumber\\
    &\leq \frac{C_g\eta_g(\tau)}{4}\cdot\norm{(I-R)\cdot\yy(\tau)}^2+ \frac{c\eta_\ell(\tau)}{4}\norm{\nabla f(\bar{\xx}(\tau))}^2,\; \forall \tau \in[0,t]\label{eq:ct:bound_g}.
\end{align}
Substituting \eqref{eq:ct:bound_g} to \eqref{eq:ct:2}, we have: 
\[\begin{aligned}
    \dot{{\cE}}(\tau)&\leq -\frac{C_g\eta_g(\tau)}{4}\cdot\norm{(I-R)\cdot\yy(\tau)}^2- \frac{c\eta_\ell(\tau)}{4}\norm{\nabla f(\bar{\xx}(\tau))}^2<0, \; \forall \tau \in [0,t].
\end{aligned}
\]
Integrating over time, it gives $\cE(t) = \cE(0) + \int^t_{0} \dot{\cE}(s) \md s \leq \cE(0)$.  Therefore, $\mS(\tau):= \{\xx \mid f(\xx)\leq \cE(\tau) + f^\star\} \subseteq \mS(0)$, $D(\tau)\leq D(0),~ \forall \tau \in [0,t]$. So we can choose the parameters as:
\begin{align*}
    \eta_g(\tau) &= 1, \;c = \frac{1}{2}, \;\beta_1(\tau) = \frac{C_g}{4},\; \beta_2(\tau) = \frac{\eta_\ell(\tau)}{2}\\
    \eta_\ell(\tau) & = \max\left\{\frac{\sqrt{C_gC_{df}L_f}\norm{(I-R)\cdot \yy(\tau)}^2}{4C_{df}D(0) + 2C_{df}L_f\norm{(I-R)\cdot \yy(\tau)}^2}, \frac{\norm{\nabla f(\bar{\xx}(\tau))}^2}{8C_{df} D(0)} \right\}\\
    & = \Theta\left(\norm{(I-R)\cdot \yy(\tau)}^2 + \norm{\nabla f(\bar{\xx}(\tau))}^2\right), \; \forall \tau \in [0,t]. 
\end{align*}
Based on the above choices of parameters, we will show below that the convergence rate of the system is $\cO(1/\sqrt{t})$.
If $\min_{\tau \in [0,t]} \norm{(I-R)\cdot \yy(\tau)}^2 + \norm{\nabla f(\bar{\xx}(\tau))}^2 = \cO(\frac{1}{\sqrt{t}})$, then the result is achieved. 
Otherwise we have:
\begin{equation}\label{eq:max_gap}
    \norm{(I-R)\cdot \yy(\tau)}^2 + \norm{\nabla f(\bar{\xx}(\tau))}^2 = \Omega\left(\frac{1}{\sqrt{t}}\right), \; \forall \tau \in [0,t].
\end{equation}
This will guarantee that $\eta_\ell(\tau)= \Theta(\frac{1}{\sqrt{t}}),\; \forall \tau \in [0,t]$ and $\gamma_1(\tau) = \frac{\eta_\ell(\tau)}{4} = \Theta(\frac{1}{\sqrt{t}})$, $\gamma_2(\tau) = \frac{C_g}{4} = \cO(1),\;\forall \tau \in [0,t]$ for \pref{as:energy}. Then we apply \thref{theorm:energy:continuous} and obtain that \[\min_\tau\{\norm{(I-R)\cdot \yy(\tau)}^2 + \norm{\nabla f(\bar{\xx}(\tau))}^2\} = \max\{\cO(1/t) ,\cO(1/\sqrt{t})\} = \cO(1/\sqrt{t}).\]}

Summarizing the above two cases, we have the worst convergence rate for the algorithm as: $\max\{\cO(1/t) ,\cO(1/\sqrt{t})\} = \cO(1/\sqrt{t}).$
This completes the proof for \coref{cor:continuous}.

\comment{
\section{Convergence Guarantee for Case V}\label{sup:case_5}
In this section, we present the convergence result for Case V where $\tau_\ell > \tau_g> 0$. 

First of all, let us state the guarantees we have for  Case V.

\begin{lemma}[Dynamics of $\cE$ in Case V]\label{le:case5}
     Suppose the GCFL and LCFL satisfy properties \pref{as:gg_decrease}-\pref{as:energy}, and consider the  discretized system with $\tau_\ell = K\cdot \tau_g$. Then we have: {
    \begin{equation}
        \begin{aligned}
            \int^t_{0}\dot{\cE}(\tau)\md \tau &\leq \int^t_{0}\bigg(-\underbrace{\left(\frac{\gamma_1(\tau)}{2}-C_{51}(\tau)\right)}_{\rm := \hat{\gamma}_1(\tau)}\cdot\norm{\nabla f(\bar{\xx}(\tau))}^2 \\
            & \quad - \underbrace{\left(\frac{\gamma_2(\tau)}{2} - C_{52}(\tau)\right)}_{\rm := \hat{\gamma}_2(\tau)}\cdot \norm{(I - R)\cdot\yy(\tau)}^2\bigg)\md\tau,
        \end{aligned}
    \end{equation}}
    where we have defined: 
    {
    \begin{align*}
    C_{51} & := C_{55}C_{54}, \quad C_{52} := \left(\frac{C_g\tau_g\eta_g(\tau)}{1-C_g\tau_g\eta_g(\tau)}\right)^2\cdot\frac{C_g\eta_g^2(\tau)}{2\gamma_2(\tau)} +C_{55}C_{53},\\
        C_{55}& := \frac{L^2\eta_\ell^2(\tau)\cdot(1+2C_g\tau^2_g\eta_g^2(\tau))}{2\min\{N\gamma_1(\tau), \gamma_2(\tau)\}}, \; C_{53} :=\frac{\left(\frac{1-C'_y}{C'^2_y}\right)+ 4L^2L_f^2C^2_\ell C_f}{1-2L^2C'_\ell},\\
        C_{54} &:= \frac{4L^2C'_\ell C_f}{1-2L^2C_\ell^2}, \quad C'_\ell:=  C'_\ell:=  \frac{\tau_\ell^2 \eta_\ell^2(\tau)}{C'_yC_g\tau_g\eta_g(\tau)}, \mbox{ and}\quad C'_y := (1-C_g\tau_g\eta_g(\tau))^K.
    \end{align*}}%
\end{lemma}

We can check that when $\tau_g = \tau_\ell = 0$, $C_{51}(\tau)$ and $C_{52}(\tau)$ are both zero.

From \leref{le:case5}, by solving $\gamma_1(t) \geq 2C_{51}, \gamma_2(t) \geq 2C_{52}(t)$, we have:
{\begin{equation}\label{eq:step:5}
    \begin{aligned}
        \tau_g &\leq \frac{1}{C_g\eta_g(t)}\cdot\min\left\{\frac{\gamma_2(t)}{\gamma_2(t)+\sqrt{2C_g}\eta_g(t)}, 1-\left( 1+ \frac{\tilde{\gamma}_2(t)\sqrt{1-2L^2C_\ell^2}}{\sqrt{6}L^2\eta_g(t)}\right)^{-1/K}\right\},\\
        \tau_\ell &\leq \frac{C_g\eta_\ell(t)\tilde{\gamma}_1(t)}{\sqrt{12L^2(C_f+2\tilde{\gamma}_1^2(t))}L\eta_\ell(t)},
    \end{aligned}
\end{equation}
}%
where we have defined ${\tilde{\gamma}^2_1(t)} := \min\{N\gamma_1^2(t), \gamma_1(t)\cdot\gamma_2(t)\}, \; \tilde{\gamma}_2^2(t) := \min\{\gamma_2^2(t), N\gamma_1(t)\cdot\gamma_2(t)\}.$ The above result says that, the maximum sampling intervals are determined by the constants related to the continuous-time system. Note that the local sampling interval $\tau_\ell$ can be chosen independent of $K$, while the global sampling interval $\tau_g$ is {\it inversely proportional} to $K$. In addition,  the number of communications between any two local updates (i.e., $K$) should not be too large, and such an upper bound is also determined by the continuous-time system. 

\subsection{Proof of \leref{le:case5}}
Following  Case IV, the sampling error {$\hat{\cE}(t)$} in \eqref{eq:hat_e} can be bounded exactly the same as \eqref{eq:hat_e_345}, that is, we have:
\begingroup
\allowdisplaybreaks
\begin{align}
\begin{split}
\label{eq:E:hat:case5}
    \hat{\cE}(t)&
    \leq \frac{\gamma_1(t)}{2}\norm{\nabla f(\bar{\xx})}^2+ \frac{\gamma_2(t)}{2}\norm{(I - R)\cdot\yy(t)}^2 + \frac{\eta_g^2(t)}{2\gamma_2(t)} \norm{(I-R)\cdot(\yy(t)-\yy(t_g))}^2\\
    &\quad + \frac{L^2\eta_\ell^2(t)}{2\min\{N\gamma_1(t), \gamma_2(t)\}}\left(\norm{\yy(t)-\yy(t_\ell)}^2+\norm{\zz(t)-\zz(t_\ell )}^2\right).
\end{split}
\end{align}
\endgroup%

We need to bound the last three terms. Before we start, 
recall that in Case V we have $\tau_\ell  = K\tau_g$. The update of the states is given by:
{
\begin{align}
\begin{split}\label{eq:case_5_1}
 \yy(t_\ell+ (k+1)\tau_g) &= \yy(t_\ell+ k\tau_g) - \int^{t_\ell+(k+1)\tau_g}_{t_\ell+k\tau_g}\eta_g(s)\hat{u}_{g}(s) + \eta_\ell(s)\hat{u}_{\ell,y}(s)\md s \\
    \zz(t_\ell+ K\tau_g) &= \zz(t_\ell) - \int^{t_\ell+K\tau_g}_{t_\ell}\eta_\ell(s)\cdot \hat{u}_{\ell,z}(s)\md s.
    \end{split}
\end{align}}%
It follows that we have the following bound for the successive different of $\zz$:
\begin{align}\label{eq:z:norm}
      \|\zz(t_\ell+ K\tau_g) - \zz(t_\ell)\| & \le  \int^{t_\ell+K\tau_g}_{t_\ell}\eta_\ell(s)\md s \cdot \| \hat{u}_{\ell,z}(t)\|\nonumber\\
      & \le K \tau_g \cdot  \eta_\ell(t)\cdot \|\hat{u}_{\ell,z}(t)\|
\end{align}
where the first inequality holds because $\hat{u}_{\ell,z}(t) = u_{\ell,z}(t_{\ell}), \forall~t\in [t_{\ell} + \tau_{\ell}]$; the second inequality holds because the assumption that $\eta_{\ell}(s)$ is constant during the sampling interval $s\in [t_{\ell}, t_{\ell} + \tau_{\ell}]$. Next, let us proceed to bound the last three terms of \eqref{eq:E:hat:case5}.

First, we bound the last two terms of $\hat{\cE}(t)$. Apply \pref{as:gg_linear} to the first relation in \eqref{eq:case_5_1}, {and use the fact that $\hat{u}_{\ell,y}(t) = u_{\ell,y}(t_\ell), \; \forall~t \in [t_{\ell}, t_{\ell}+ \tau_\ell]$}, we obtain:
\begingroup\allowdisplaybreaks
{\begin{align}
    &(I-R)\cdot\yy(t_\ell  + (k+1)\tau_g) = (I-R)\cdot\yy(t_\ell+ k\tau_g) \nonumber\\
    &\quad - \int^{t_\ell+(k+1)\tau_g}_{t_\ell+k\tau_g}\eta_g(s)(I-R)\cdot W_A \yy(t_\ell+ k\tau_g) + \eta_\ell(s)(I-R)\cdot\hat{u}_{\ell,y}(s)\md s\nonumber\\
    & = \left(I - \int^{t_\ell+(k+1)\tau_g}_{t_\ell+k\tau_g}\eta_g(s)\md s)\cdot (I-R)\cdot W_A\right)\cdot \yy(t_\ell+ k\tau_g)\nonumber\\ 
    &\quad - \int^{t_\ell+(k+1)\tau_g}_{t_\ell+k\tau_g}\eta_\ell(s)\md s \cdot (I-R)\cdot u_{\ell,y}(t_\ell)\nonumber\\
    & \stackrel{(i)} = \prod^{k}_{s=0} (I - \tau_g\eta_g(t_\ell+s\tau_g))\cdot (I-R)\cdot W_A) \cdot (I-R)\cdot \yy(t_\ell) \label{eq:case_5_1_1}\\
    &\quad - \sum^k_{r=0}\prod^k_{s=r} (I - \tau_g\eta_g(t_\ell+s\tau_g))\cdot (I-R)\cdot W_A) \cdot \tau_g\eta_\ell(t_\ell+r\tau_g)\cdot (I-R)\cdot u_{\ell,y}(t_\ell),\nonumber
\end{align}}%
\endgroup
where in $(i)$ we use the fact that $\eta_g(s), \eta_\ell(s)$ are constants within the sampling intervals,
and recursively expand the updates, and finally use the fact that $(I-R)\cdot W_A = W_A$. Then we can derive the following:
\begingroup
\allowdisplaybreaks
\small\begin{align}
  &\norm{\yy(t_\ell+ k\tau_g) - \yy(t_\ell)}^2 = \norm{(I-R)\cdot (\yy(t_\ell+ k\tau_g) - \yy(t_\ell))}^2 + \norm{R\cdot(\yy(t_\ell+ k\tau_g) - \yy(t_\ell))}^2\nonumber\\
  & = \left\|\left(I - \left(\prod^{k}_{s=0} (I - \tau_g\eta_g(t_\ell+s\tau_g))\cdot W_A)\right)^{-1}\right)\cdot (I-R)\cdot \yy(t_\ell+k\tau_g)\right.\nonumber\\
  &\qquad\left.+ \left(\prod^{k}_{s=0} (I - \tau_g\eta_g(t_\ell+s\tau_g))\cdot W_A)\right)^{-1}\sum^k_{r=0}\prod^k_{s=r} (I - \tau_g\eta_g(t_\ell+s\tau_g))\cdot W_A) \right.\nonumber\\
  & \qquad \cdot \tau_g\eta_\ell(t_\ell+r\tau_g)\cdot (I-R)\cdot u_{\ell,y}(t_\ell)\Bigg\|^2+ \norm{\sum^k_{r=0}\tau_g\eta_\ell(t_\ell+r\tau_g)\cdot R u_{\ell,y}(t_\ell)}^2\nonumber\\
  & \stackrel{(i)}\leq (1+\beta)\norm{I - \left(\prod^{k}_{s=0} (I - \tau_g\eta_g(t_\ell+s\tau_g))\cdot W_A)\right)^{-1}}^2\norm{(I-R)\cdot \yy(t_\ell+k\tau_g)}^2 \nonumber\\
  & + (1+\frac{1}{\beta})\norm{\left(\prod^{k}_{s=0} (I - \tau_g\eta_g(t_\ell+s\tau_g))\cdot W_A)\right)^{-1}\hspace{-0.2cm}\sum^k_{r=0}\prod^k_{s=r} (I - \tau_g\eta_g(t_\ell+s\tau_g))\cdot W_A)\tau_g\eta_\ell(t_\ell+r\tau_g)}^2\nonumber\\
  &\qquad \cdot \norm{(I-R)\cdot u_{\ell,y}(t_\ell)}^2 + (\sum^k_{r=0}\tau_g\eta_\ell(t_\ell+r\tau_g))^2\norm{R u_{\ell,y}(t_\ell)}^2\nonumber\\
  & \stackrel{(ii)}\leq (1+\beta)\cdot(1-(1-C_g\tau_g\eta_g(t_{\ell}+k\tau_g))^{-K})^2\norm{(I-R)\cdot \yy(t_\ell+k\tau_g)}^2\nonumber\\
  & \quad + \max\left\{(1+\frac{1}{\beta}) \frac{((1-C_g\tau_g\eta_g(t_{\ell}+k\tau_g))^{-K})^2}{C_g\tau_g\eta_g(t_{\ell}+k\tau_g)}, 1\right\}\cdot\left(\sum^k_{r=0}\eta'_\ell(t_\ell+r\tau_g)\right)^2\cdot \norm{u_{\ell,y}(t_\ell)}^2,\nonumber\\
  & \stackrel{(iii)}\leq \left(\frac{1-C'_y}{C'^2_y}\right)\cdot\norm{(I-R)\cdot \yy(t_\ell+k\tau_g)}^2 + \frac{\left(\sum^k_{s=0}\eta'_\ell(t_\ell+s\tau_g)\right)^2}{C'_yC_g\tau_g\eta_g(t_\ell)}\cdot\norm{u_{\ell,y}(t_\ell)}^2,\label{eq:case_5_1_2}
\end{align}%
\endgroup
where in $(i)$ we apply Cauchy-Schwarz inequality; in $(ii)$ we first bound $\norm{(I - \tau_g\eta_g(t_\ell+k\tau_g)\cdot W_A)^{-1}}$ by $1 - C_g \tau_g \eta_g(t_\ell+k\tau_g)$ applying \pref{as:gg_linear}; in $(iii)$ we define $C'_y := (1-C_g\tau_g\eta_g(t))^K$ to simplify the first term, choose $\beta = \frac{C'_y}{1-C'_y}$, and in the last term we bound the summation by the fact that \[\sum^{K-1}_{k=0} (1 - C_g\tau_g\eta_g(t_\ell))^k = \frac{1 - (1 - C_g\tau_g\eta_g(t_\ell))^K}{1 - (1 - C_g\tau_g\eta_g(t_\ell))}\leq \frac{1}{C_g\tau_g\eta_g(t_\ell)}.\]

Then, we can bound $\norm{\yy(t)-\yy(t_\ell )}^2 +\norm{\zz(t)-\zz(t_\ell )}^2$ by:
\begingroup\allowdisplaybreaks
{\small
    \begin{align}
      &\norm{\yy(t)-\yy(t_\ell )}^2 +\norm{\zz(t)-\zz(t_\ell )}^2\nonumber\\
      &\stackrel{(i)}\leq \left(\frac{1-C'_y}{C'^2_y}\right) \cdot \norm{(I-R)\cdot \yy(t)}^2 + \frac{\left(\sum^k_{r=0}\eta'_\ell(t_\ell+r\tau_g)\right)^2}{C'_y}\norm{\hat{u}_{\ell,y}(t)}^2 + \tau^2_\ell\eta_\ell^2(t)\cdot \norm{\hat{u}_{\ell,z}(t)}^2 \nonumber\\
      &\stackrel{(ii)}\leq \left(\frac{1-C'_y}{C'^2_y}\right) \cdot \norm{(I-R)\cdot \yy(t)}^2 + C'_\ell\norm{\hat{u}_{\ell}(t)}^2 \nonumber\\
      &\stackrel{(iii)}\leq \left(\frac{1-C'_y}{C'^2_y}\right) \cdot \norm{(I-R)\cdot \yy(t)}^2 + 2C'_\ell\left(\norm{\hat{u}_{\ell}(t) - u_{\ell}(t)}^2 + \norm{u_\ell(t)}^2\right) \nonumber\\
      &\stackrel{(iv)}\leq \left(\frac{1-C'_y}{C'^2_y}\right) \cdot \norm{(I-R)\cdot \yy(t)}^2 + 2C'_\ell L^2\left(\norm{\yy(t) - \yy(t_\ell)}^2 + \norm{\zz(t) - \zz(t_\ell)}^2)\right)\nonumber\\
      &\qquad +4C'_\ell C_f\left(\norm{\nabla f(\bar{\xx}(t))}^2 + \frac{L_f^2}{N}\norm{(I-R)\cdot \xx(t)}^2\right) \nonumber\\
      &\stackrel{(v)}\leq \frac{\left(\frac{1-C'_y}{C'^2_y}\right)+ 4L^2L_f^2C^2_\ell C_f}{1-2L^2C'_\ell}\norm{(I-R)\cdot\yy(t)}^2 + \frac{4L^2C'_\ell C_f}{1-2L^2C_\ell^2}\norm{\nabla f(\bar{\xx}(t))}^2,\label{eq:case_5_2}
    \end{align}}%
    \endgroup
where in $(i)$ we  apply \eqref{eq:case_5_1_2} to the first term and  \eqref{eq:z:norm} to the second term; in $(ii)$ we define $C'_\ell:=  \frac{\tau_\ell^2 \eta_\ell^2(t)}{C'_yC_g\tau_g\eta_g(t_\ell)}$, and merge the last two terms; in $(iii)$ we add and subtract $u_{\ell}(t)$ to the last term and apply Cauchy-Schwarz inequality; in $(iv)$ we apply \pref{as:gl_smooth} to the second term and \pref{as:gl_decrease} to the last term, add and subtract $\nabla f(\bar{\xx}(t))$ and use \asref{as:smooth}; in $(v)$, we move $\norm{\yy(t) - \yy(t_\ell)}^2 + \norm{\zz(t) - \zz(t_\ell)}^2$ to the left, by choosing $\tau_\ell < \sqrt{2}LC'_y$, such that $1-2L^2C'_\ell>0$, we divide both sides by $1-2L^2C'_\ell$ and merge the terms.

Next, we bound the term involving $\norm{(I-R)\cdot(\yy(t)-\yy(t_g))}^2$ in \eqref{eq:E:hat:case5}. Since $\|I-R\|\le 1$, it is sufficient to bound $\norm{\yy(t)-\yy(t_g)}^2$. We have:
{\footnotesize
\begingroup
\allowdisplaybreaks
\begin{align*}
    &\norm{\yy(t)-\yy(t_g)}^2 \stackrel{\eqref{eq:case_5_1_2}}\leq \left(\frac{C_g\tau_g\eta_g(t)}{1-C_g\tau_g\eta_g(t)}\right)^2\cdot\norm{(I-R)\cdot\yy(t)}^2 + \tau^2_g\eta_\ell^2(t) \norm{\hat{u}_{\ell,y}(t)}^2\\
    &\stackrel{(i)}\leq \left(\frac{C_g\tau_g\eta_g(t)}{1-C_g\tau_g\eta_g(t)}\right)^2\cdot\norm{(I-R)\cdot\yy(t)}^2 + 2\tau^2_g\eta_\ell^2(t)\left( \norm{\hat{u}_{\ell}(t) - u_{\ell}(t)}^2 + \norm{u_\ell(t)}^2\right)\\
    & \stackrel{(ii)}\leq \left(\frac{C_g\tau_g\eta_g(t)}{1-C_g\tau_g\eta_g(t)}\right)^2\cdot\norm{(I-R)\cdot\yy(t)}^2 + 2L^2\tau^2_g\eta_\ell^2(t)\left( \norm{\yy(t)-\yy(t_\ell )}^2 +\norm{\zz(t)-\zz(t_\ell )}^2 \right)\\
    & + 2C_fL^2\tau^2_g\eta_\ell^2(t)\norm{\nabla f(\xx(t))}^2\\
    & \stackrel{(iii)}\leq \left(\frac{C_g\tau_g\eta_g(t)}{1-C_g\tau_g\eta_g(t)}\right)^2\cdot\norm{(I-R)\cdot\yy(t)}^2 + 2L^2\tau^2_g\eta_\ell^2(t)\left( \norm{\yy(t)-\yy(t_\ell )}^2 +\norm{\zz(t)-\zz(t_\ell )}^2 \right)\\
    & +  4C_fL^2\tau^2_g\eta_\ell^2(t)\left(\norm{\nabla f(\bar{\xx}(t))}^2 + \frac{L_f^2}{N}\norm{(I-R)\cdot \xx(t)}^2\right)\\
     & \stackrel{(iv)}\leq \left(\left(\frac{C_g\tau_g\eta_g(t)}{1-C_g\tau_g\eta_g(t)}\right)^2 + 2L^2\tau^2_g\eta_\ell^2(t)\cdot C_{53}\right) \cdot\norm{(I-R)\cdot \yy(t)}^2 +  2L^2\tau^2_g\eta_\ell^2(t)\cdot C_{54}\norm{\nabla f(\bar{\xx}(t))}^2.
    \end{align*}
\endgroup
}
where in $(i)$ we add and subtract $\norm{u_\ell(t)}^2$; in $(ii)$ we apply \pref{as:gl_smooth} to the second term and \pref{as:gl_decrease} to the last term; in $(iii)$ we add and subtract $\nabla f(\bar{\xx}(t))$ and use \asref{as:smooth}; in $(iv)$ we plug in \eqref{eq:case_5_2} and define 
{\footnotesize\[C_{53} :=\frac{\left(\frac{1-C'_y}{C'^2_y}\right)+ 4L^2L_f^2C^2_\ell C_f}{1-2L^2C'_\ell},\quad  C_{54} := \frac{4L^2C'_\ell C_f}{1-2L^2C_\ell^2}.\]}

Therefore we have:
\begin{align*}
    \hat{\cE}(t) &\leq  \left(\frac{\gamma_2(t)}{2}+\left(\frac{C_g\tau_g\eta_g(t)}{1-C_g\tau_g\eta_g(t)}\right)^2\cdot\frac{C_g\eta_g^2(t)}{2\gamma_2(t)} +C_{55}C_{53}\right)\cdot \norm{(I-R)\cdot \yy(t)}^2\\
    &\quad + \left(\frac{\gamma_1(t)}{2}+C_{55}C_{54}\right) \cdot\norm{\nabla f(\bar{\xx}(t))}^2,
\end{align*}
where we define $C_{55} := \frac{L^2\eta_\ell^2(t)\cdot(1+2C_g\tau^2_g\eta_g^2(t))}{2\min\{N\gamma_1(t), \gamma_2(t)\}}$.

Substitute the above result in \eqref{eq:hat:int}, $\cE(t)$ can be bounded as{\footnotesize
\begin{align*}
        \cE(t) - \cE(0) \leq -\int^t_0\left((\frac{\gamma_1(\tau)}{2}-C_{51}(\tau))\cdot\norm{\nabla f(\bar{\xx}(\tau))}^2 - (\frac{\gamma_2(\tau)}{2} - C_{52}(\tau))\cdot\norm{(I - R)\cdot\yy(\tau)}^2\right)\md \tau,
    \end{align*}}%
where $C_{51} := C_{55}C_{54}$, $C_{52} := \left(\frac{C_g\tau_g\eta_g(\tau)}{1-C_g\tau_g\eta_g(\tau)}\right)^2\cdot\frac{C_g\eta_g^2(\tau)}{2\gamma_2(\tau)} +C_{55}C_{53}$. The proof is completed.
}

\section{Verify Property P5 for DGT Algorithm}\label{sec:P5:DGT}
Recall that the derivative of the energy function is given by:
\small{
\begin{align}
    \dot{\cE}(t) & = -\lin{\nabla f(\bar{\xx}(t)), \frac{1}{N}\sum^N_{i=1}u_{\ell,x}(t)} -\lin{(I-R)\cdot\yy(t), u_{g,y}(t) + u_{\ell,y}(t)}\nonumber\\
    & \stackrel{\eqref{eq:GT:controller}}= -\lin{\nabla f(\bar{\xx}(t)), c\bar{\vv}(t)} -\lin{(I-R)\cdot\yy(t), (I-W)\cdot\yy(t)}\label{eq:E:dynamics:supp}\\
    & \qquad- \lin{(I-R)\cdot\xx(t), c \vv(t)} + \lin{(I-R)\cdot\vv(t), \nabla f(\xx(t)) - \nabla f(\zz(t))}.\nonumber
\end{align}
}
Then we bound each term on the RHS above separtately.

To bound the first term, note that:
\begingroup
{\allowdisplaybreaks}
{\small
\begin{align*}
    \frac{c}{2}\norm{\nabla f(\bar{\xx}(t)) - \bar{\vv}(t)}^2 & = \frac{c}{2}\norm{\nabla f(\bar{\xx}(t))- \frac{1}{N}\sum^N_{i=1}\nabla f(x_i(t)) + \frac{1}{N}\sum^N_{i=1}\nabla f(x_i(t)) - \bar{\vv}(t)}^2\\
    & \stackrel{(i)}{\leq} c\left(\frac{1}{N}\sum^N_{i=1}\norm{\nabla f(\bar{\xx}(t)) - \nabla f(x_i(t))}^2 + \norm{\frac{1}{N}\sum^N_{i=1}\nabla f(x_i(t)) - \bar{\vv}(t)}^2\right)\\
    & \stackrel{(ii)}{\leq} c\left(\frac{L_f}{N}\sum^N_{i=1}\norm{\bar{\xx}(t) - x_i(t)}^2 + \norm{\frac{1}{N}\sum^N_{i=1}\nabla f(x_i(t)) - \bar{\vv}(t)}^2\right)\\
    & \stackrel{(iii)}{\leq} c\left(\frac{L_f}{N}\norm{(I-R)\cdot\xx(t)}^2 +  \norm{\frac{\bone^T}{N}\nabla f(\xx(t)) - \bar{\vv}(t)}^2\right),
\end{align*}}%
\endgroup%
where in $(i)$ we apply \eqref{eq:ab} and Jensen's inequality; in $(ii)$ we apply \asref{as:smooth}; in $(iii)$ we substitute the definition of $R$. From \eqref{eq:gt:state_v},
$ \bar{\vv}(t) = \frac{\bone^T}{N}\nabla f(\xx(t)),$ and we have{\small
\[\frac{c}{2}\norm{\nabla f(\bar{\xx}(t)) - \bar{\vv}(t)}^2 \leq c\frac{L_f}{N}\norm{(I-R)\cdot\xx(t)}^2.\]}%
{So the first term in \eqref{eq:E:dynamics} can be bounded as
\begin{align}
    -\lin{\nabla f(\bar{\xx}(t)), c\bar{\vv}(t)} & = -\frac{c}{2}\left(\norm{\nabla f(\bar{\xx}(t))}^2 +\norm{\bar{\vv}(t)}^2-\norm{\nabla f(\bar{\xx}(t)) - \bar{\vv}(t)}^2\right)\\
    & \le -\frac{c}{2}\left(\norm{\nabla f(\bar{\xx}(t))}^2 +\norm{\bar{\vv}(t)}^2-\frac{2L_f}{N}\norm{(I-R)\cdot\xx(t)}^2\right) \nonumber.
\end{align}

}

The second term in \eqref{eq:E:dynamics}  can be bounded  by directly applying \pref{as:gg_decrease}. That is, we have:
\begin{align*}
     -\lin{(I-R)\cdot\yy(t), (I-W)\cdot\yy(t)} \leq -C_g \norm{(I-R)\cdot\yy(t)}^2.
\end{align*}
Next, the third term in \eqref{eq:E:dynamics} can be bounded as: 
\begin{align*}
    & -c  \lin{(I-R)\cdot\xx(t), \vv(t)}  \stackrel{\eqref{eq:IR}}= -c \lin{(I-R)\cdot\xx(t), (I-R)\cdot\vv(t)} \\
    & \stackrel{\eqref{eq:ab}}\leq \frac{c}{2}\cdot\left(\norm{(I-R)\cdot\xx(t)}^2 + \norm{(I-R)\cdot\vv(t)}^2\right) = \frac{c}{2}\norm{(I-R)\cdot\yy(t)}^2.
\end{align*}

Finally, we bound the last term in \eqref{eq:E:dynamics} by:
{\small
\begin{align*}
    &\lin{(I-R)\cdot\vv(t), \nabla f(\xx(t)) - \nabla f(\zz(t))} \stackrel{\eqref{eq:IR}}=  \lin{(I-R)\cdot\vv(t), (I-R)\cdot(\nabla f(\xx(t)) - \nabla f(\zz(t)))}\\
    &\stackrel{\eqref{eq:ab}}\leq \frac{\beta}{2}\norm{(I-R)\cdot\vv(t)}^2 + \frac{1}{2\beta}\norm{(I-R)\cdot(\nabla f(\xx(t)) - \nabla f(\zz(t)))}^2\\
    &\stackrel{(i)}= \frac{\beta}{2}\norm{(I-R)\cdot\vv(t)}^2 + \frac{1}{2\beta N}\sum^N_{i=1}\norm{\frac{\bone^T}{N}\nabla f(\xx(t)) -\nabla f_i(x_i(t)) - \frac{\bone^T}{N}\nabla f(\zz(t)) +\nabla f_i(z_i(t))}^2,
\end{align*}
}%
where $(i)$ is due to $R := \frac{1}{N}\bone \bone^T$. The last term above can be further bounded by:
\begingroup
\allowdisplaybreaks
{\footnotesize
\begin{align*}
    &\frac{1}{2\beta N}\sum^N_{i=1}\norm{\frac{\bone^T}{N}\nabla f(\xx(t)) -\nabla f_i(x_i(t)) - \frac{\bone^T}{N}\nabla f(\zz(t)) +\nabla f_i(z_i(t))}^2\\
    & \stackrel{(i)}= \frac{1}{2\beta N}\sum^N_{i=1}\norm{\left(\frac{\bone^T}{N}\nabla f(\xx(t))- \nabla f(\bar{\xx}(t))\right) + \left(\nabla f(\bar{\xx}(t)) - \frac{\bone^T}{N}\nabla f(\zz(t))\right) - (\nabla f_i(x_i(t))-\nabla f_i(z_i(t)))}^2\\
    & \leq \frac{2}{\beta N}\sum^N_{i=1}\left(\norm{\frac{\bone^T}{N}\nabla f(\xx(t))- \nabla f(\bar{\xx}(t))}^2+ \norm{\nabla f(\bar{\xx}(t))- \frac{\bone^T}{N}\nabla f(\zz(t))}^2 + \norm{\nabla f_i(x_i(t))- \nabla f_i(z_i(t))}^2\right)\\
    &\stackrel{(ii)}\leq \frac{2L_f}{\beta}(\norm{(I-R)\cdot\xx(t)}^2 + \norm{\bar{\xx}(t) - \zz(t)}^2 + \norm{\xx(t) - \zz(t)}^2)\\
    &= \frac{2L_f}{\beta}(\norm{(I-R)\cdot\xx(t)}^2 + \norm{\bar{\xx}(t) - \bar{\zz}(t) + \bar{\zz}(t)-\zz(t)}^2 +  \norm{\xx(t) -\bar{\xx}(t) + \bar{\xx}(t) -\bar{\zz}(t) + \bar{\zz}(t)- \zz(t)}^2)\\
    & \leq \frac{8L_f}{\beta}(\norm{(I-R)\cdot\xx(t)}^2 + \norm{\bar{\xx}(t) - \bar{\zz}(t)}^2 + \norm{(I-R)\cdot\zz(t)}^2),
\end{align*}}
\endgroup%
where in $(i)$ we add and substracts $\nabla f(\bar{\xx}(t))$; in $(ii)$ we apply \asref{as:smooth}. 

Finally, we analyze $\norm{\bar{\xx}(t) - \bar{\zz}(t)}^2$:
\begin{align*}
    \norm{\bar{\xx}(t)- \bar{\zz}(t)}^2 &\stackrel{\eqref{eq:gt:dynamics}} = \norm{\frac{\bone^T}{N}\int^t_{0}\left((I-W)\cdot\xx(\tau) - c\vv(\tau)\right)\e^{-(t-\tau)}\md \tau}^2\\
    &\stackrel{(i)} = c^2\norm{\int^t_{0} \bar{\vv}(\tau)\e^{-(t-\tau)}\md \tau}^2 \stackrel{(ii)} \leq c^2\int^t_{0}e^{-(t-\tau)}\md\tau\cdot\int^t_{0} \norm{\bar{\vv}(\tau)}^2\e^{-(t-\tau)}\md \tau\\
    &\leq c^2\int^t_{0} \norm{\bar{\vv}(\tau)}^2\e^{-(t-\tau)}\md \tau,
\end{align*}
where in $(i)$ we apply \eqref{eq:gg_bound}, that $\bone^T W_A =0$, and in this case $W_A = (I-W)$; in $(ii)$ we use Cauchy–Schwarz inequality to break the integration. 

Plugging in the above into \eqref{eq:GT:controller}, the final bound we have is:
\small{
\begin{align}
    \dot{\cE}(t) &\leq -\frac{c}{2}\norm{\nabla f(\bar{\xx}(t))}^2 -\frac{c}{2}\norm{\bar{\vv}(t)}^2 + \frac{8L_fc^2}{\beta}\int^t_{0} \norm{\bar{\vv}(\tau)}^2\e^{-(t-\tau)}\md \tau\\
    &\quad- \bigg(C_g - \frac{c+2cL_f/N+\beta + 16cL_f/\beta}{2}\bigg)\cdot\norm{(I-R)\cdot\yy(t)}^2.\nonumber
\end{align}
}
Integrating the above relation over time, we have:
\begingroup
\allowdisplaybreaks
\small{
\begin{align*}
    \int^t_{0}\dot{\cE}&(\tau)\md \tau \leq -\frac{c}{2}\int^t_{0}\norm{\nabla f(\bar{\xx}(\tau))}^2\md \tau  + \frac{8L_fc^2}{\beta}\int^t_{0}\int^\tau_{0}\norm{\bar{\vv}(\tau_1)}^2\e^{-(\tau-\tau_1)}\md \tau_1 \md \tau\\
    &\quad-\frac{c}{2}\int^t_{0}\norm{\bar{\vv}(\tau)}^2\md \tau- \left(C_g - \frac{c+2cL_f+\beta + 16cL_f/\beta}{2}\right)\cdot\int^t_{0}\norm{(I-R)\yy(\tau)}^2\md \tau \nonumber\\
     & \stackrel{(i)} = -\frac{c}{2}\int^t_{0}\norm{\nabla f(\bar{\xx}(\tau))}^2\md \tau  + \frac{8L_fc^2}{\beta}\int^t_{0} \left(\norm{\bar{\vv}({\tau_1})}^2\int^t_{\tau_1}\e^{-(\tau-\tau_1)}\md \tau \right)\md \tau_1\\
     &\quad-\frac{c}{2}\int^t_{0}\norm{\bar{\vv}(\tau)}^2\md \tau- \left(C_g - \frac{c+2cL_f+\beta + 16cL_f/\beta}{2}\right)\cdot\int^t_{0}\norm{(I-R)\cdot\yy(\tau)}^2\md \tau\nonumber\\
    &\stackrel{(ii)}\leq -\frac{c}{2}\int^t_{0}\norm{\nabla f(\bar{\xx}(\tau))}^2\md \tau -\frac{c-8L_fc^2/\beta}{2}\int^t_{0}\norm{\bar{\vv}(\tau)}^2\md \tau\\
    &\quad- (C_g - \frac{c+2cL_f+\beta + 16cL_f/\beta}{2})\cdot\int^t_{0}\norm{(I-R)\cdot\yy(\tau)}^2\md \tau,\nonumber
\end{align*}
}%
\endgroup
where in $(i)$ we switch the order of integration; 
in $(ii)$ we apply that
 {\small $\int^t_{\tau_1}\e^{-(t-\tau)}\md \tau \leq 1$.}
 

\end{document}